%% file: IMC.tex
\documentclass[english]{article}
\usepackage[T1]{fontenc}
\usepackage[latin9]{inputenc}
\usepackage{color}
\usepackage{babel}
\usepackage{float}
\usepackage{mathtools}
\usepackage{bm}
\usepackage{dsfont}
\usepackage{amsmath}
\usepackage{amssymb}
\usepackage{booktabs}
\usepackage{graphicx}
\usepackage{caption}
\usepackage{geometry}
\usepackage[authoryear,round]{natbib}
\usepackage{algorithm}
\usepackage{algpseudocode}
\usepackage[pdfusetitle,
 bookmarks=true,bookmarksnumbered=false,bookmarksopen=false,
 breaklinks=false,pdfborder={0 0 1},backref=false,colorlinks=true,
 linkcolor=refred,citecolor=papersblue,urlcolor=papersblue]
 {hyperref}

\makeatletter

\floatstyle{ruled}
\newfloat{algorithm}{tbp}{loa}
\providecommand{\algorithmname}{Algorithm}
\floatname{algorithm}{\protect\algorithmname}

\usepackage{babel}

\usepackage{enumitem}
\usepackage{amsthm}
\theoremstyle{plain}

\newtheorem{lemma}{\textbf{Lemma}}
\newtheorem{theorem}{\textbf{Theorem}}
\newtheorem{corollary}{\textbf{Corollary}}
\newtheorem{assumption}{\textbf{Assumption}}

\newtheorem{definition}{\textbf{Definition}}

\theoremstyle{definition}

\usepackage{color}
\definecolor{cm}{RGB}{0,0,200}
\definecolor{papersblue}{RGB}{0,0,200}
\definecolor{refred}{RGB}{200,0,0}

\definecolor{yy}{RGB}{160, 32, 240}

\definecolor{ct}{RGB}{4, 175, 112}

\author{
  Yuepeng Yang\thanks{Department of Statistics and Data Sciences, Yale University} \vspace{0.5em}\\
   Yale 
  \and
  Cong Ma\thanks{Department of Statistics, University of Chicago} \vspace{0.5em}\\
   UChicago 
} 

\makeatother

\begin{document}
\title{Sample-efficient inductive matrix completion with noise and inexact
    side-information}
\date{June 9, 2026}
\maketitle
\begin{abstract}
    Inductive matrix completion (IMC) is a variant of low-rank matrix completion that incorporates row and
    column side-information. In principle, it can reduce the effective dimension of the recovery problem from the ambient matrix size to the dimension of the side-information features.
    Existing theory, however, does not fully realize this advantage in the noisy setting: sample-efficient guarantees only apply to noiseless recovery, while noisy guarantees require sample sizes comparable to ordinary matrix completion. This paper closes this gap for noisy IMC. We analyze a nonconvex projected gradient descent algorithm with spectral initialization and prove that, under exact side-information, it achieves linear convergence and stable recovery at a sample complexity governed by the effective side-information dimension rather than the ambient matrix dimension. The key technical ingredient is a local regularity condition for the IMC loss that holds at this reduced sample size, despite the mismatch between the observation pattern and the side-information subspaces.

    We further extend the analysis to inexact side-information, showing that the same reduced sample complexity is preserved and that the estimation error degrades optimally with the level of subspace misspecification. Motivated by this trade-off, we also propose a penalized interpolation between IMC and ordinary matrix completion that balances sample efficiency against robustness to imperfect side-information. Simulations and experiments on the MovieLens dataset support the theoretical findings and illustrate the practical benefits of exploiting side-information in low-sample regimes.
\end{abstract}

\input{sections/01-introduction.tex}
\input{sections/02-problem-setup.tex}

\input{sections/03-exact-side-information.tex}
\input{sections/04-analysis.tex}
\input{sections/05-inexact-side-information.tex}
\input{sections/06-experiments.tex}
\input{sections/07-discussion.tex}

\bibliographystyle{plainnat}
\bibliography{All-of-Bibs}

\appendix

\input{sections/09-proof-of-subgaussian-corollary.tex}
\input{sections/10-proof-of-inexact-side-information.tex}
\input{sections/11-analysis-for-spectral-initialization.tex}
\input{sections/12-proof-of-regularity-condition.tex}
\input{sections/13-supporting-lemmas.tex}
\input{sections/14-further-experimental-results.tex}

\end{document}

%% file: sections/01-introduction.tex
\section{Introduction}

Low-rank matrix completion is a fundamental problem in high-dimensional
statistics and machine learning. We consider a rank-$r$ ground-truth
matrix $\bm{L}^{\star}\in\mathbb{R}^{n_{1}\times n_{2}}$ corrupted
by an additive noise matrix $\bm{E}\in\mathbb{R}^{n_{1}\times n_{2}}$,
where each entry in $[n_{1}]\times[n_{2}]$ is observed independently
with probability $p$. Let $\Omega$ denote the set of observed entries.
Given the observed matrix
\begin{equation}
\bm{M}=\mathcal{P}_{\Omega}(\bm{L}^{\star}+\bm{E}),\label{eq:mc}
\end{equation}
the goal in the vanilla matrix completion problem is to estimate $\bm{L}^{\star}$. This problem is now well
understood in several canonical regimes. Let $n$ denote the larger
of $n_{1}$ and $n_{2}$. Prior work shows that with $\widetilde{O}(nr)$
observed entries, one can achieve exact recovery in the noiseless
setting when $\bm{E}=\bm{0}$ \citep{candes2010power,recht2011simpler}
and near-optimal estimation error in the noisy setting \citep{Chen2020,chen2020noisy}. 

In the vanilla matrix completion problem, the only constraint on $\bm{L}^{\star}$
is that it is low-rank. In many applications, however, additional information about the row and column spaces is available.
We model this information through two orthonormal side-information matrices $\bm{X}\in\mathbb{R}^{n_{1}\times a_{1}}$
and $\bm{Y}\in\mathbb{R}^{n_{2}\times a_{2}}$, whose columns represent the possible column and row spaces of $\bm{L}^{\star}$.
In the case with exact side-information, i.e., $\bm{L}^{\star}$ is contained in the column and row space spanned by $\bm{X}$ and $\bm{Y}$, we can write $\bm{L}^{\star}$
as $\bm{X}\bm{Z}^{\star}\bm{Y}^{\top}$ for some rank-$r$ matrix
$\bm{Z}^{\star}\in\mathbb{R}^{a_{1}\times a_{2}}$ and the observed
matrix as
\begin{equation}
\bm{M}=\mathcal{P}_{\Omega}(\bm{X}\bm{Z}^{\star}\bm{Y}^{\top}+\bm{E}).\label{eq:imc}
\end{equation}
The goal now is to recover the matrix $\bm{L}^{\star}=\bm{X}\bm{Z}^{\star}\bm{Y}^{\top}$
and this problem is sometimes called \emph{inductive matrix completion}
(IMC) or subspace-constrained matrix completion. 
For comparison,
we refer to \eqref{eq:mc} as the ordinary or non-inductive matrix completion problem.

The key statistical advantage of IMC is that it reduces the effective dimension of the problem.
Since the unknown
matrix is now parameterized by the smaller core matrix $\bm{Z}^{\star}\in\mathbb{R}^{a_{1}\times a_{2}}$,
one expects both the sample complexity and the estimation error to be governed by the feature
dimensions $a_{1},a_{2}$ rather than the ambient dimensions $n_{1},n_{2}$.
Side-information of this form arises naturally in problems where rows and columns have covariates or known structural features, such as multi-label learning \citep{xu2013speedup,jain2013provable},
gene-disease association prediction \citep{natarajan2014inductive,chen2018predicting,lu2018prediction,li2020neural},
recommendation systems \citep{shin2015tumblr,jin2016provable}, link
prediction \citep{zhao2022novel}, and semi-supervised clustering \citep{chiang2015matrix,zhao2017multi,nazarov2018sparse}.

Existing theory captures only part of this picture. In the noiseless setting, prior work shows that IMC can achieve exact recovery with substantially fewer samples than ordinary matrix completion, with the sample complexity governed by the lower-dimensional side-information structure \citep{chen2015incoherence,xu2013speedup}.
In the noisy setting, \citet{chen2022nonconvex} proves a sharp estimation error guarantee,
but requires a sample rate comparable to ordinary matrix completion. It is therefore natural to ask whether these two types of
results can be unified:
\begin{quote}
\emph{Can one estimate $\bm{L}^{\star}$ in the presence of noise $\bm{E}$
with both a reduced inductive sample complexity and sharp statistical error guarantees?}
\end{quote}

A second major limitation of some prior works on IMC is that they require exact side-information. This
can be restrictive as the side-information may often be inexact in practice. This raises a second question:
\begin{quote}
\emph{When the side-information is inexact, can one still estimate $\bm{L}^{\star}$ with the reduced sample
complexity, and how does the estimation error scale with the side-information inexactness?}
\end{quote}

In this work, we answer both questions affirmatively. 
We summarize our result below.

\subsection{Our Contributions}
While our results hold for rectangular matrices and general additive noise, we restrict our discussion to the symmetric case $n_1=n_2=n$ and $a_1=a_2=a$ and entrywise Gaussian noise $\mathcal{N}(0, \sigma^2)$ here for simplicity. 

\paragraph{Sharp guarantees at the reduced sample complexity.}
Our first main result (Theorem~\ref{thm:main}) proves that in the exact side-information setting, projected gradient descent with spectral initialization achieves sharp estimation error guarantees with the reduced sample complexity.

For sample complexity, we show that the projected gradient descent succeeds as long as
\[
    n^2 p
    \gtrsim  \mu_1 a \;\mathrm{poly}(\kappa, r, \mu_0, \log n),
\]
where $\mu_0$ and $\mu_1$ are incoherence parameters, $\kappa$ is the condition number, and $p$ is the sampling probability. 
For the estimation error, after convergence, we show an error guarantee of 
\[
\|\widehat{\bm{L}} - \bm{L}^{\star}\|_{\mathrm{F}}\lesssim \sigma \sqrt{\frac{\mu_1 a }{p}} \;\mathrm{poly}(\kappa, r, \log n).
\]
In comparison, ordinary matrix completion requires a sample complexity of $n^2 p \gtrsim n \;\mathrm{poly}(\mu_0, \kappa, r, \log n)$ and achieves an estimation error of $O(\sigma \sqrt{n/p})$ \citep{ma2018implicit}. Hence, side-information improves not only the sample complexity but also the noisy estimation error, replacing the ambient dimension $n$ by the effective dimension $\mu_1 a$.

The main technical ingredient behind this result is to establish a regularity condition for the nonconvex IMC loss at the reduced sample complexity. The key difficulty lies in the pattern mismatch between the grid-based observation set $\Omega$ and the side-information matrices $X$ and $Y$, which prohibits the application of the standard proof strategy from the ordinary matrix completion literature, such as leave-one-out arguments.

\paragraph{Inexact side-information with controlled inexactness.}
Our second main result (Theorem~\ref{thm:inexact}) extends the theory to the more realistic setting where the side-information subspaces are only approximately correct. We model this mismatch
by the side-information inexactness $\delta$, defined through the projection distance between the given side-information
subspaces and the true row and column spaces. We show that the same
reduced sample complexity as in Theorem~\ref{thm:main} remains sufficient under a mild additional regularity condition, while the estimation error incurs an additional $O(\delta)$ term. One cannot in general achieve an error smaller than this term in order.

\paragraph{Interpolation between IMC and ordinary matrix completion.}
Finally, we propose a penalized formulation for the inexact side-information setting that interpolates between ordinary matrix completion and IMC. The tuning parameter $\lambda$ controls the strength of the side-information penalty: when $\lambda=0$, the method reduces to ordinary matrix completion, while for $\lambda\to\infty$, it enforces stronger alignment with the side-information subspaces and tends to IMC. This combines the advantage of the inductive structure that allows reduced sample complexity with the robustness of ordinary matrix completion that avoids misspecification bias. Our simulations and MovieLens experiments show that the interpolated estimator can improve over both pure IMC and ordinary matrix completion in the inexact side-information regime.

\subsection{Related Work}

\paragraph{Matrix completion without side-information.}
The classical matrix completion problem is now well understood in both
noiseless and noisy regimes. Convex approaches based on nuclear norm
minimization established exact recovery under incoherence and near-optimal
sampling conditions in the noiseless setting
\citep{candes2009exact,candes2010power,recht2011simpler}. For noisy
observations, stable recovery and near-optimal estimation guarantees were
developed in
\citep{candes2010matrix,negahban2012restricted,chen2020noisy}. These
results serve as the non-inductive benchmark in which the sample complexity
and estimation error scale with the ambient matrix dimension.

A parallel line of work studies computationally efficient nonconvex methods
for matrix completion, including spectral initialization, alternating
minimization, gradient descent, leave-one-out analyses, and scaled gradient
methods
\citep{keshavan2009matrix,keshavan2010matrix,jain2013low,zheng2016convergence,
ma2018implicit,ding2020leave,tong2021accelerating}. These works provide the
algorithmic template for many modern analyses of low-rank matrix estimation,
including the projected-gradient approach studied here.

\paragraph{Inductive matrix completion with exact side-information.}
In the noiseless setting, early theoretical work for IMC established that exact recovery is possible
with substantially fewer observations
\citep{xu2013speedup,jain2013provable}. Subsequent work developed more
refined nonconvex algorithms for this setting. In particular,
\citet{zhang2018fast} analyzed a multi-phase Procrustes-flow method with
linear convergence, while
\citet{zilber2022inductive} studied the optimization landscape of IMC and
proved the absence of spurious local minima under suitable assumptions.

The understanding of IMC with noisy observations or inexact side-information
is less complete. \citet{chiang2015matrix} study noisy side-information via
a feature-driven component plus a residual component.
\citet{chen2022nonconvex} study nonconvex algorithms for noisy IMC and give a sharp statistical rate at a sample
complexity comparable to ordinary matrix completion. 
\citet{ledent2022generalization} takes a different perspective, proving
generalization bounds for IMC under low-noise assumptions and
Lipschitz-bounded losses, which are not directly comparable to our results.

\paragraph{Broader models using side-information.}
In addition to the subspace-based IMC model, there is also a broad literature
incorporating auxiliary information into matrix completion, with applications
to multi-label learning \citep{xu2013speedup,jain2013provable},
gene-disease association prediction
\citep{natarajan2014inductive,chen2018predicting,lu2018prediction,li2020neural},
recommendation \citep{shin2015tumblr, zhong2019provable}, link prediction \citep{zhao2022novel}, and graph-based matrix
completion \citep{zhang2019inductive, shen2021inductive}.

\subsection{Organization}

The rest of the paper is organized as follows. Section~\ref{sec:Problem-Setup}
introduces the model, assumptions, and algorithm. Section~\ref{sec:Main-result}
establishes our main result for noisy IMC with exact side-information.
Section~\ref{sec:Analysis-for-Theorem} gives the core regularity condition
and supporting analysis. Section~\ref{sec:Inexact-SI} extends the theory
to inexact side-information and introduces the interpolation scheme.
Section~\ref{sec:Experiments} presents simulation studies and real-data
experiments, and Section~\ref{sec:Discussion} concludes with discussion
and future directions.

\subsection{Notation}

Throughout the paper, bold uppercase letters such as $\bm{L}$ and $\bm{M}$
denote matrices, and $[n]=\{1,\dots,n\}$. For a matrix $\bm{A}$, we use
$\|\bm{A}\|$, $\|\bm{A}\|_{\mathrm{F}}$, and $\|\bm{A}\|_{\infty}$ to denote
its spectral norm, Frobenius norm, and entrywise maximum norm, respectively.
We write $\sigma_{j}(\bm{A})$ for the $j$th largest singular value of $\bm{A}$. 
The notation $\mathcal{P}_{\Omega}(\cdot)$ denotes the projection operator that
keeps entries in $\Omega$ and sets all others to zero. We use $c,C>0$ to denote
universal constants whose values may change from line to line, and we write
$a\lesssim b$ and $a\gtrsim b$ when $a\le Cb$ and $a\ge cb$, respectively.

%% file: sections/02-problem-setup.tex
\section{Problem Setup \label{sec:Problem-Setup}}

This section formalizes the model and assumptions in the exact side-information setting, introduces
the notation used throughout the paper, and presents the projected
gradient descent algorithm.

\subsection{Model formulation with exact side-information}
\label{subsec:model-formulation-exact-si}

We observe
\[
\bm{M}=\mathcal{P}_{\Omega}(\bm{L}^{\star}+\bm{E}),
\]
where $\bm{E}$ is the noise matrix and $\Omega \subset [n_1] \times [n_2]$ is the set of observed entries. 
In the exact side-information setting, we are given orthonormal matrices
$\bm{X}\in\mathbb{R}^{n_{1}\times a_{1}}$ and $\bm{Y}\in\mathbb{R}^{n_{2}\times a_{2}}$ that satisfy the following assumption.

\begin{assumption}\label{assumption:exact_SI}

$\mathrm{col}(\bm{L}^{\star})\subset\mathrm{col}(\bm{X})$ and $\mathrm{row}(\bm{L}^{\star})\subset\mathrm{col}(\bm{Y}).$

\end{assumption}

This immediately implies that $\bm{L}^{\star}=\bm{X}\bm{Z}^{\star}\bm{Y}^{\top}$
for some rank-$r$ matrix $\bm{Z}^{\star}\in\mathbb{R}^{a_{1}\times a_{2}}$
and the observed matrix can be written as 
\[
\bm{M}=\mathcal{P}_{\Omega}(\bm{X}\bm{Z}^{\star}\bm{Y}^{\top}+\bm{E}).
\]

We let $n\coloneqq n_{\max}\coloneqq n_{1}\vee n_{2}$ and $n_{\min}\coloneqq n_{1}\wedge n_{2}$,
where $\vee$ and $\wedge$ denote the maximum and minimum, respectively.
To enable the analysis, a recurring concept that is used in matrix
completion literature \citep[e.g.,][]{candes2009exact,candes2010matrix}
is incoherence. 

\begin{definition}[Incoherence] We say a rank-$k$ matrix $M \in \mathbb R^{n_1 \times n_2}$ is
$\mu$-incoherent if
\[
    \|\bm U\|_{2,\infty} \le \sqrt{\frac{\mu k}{n_1}},
    \qquad
    \|\bm V\|_{2,\infty} \le \sqrt{\frac{\mu k}{n_2}},
\]
where $\bm{U}\bm{\Sigma}\bm{V}^{\top}$ is the compact SVD of $\bm{M}$.
\end{definition} 
Here $\|\cdot\|_{2,\infty}$ is the two-to-infinity norm, or equivalently, the maximum $L_2$ row norm.
Note that since $\bm{X}$ is an orthonormal matrix,
its rank is $a_1$ and it is $\|\bm X\|_{2,\infty}$-incoherent. The same holds for $\bm Y$. A smaller incoherence parameter implies that the singular vectors are more spread out across rows. We then make the following assumptions.

\begin{assumption}Throughout this paper we assume the following properties
hold:

\begin{enumerate}[label={(\alph*)}, ref={Assumption~\theassumption(\alph*)}]
\label{assumption:main}

\item\label{assumption:sampling} \textbf{(Uniform sampling)} For each
entry $(i,j)\in[n_{1}]\times[n_{2}]$, $(i,j)$ is observed in $\Omega$ independently with
probability $p>0$.

\item  \label{assumption:singular} \textbf{(Singular values)} Denote
the maximum and minimum nonzero singular values of $\bm{Z}^{\star}$
as $\sigma_{\max}$, $\sigma_{\min}$ and let the condition number
be $\kappa\coloneqq\sigma_{\max}/\sigma_{\min}$.

\item  \label{assumption:incoherence} \textbf{(Incoherence)} We
assume $\bm{L}^{\star}$ is $\mu_{0}$-incoherent, $\bm{X}$ is $\mu_{1}$-incoherent,
and $\bm{Y}$ is $\mu_{2}$-incoherent.

\end{enumerate}

\end{assumption}

These assumptions are rather standard and analogous to those in ordinary matrix completion. In most cases we will treat $\kappa$ and $\mu_{0}$
as constant order parameters. One can also show that $1\le\mu_{1}\le n_{1}/a_{1}$
and $1\le\mu_{2}\le n_{2}/a_{2}$. As we will see later, $\mu_{1},\mu_{2}$
control how informative the side-information is: the smaller the quantities
$\mu_{1}a_{1}$ and $\mu_{2}a_{2}$ are, the smaller the effective problem
size becomes in our sample complexity bounds.

\subsection{Main Algorithm: Projected Gradient Descent}

Our algorithm is inspired by a standard nonconvex approach for ordinary
matrix completion, where a rank-$r$ matrix is parameterized as
$\bm{U}\bm{V}^{\top}$ and fitted to the observed entries by minimizing
\[
\min_{\bm{U}\in\mathbb{R}^{n_{1}\times r},\bm{V}\in\mathbb{R}^{n_{2}\times r}}
\frac{1}{2p}\|\mathcal{P}_{\Omega}(\bm{U}\bm{V}^{\top}-\bm{M})\|_{\mathrm{F}}^{2}.
\]
Despite its nonconvexity, this formulation is known to be both
statistically and computationally efficient for ordinary matrix
completion; see, for example,
\citet{keshavan2010matrix,jain2013low,zheng2016convergence}. In the
inductive setting, we instead use the side-information to parameterize
the estimate as $\bm{X}\bm{A}\bm{B}^{\top}\bm{Y}^{\top}$, thereby
optimizing over the lower-dimensional factors
$\bm{A}\in\mathbb{R}^{a_{1}\times r}$ and
$\bm{B}\in\mathbb{R}^{a_{2}\times r}$.

Under Assumption~\ref{assumption:exact_SI}, the ground truth can be
written as $\bm{L}^{\star}=\bm{X}\bm{Z}^{\star}\bm{Y}^{\top}$. Let
\[
\bm{Z}^{\star}
=
\bm{U}^{\star}\bm{\Sigma}^{\star}\bm{V}^{\star\top}
\]
be the compact SVD of $\bm{Z}^{\star}$, where
$\bm{U}^{\star}\in\mathbb{R}^{a_{1}\times r}$ and
$\bm{V}^{\star}\in\mathbb{R}^{a_{2}\times r}$. We write
\[
\bm{Z}^{\star}
=
\bm{A}^{\star}\bm{B}^{\star\top},
\qquad
\bm{A}^{\star}
\coloneqq
\bm{U}^{\star}\bm{\Sigma}^{\star 1/2},
\qquad
\bm{B}^{\star}
\coloneqq
\bm{V}^{\star}\bm{\Sigma}^{\star 1/2}.
\]
Motivated by this factorization, we estimate $\bm{Z}^{\star}$ by
minimizing the nonconvex objective
\begin{equation}
f(\bm{A},\bm{B})
\coloneqq
\frac{1}{2p}
\left\|
\mathcal{P}_{\Omega}
\left(
\bm{X}\bm{A}\bm{B}^{\top}\bm{Y}^{\top}-\bm{M}
\right)
\right\|_{\mathrm{F}}^{2}
+
\frac{1}{16}
\left\|
\bm{A}^{\top}\bm{A}-\bm{B}^{\top}\bm{B}
\right\|_{\mathrm{F}}^{2}.
\label{eq:loss_imc}
\end{equation}
The second term is a standard balancing regularizer that controls the
scaling ambiguity of the factorization $\bm{A}\bm{B}^{\top}$ by
encouraging the two factors to have matching Gram matrices.

We solve \eqref{eq:loss_imc} by projected gradient descent with spectral
initialization. 

The spectral initialization is computed by taking the top-$r$ SVD of
$p^{-1}\bm{X}^{\top}\mathcal{P}_{\Omega}(\bm{M})\bm{Y}$, which is a
natural estimate of the latent core $\bm{Z}^{\star}$. 
Let
\[
\bm{Z}^{0}
=
\bm{U}^{0}\bm{\Sigma}^{0}\bm{V}^{0\top}
\]
be the top-$r$ SVD of
$p^{-1}\bm{X}^{\top}\mathcal{P}_{\Omega}(\bm{M})\bm{Y}$. We set
\[
\bm{A}^{0}
=
\bm{U}^{0}(\bm{\Sigma}^{0})^{1/2},
\qquad
\bm{B}^{0}
=
\bm{V}^{0}(\bm{\Sigma}^{0})^{1/2}.
\]
This provides a good
starting point for the iterative gradient descent updates. 

For the gradient descent updates, we compute the gradients of
\eqref{eq:loss_imc} with respect to $\bm{A}$ and $\bm{B}$:
\begin{subequations}
\begin{align}
\nabla_{\bm{A}} f(\bm{A},\bm{B})
&=
\frac{1}{p}
\bm{X}^{\top}
\mathcal{P}_{\Omega}
\left(
\bm{X}\bm{A}\bm{B}^{\top}\bm{Y}^{\top}-\bm{M}
\right)
\bm{Y}\bm{B}
+
\frac{1}{4}
\bm{A}
\left(
\bm{A}^{\top}\bm{A}-\bm{B}^{\top}\bm{B}
\right),
\label{eq:grad_A_imc}
\\
\nabla_{\bm{B}} f(\bm{A},\bm{B})
&=
\frac{1}{p}
\bm{Y}^{\top}
\mathcal{P}_{\Omega}
\left(
\bm{X}\bm{A}\bm{B}^{\top}\bm{Y}^{\top}-\bm{M}
\right)^{\top}
\bm{X}\bm{A}
+
\frac{1}{4}
\bm{B}
\left(
\bm{B}^{\top}\bm{B}-\bm{A}^{\top}\bm{A}
\right).
\label{eq:grad_B_imc}
\end{align}
\end{subequations}
We also apply a projection step after the initialization and iterative update, with the projection set 
\[
\mathcal{C}
\coloneqq
\left\{
\begin{bmatrix}
\bm{A}\\
\bm{B}
\end{bmatrix}
:
\max
\left\{
\|\bm{X}\bm{A}\|_{2,\infty},
\|\bm{Y}\bm{B}\|_{2,\infty}
\right\}
\le
2\sqrt{\frac{\mu_{0}r}{n_{\min}}}\|\bm{A}^{0}\|
\right\}.
\]
The projection step is used to enforce incoherence control in the theoretical analysis. Empirically, however, the
incoherence constraint is never active in our experiments; see
Appendix~\ref{subsec:Redundancy-of-projection}. We therefore implement
Algorithm~\ref{alg:main} without the projection step in all experiments.
Proving the same convergence guarantee for unprojected gradient descent
remains an open question; see Section~\ref{sec:Discussion}.

\begin{algorithm}[t]
\caption{\label{alg:main}Projected gradient descent with spectral initialization for IMC}
\begin{algorithmic}[1]
\State \textbf{Input:} observed matrix $\bm{M}$, sampling set $\Omega$,
side-information matrices $\bm{X},\bm{Y}$, sampling probability $p$,
and step size $\eta$.
\State Compute $(\bm{A}^{0},\bm{B}^{0})$ from
the top-$r$ SVD of
$p^{-1}\bm{X}^{\top}\mathcal{P}_{\Omega}(\bm{M})\bm{Y}$, and set
$\begin{bmatrix}
\bm{A}^{1} \\
\bm{B}^{1}
\end{bmatrix}
=
\mathcal{P}_{\mathcal{C}}
\left(
\begin{bmatrix}
\bm{A}^{0} \\
\bm{B}^{0}
\end{bmatrix}
\right).
$
\For{$t=1,2,\ldots$}
\[
\begin{bmatrix}
\bm{A}^{t+1}\\
\bm{B}^{t+1}
\end{bmatrix}
=
\mathcal{P}_{\mathcal{C}}
\left(
\begin{bmatrix}
\bm{A}^{t}\\
\bm{B}^{t}
\end{bmatrix}
-
\eta
\begin{bmatrix}
\nabla_{\bm{A}}f(\bm{A}^{t},\bm{B}^{t})\\
\nabla_{\bm{B}}f(\bm{A}^{t},\bm{B}^{t})
\end{bmatrix}
\right).
\]
\EndFor
\State \textbf{Output:} $\bm{Z}^{t}=\bm{A}^{t}\bm{B}^{t\top}$ and
$\bm{L}^{t}=\bm{X}\bm{Z}^{t}\bm{Y}^{\top}$.
\end{algorithmic}
\end{algorithm}

%% file: sections/03-exact-side-information.tex
\section{Exact Side-Information \label{sec:Main-result}}

In this section we present the main theoretical guarantee for Algorithm~\ref{alg:main}
under the exact side-information assumption (Assumption~\ref{assumption:exact_SI}) and highlight some key components of the analysis.

\subsection{Main Result}
For a general noise matrix $\bm{E}$, we define the \emph{effective noise size} as
\[
\gamma_{E}\coloneqq\frac{1}{p}\left\Vert \bm{X}^{\top}\mathcal{P}_{\Omega}(\bm{E})\bm{Y}\right\Vert ,
\]
which captures the magnitude of the noise $\bm{E}$ projected onto
the side-information and partial observation subspaces. The following theorem characterizes the sample complexity and estimation error of Algorithm~\ref{alg:main}.

\begin{theorem}\label{thm:main} Suppose that Assumption~\ref{assumption:exact_SI}
and \ref{assumption:main} hold. In addition, assume that the sample
size and the noise level satisfy
\begin{equation}
n_{1}n_{2}p\ge C_{\mathrm{sample}}\kappa^{2}\left(\mu_{1}a_{1}\vee\mu_{2}a_{2}\right)\mu_{0}^{2}r^{2}\log n,\qquad\text{and}\qquad\frac{\gamma_{E}}{\sigma_{\min}}\le C_{\mathrm{noise}}\frac{1}{\sqrt{\kappa r}}\label{eq:sample_noise_assumption}
\end{equation}
for some sufficiently large (resp.~small) constant $C_{\mathrm{sample}}>0$
(resp.~$C_{\mathrm{noise}}>0$). With probability at least $1-O(n^{-10})$,
if $\eta\le (480000\left(n/n_{\min}\right)^{2}\mu_{0}^{2}r^{2}\kappa\sigma_{\max})^{-1}$,
the iterates $\bm{L}^{t}$ of Algorithm~\ref{alg:main} satisfy
\begin{align}
\|\bm{L}^{t}-\bm{L}^{\star}\|_{\mathrm{F}} & \le C_{\mathrm{F}}\left[\rho^{t-1}+\sqrt{\kappa}\cdot\frac{\gamma_{E}}{\sigma_{\min}}\right]\|\bm{L}^{\star}\|_{\mathrm{F}},\label{eq:L-bound-fro}
\end{align}
where $\rho=(1-\eta\sigma_{\min}/20)^{1/2}$ and $C_{\mathrm{F}}>0$
is an absolute constant.
\end{theorem}
The proof is deferred to Section~\ref{sec:Analysis-for-Theorem}, where we highlight
the two key lemmas underlying the result. The following corollary interprets this result when the noise is entrywise i.i.d. sub-Gaussian.

\begin{corollary}\label{cor:main_subgaussian}
Assume that $[\bm{E}]_{ij}$ are i.i.d.\ mean-zero sub-Gaussian random variables
with standard variance proxy $\sigma^{2}$
\citep[see, e.g.,][]{wainwright2019high}. Suppose the assumptions of Theorem~\ref{thm:main} hold, except that
the noise level condition in \eqref{eq:sample_noise_assumption}
is replaced by 
\[
\frac{\sigma}{\sigma_{\min}}\sqrt{\frac{(\mu_{1}a_{1}\vee\mu_{2}a_{2})\log n}{p}}
\le c_{\mathrm{noise}}\frac{1}{\sqrt{\kappa r}}
\]
for a sufficiently small absolute constant $c_{\mathrm{noise}}>0$,
then with probability at least $1-O(n^{-10})$, the iterates
$\bm{L}^{t}$ of Algorithm~\ref{alg:main} satisfy
\[
\|\bm{L}^{t}-\bm{L}^{\star}\|_{\mathrm{F}}
\le C_{\mathrm{G}}
\left[
\rho^{t-1}
+
\frac{\sigma}{\sigma_{\min}}
\sqrt{\frac{(\mu_{1}a_{1}\vee\mu_{2}a_{2}) \kappa\log n}{p}}
\right]\|\bm{L}^{\star}\|_{\mathrm{F}},
\]
where $\rho=(1-\eta\sigma_{\min}/20)^{1/2}$ and
$C_{\mathrm{G}}>0$ is an absolute constant.
\end{corollary}
The proof is deferred to Appendix~\ref{subsec:proof-of-main-subgaussian-corollary}.
To understand these results more intuitively, we make the following remarks in the symmetric case
\begin{equation}
n_{1}=n_{2}=n,\quad a_{1}=a_{2}=a,\quad\mu_{1}=\mu_{2},\label{eq:simplified_case}
\end{equation}
and treat $\kappa$, $\mu_{0}$, $r$ as constants throughout.

\paragraph{Sample complexity.} Under the simplified case~\eqref{eq:simplified_case},
Theorem~\ref{thm:main} requires $n^{2}p\gtrsim\mu_{1}a\mu_{0}^{2}r^{2}\log n$.
Compared to the standard sample complexity $n^{2}p\gtrsim n\mu_{0}r\log n$
in ordinary noisy matrix completion \citep{negahban2012restricted,chen2020noisy},
the ambient dimension $n$ is replaced by the effective problem size
$\mu_{1}a$, which ranges from $a$ (when $\bm{X}$ is perfectly spread)
to $n$ (when side-information provides no benefit). This indicates a significantly reduced sample complexity requirement when informative side-information is available.

\paragraph{Estimation error.} 
Since $X$ and $Y$ have orthonormal columns, we always have
\[
    \gamma_E
    =
    \frac{1}{p}\left\|X^\top \mathcal P_\Omega(E)Y\right\|
    \le
    \frac{1}{p}\left\|\mathcal P_\Omega(E)\right\|.
\]
This shows that the estimation error depends only on the part of
the sampled noise seen through the side-information subspaces, shrinking the effective noise level.

This can be observed quantitatively for entrywise i.i.d. sub-Gaussian noise.
In the simplified case \eqref{eq:simplified_case}, taking $\kappa,\mu_{0},r$
as constants and taking $t$ large enough, Corollary~\ref{cor:main_subgaussian} shows that 
\[
    \frac{\|\bm{L}^{t}-\bm{L}^{\star}\|_{\mathrm{F}}}{\|\bm{L}^{\star}\|_{\mathrm{F}}} 
    \lesssim
    \frac{\sigma}{\sigma_{\min}}\sqrt{\frac{\mu_{1}a\log n}{p}}.
\]
In comparison, the optimal rate in ordinary noisy matrix completion \citep{chen2020noisy} is 
\[
    \frac{\|\bm{L}^{t}-\bm{L}^{\star}\|_{\mathrm{F}}}{\|\bm{L}^{\star}\|_{\mathrm{F}}} 
    \lesssim
    \frac{\sigma}{\sigma_{\min}}\sqrt{\frac{n}{p}}.\]
Thus the inclusion of side-information improves the estimation error guarantee, with the ambient $n$ replaced by the effective
problem size $\mu_{1}a$.

\paragraph{Comparison with prior IMC results.} In the noiseless setting, 
\citet{chen2015incoherence} achieves exact recovery
with sample complexity $n^{2}p\gtrsim\mu_{1}a\mu_{0}r\log n$. Our result matches this up to a $\mu_0 r$ factor. In the noisy setting, \citet{chen2022nonconvex} achieves near-optimal
estimation error but only at the suboptimal
sample complexity $n^{2}p\gtrsim n\mu_{0}r\log n$. Our result combines these two fronts by achieving both low estimation
error and the reduced sample complexity $n^{2}p\gtrsim\mu_{1}a\mu_{0}^{2}r^{2}\log n$.

%% file: sections/04-analysis.tex
\subsection{Analysis \label{sec:Analysis-for-Theorem}}

The proof of Theorem~\ref{thm:main} follows a standard two-step analysis for nonconvex
low-rank matrix estimation. First, we show that the spectral initializer
enters a local basin around the ground truth. Second, we show that projected
gradient descent contracts linearly inside this basin, up to a statistical
error floor determined by the observation noise.

The main technical difficulty is that the inductive parametrization changes the
space on which the sampling operator must be controlled. In ordinary matrix
completion, the operator $\mathcal{P}_{\Omega}$ acts directly on the
ambient low-rank matrix. In IMC, the algorithm optimizes over the
lower-dimensional core matrix $\bm{Z}$, but the observations are sampled
entrywise only after the lifting map
$\bm{Z}\mapsto \bm{X}\bm{Z}\bm{Y}^{\top}$. 
This mismatch between the optimization coordinates and the observation
coordinates makes it necessary to show that random sampling still provides
the local curvature required by the IMC loss at the reduced sample size.

We will introduce two lemmas in this section before the proof of Theorem~\ref{thm:main} that formalize these ingredients. Lemma~\ref{lem:spec_init}
shows that the spectral initialization lies inside the local basin.
Lemma~\ref{lem:rc} is the main local-geometry result: it proves that,
throughout this basin, the IMC loss satisfies a regularity condition that
implies contraction of projected gradient descent.

We begin by introducing the notation used to measure distance between
factorizations. For any two matrices $\bm{A}\in\mathbb{R}^{a_1\times r}$
and $\bm{B}\in\mathbb{R}^{a_2\times r}$, let their vertical concatenation be
\[
    \bm{F}:=\begin{bmatrix} \bm{A} \\ \bm{B} \end{bmatrix}.
\]
Because the factorization
$\bm{Z}^\star=\bm{A}^\star\bm{B}^{\star\top}$ is invariant under a common
orthogonal rotation of the factors, we measure the factor error modulo this
rotation. Define
\[
    \bm{H}
    := \arg\min_{\bm{R}\in\mathcal{O}^{r\times r}}
        \|\bm{F}\bm{R}-\bm{F}^\star\|_{\mathrm F}
    = \arg\min_{\bm{R}\in\mathcal{O}^{r\times r}}
        \|\bm{F}-\bm{F}^\star \bm{R}^{\top}\|_{\mathrm F}.
\]
Abusing notation slightly, we write $f(\bm{F})=f(\bm{A},\bm{B})$, and
define
\[
    \bar{\bm{A}} := \bm{A}^\star \bm{H}^{\top},\qquad
    \bar{\bm{B}} := \bm{B}^\star \bm{H}^{\top},\qquad
    \bar{\bm{F}} := \bm{F}^\star \bm{H}^{\top},\qquad
    \bm{\Delta} := \bm{F}-\bar{\bm{F}}.
\]
We also write
$\bm{\Delta}_{A}:=\bm{A}-\bar{\bm{A}}$ and
$\bm{\Delta}_{B}:=\bm{B}-\bar{\bm{B}}$. When referring to the $t$-th
iterate, we add a superscript $t$ to these quantities.

With this notation in place, we first state the spectral initialization
guarantee. The proof is deferred to Appendix~\ref{sec:Proof_of_spec_init}.
\begin{lemma} \label{lem:spec_init}Suppose that Assumption~\ref{assumption:main}
    holds. In addition, assume that the sample size and the noise level
    satisfy
    \[
        n_{1}n_{2}p\ge C_{\mathrm{sample}}\kappa^{2}\left(\mu_{1}a_{1}\vee\mu_{2}a_{2}\right)\mu_{0}^{2}r^{2}\log n,\qquad\text{and}\qquad\frac{\gamma_{E}}{\sigma_{\min}}\le C_{\mathrm{noise}}\frac{1}{\sqrt{\kappa r}},
    \]
    for some sufficiently large (resp.~small) constant $C_{\mathrm{sample}}>0$
    (resp.~$C_{\mathrm{noise}}>0$). Then the spectral initialization
    $\bm{A}^{1},\bm{B}^{1}$ satisfy
    \[
        \|\bm{\Delta}^{1}\|_{\mathrm{F}}\le\sqrt{\frac{20r}{\sigma_{\min}}}\left(C_{\mathrm{spectral}}\|\bm{Z}^{\star}\|\sqrt{\frac{\mu_{0}r\left(\mu_{1}a_{1}\vee\mu_{2}a_{2}\right)}{n_{1}n_{2}p}\log n}+\gamma_{E}\right)\le\frac{1}{10}\sqrt{\sigma_{\min}}.
    \]
    Moreover $\|\bm A^0\|\le 2\|\bm{A}^\star\|$.

\end{lemma}

For fast convergence in the gradient descent steps, the loss function
must exhibit favorable local geometry. We will show that for the projected gradient descent step at iteration $t$, the error satisfies
\[ \|\bm{\Delta}^{t+1}\|_{\mathrm F}^{2} \le \|\bm{\Delta}^{t}\|_{\mathrm F}^{2} +\eta^{2}\|\nabla f(\bm{F}^{t})\|_{\mathrm F}^{2} -2\eta\langle \nabla f(\bm{F}^{t}),\bm{\Delta}^{t}\rangle . \]
This recursive relation inspires the following regularity condition.
\begin{definition}[Regularity condition] We say a function $f$ satisfies
    regularity condition $\mathrm{RC}(\alpha,\beta,\gamma)$ at $\bm{F}$
    if
    \[
        \left\langle \nabla f(\bm{F}),\bm{\Delta}\right\rangle \ge\frac{1}{\alpha}\left\Vert \bm{F}-\bar{\bm{F}}\right\Vert _{\mathrm{F}}^{2}+\frac{1}{\beta}\|\nabla f(\bm{F})\|_{\mathrm{F}}^{2}-\gamma.
    \]
\end{definition}
The first term on the right-hand side describes the local curvature which gives contraction in the factor error, while the
second term absorbs the squared gradient norm $\|\nabla f(\bm{F})\|_{\mathrm{F}}^{2}$ in the recursion. The additive term $\gamma$ accounts for the perturbation
caused by noise and determines the final statistical error
floor.

The next lemma verifies
this regularity condition throughout the local basin. In particular, the incoherence condition it requires can be satisfied with the projection step.
The proof is deferred to Appendix~\ref{sec:Proof-of-regularity}.

\begin{lemma}\label{lem:rc} Suppose that
    \[
        n_{1}n_{2}p\ge C_{\mathrm{sample}}\kappa^{2}\left(\mu_{1}a_{1}\vee\mu_{2}a_{2}\right)\mu_{0}^{2}r^{2}\log n\qquad\text{and}\qquad\frac{\gamma_{E}}{\sigma_{\min}}\le C_{\mathrm{noise}}\frac{1}{\sqrt{\kappa r}}
    \]
    for a large enough constant $C_{\mathrm{sample}}>0$ and a small enough constant $C_{\mathrm{noise}}>0$. 
    If $\bm{F}=\begin{bmatrix}\bm{A} \\
            \bm{B}
        \end{bmatrix}$ satisfies
    \begin{subequations}\label{eq:rsc_assumption}
    \begin{align}
        \left\Vert \bm{\Delta}\right\Vert _{\mathrm{F}} & \le\frac{1}{10}\sqrt{\sigma_{\min}},\label{eq:rsc_assumption_distance}\\
        \|\bm{X}\bm{A}\|_{2,\infty}                     & \le4\sqrt{\frac{\mu_{0}r}{n_{\min}}}\|\bm{A}^{\star}\|,\label{eq:rsc_assumption_incoherence_A}\\
        \|\bm{Y}\bm{B}\|_{2,\infty}                     & \le4\sqrt{\frac{\mu_{0}r}{n_{\min}}}\|\bm{A}^{\star}\|,\label{eq:rsc_assumption_incoherence_B}
    \end{align}
    \end{subequations}
    then $f$ satisfies regularity condition $\mathrm{RC}(40/\sigma_{\min},960000\left(n/n_{\min}\right)^{2}\mu_{0}^{2}r^{2}\kappa\sigma_{\max},21r\kappa\gamma_{E}^{2})$
    at $\bm{F}$.
\end{lemma}

The full proof of Theorem~\ref{thm:main} is deferred to Appendix~\ref{subsec:proof-of-main-theorem}. We present a sketch here.

\paragraph{Proof sketch of Theorem~\ref{thm:main}.}

The proof follows from Lemmas~\ref{lem:spec_init} and~\ref{lem:rc}
through a standard induction argument. Lemma~\ref{lem:spec_init} shows
that the projected spectral initializer $\bm{F}^{1}$ satisfies the incoherence conditions required by Lemma~\ref{lem:rc}.

Suppose inductively that the same conditions hold for $\bm{F}^{k}$.
Since the aligned ground truth
$\bar{\bm{F}}^{k}=\bm{F}^{\star}(\bm{H}^{k})^{\top}$ belongs to the
projection set $\mathcal{C}$, the projection step cannot increase the
distance to $\bar{\bm{F}}^{k}$. Hence
\begin{align*}
    \|\bm{\Delta}^{k+1}\|_{\mathrm{F}}^{2}
    &\le
    \|\bm{F}^k-\eta\nabla f(\bm{F}^k)-\bar{\bm{F}}^{k}\|_{\mathrm{F}}^{2} \\
    &= \|\bm{\Delta}^k - \eta \nabla f(\bm{F}^k)\|_{\mathrm{F}}^{2} \\
    &=
    \|\bm{\Delta}^{k}\|_{\mathrm{F}}^{2}
    +\eta^{2}\|\nabla f(\bm{F}^{k})\|_{\mathrm{F}}^{2}
    -2\eta
    \left\langle
        \nabla f(\bm{F}^{k}),\bm{\Delta}^{k}
    \right\rangle .
\end{align*}
Let the regularity condition in Lemma~\ref{lem:rc} be $\mathrm{RC}(\alpha,\beta,\gamma)$.
Applying the regularity condition and using the
stepsize condition $\eta\le 2/\beta$ gives the recursive relation
\[
    \|\bm{\Delta}^{k+1}\|_{\mathrm{F}}^{2}
    \le
    \left(1-\frac{2\eta}{\alpha}\right)
    \|\bm{\Delta}^{k}\|_{\mathrm{F}}^{2}
    +
    2\eta\gamma.
\]
One can then check that this, together with the projection step, preserves the conditions in Lemma~\ref{lem:rc}. 

Applying the recursion repeatedly gives
\[
    \|\bm{\Delta}^{t}\|_{\mathrm{F}}^{2}
    \le
    \left(1-\frac{2\eta}{\alpha}\right)^{t-1}
    \|\bm{\Delta}^{1}\|_{\mathrm{F}}^{2}
    +
    \alpha \gamma .
\]
Finally we make the connection that 
\[
    \|\bm{Z}^{t}-\bm{Z}^{\star}\|_{\mathrm{F}}^{2}
    \le
    6\sigma_{\max}\|\bm{\Delta}^{t}\|_{\mathrm{F}}^{2}.
\]
Combining this with the bound for spectral initialization and the fact that
$\|\bm{X}\bm{Z}\bm{Y}^{\top}\|_{\mathrm{F}}=\|\bm{Z}\|_{\mathrm{F}}$
gives the bound for
$\|\bm{L}^{t}-\bm{L}^{\star}\|_{\mathrm{F}}$. 

%% file: sections/05-inexact-side-information.tex
\section{Inexact Side-Information \label{sec:Inexact-SI}}

The exact side-information theory in Theorem~\ref{thm:main} assumes that
\[
    \operatorname{col}(\bm L^\star)\subseteq \operatorname{col}(\bm X),
    \qquad
    \operatorname{row}(\bm L^\star)\subseteq \operatorname{col}(\bm Y).
\]
In practice, the available side-information subspaces may only approximate
the true column and row spaces of $\bm L^\star$. We now extend our theory to
this inexact side-information setting.

The key observation is that even when $\bm X$ and $\bm Y$ are misspecified,
the projected matrix
\[
    \bm L_{\mathrm{proj}}^\star
    :=
    \bm X\bm X^\top \bm L^\star \bm Y\bm Y^\top
\]
is still described by the inductive model in
Section~\ref{subsec:model-formulation-exact-si}. The remaining component
\[
    \bm E_{\mathrm{mis}}
    :=
    \bm L^\star-\bm X\bm X^\top \bm L^\star \bm Y\bm Y^\top
\]
can therefore be interpreted as an additional noise term caused by
side-information misspecification.

The rest of this section formalizes the inexact side-information setting.
The main result in Theorem~\ref{thm:inexact} characterizes the
additional estimation error caused by the inexact side-information in
terms of the side-information inexactness $\delta$.

\subsection{Setup}

Let $\bm L^\star=\bm U^\star\bm\Sigma^\star\bm V^{\star\top}$ be
the rank-$r$ singular value decomposition of $\bm L^\star$. Define
the orthogonal projections onto the observed side-information subspaces by
\[
    \bm P_X\coloneqq \bm X\bm X^\top,
    \qquad
    \bm P_Y\coloneqq \bm Y\bm Y^\top.
\]
We quantify side-information inexactness through the following projection-distance
condition.

\begin{assumption}[Side-information inexactness]\label{assumption:inexact_SI}
The side-information inexactness is
\[
    \delta_X\coloneqq \|(\bm I-\bm P_X)\bm U^\star\|,
    \qquad
    \delta_Y\coloneqq \|(\bm I-\bm P_Y)\bm V^\star\|,
    \qquad
    \delta\coloneqq \delta_X\vee\delta_Y.
\]
\end{assumption}

The scalar $\delta$ measures how much of the true column and row spaces
falls outside the observed side-information subspaces. Thus $\delta=0$
recovers exact side-information.

This projection-distance definition is equivalent to a
principal-angle formulation. If $\bm\Theta_1$ and $\bm\Theta_2$ are the
principal-angle matrices between
$\operatorname{col}(\bm X),\operatorname{col}(\bm U^\star)$ and
$\operatorname{col}(\bm Y),\operatorname{col}(\bm V^\star)$, respectively,
then $\delta_X=\|\sin\bm\Theta_1\|$ and
$\delta_Y=\|\sin\bm\Theta_2\|$.

The same quantity controls the residual caused by inexact side-information:
\begin{equation}
    \frac{\delta}{\sqrt{r}\kappa}\|\bm L^\star\|_{\mathrm F}
    \le
    \|\bm E_{\mathrm{mis}}\|_{\mathrm F}
    \le \sqrt{2}\,\delta\|\bm L^\star\|_{\mathrm F}.
    \label{eq:projection-residual-bound}
\end{equation}
Hence, when $r$ and $\kappa$ are treated as constants, this residual is
of order $\delta\|\bm L^\star\|_{\mathrm F}$. The proof of
\eqref{eq:projection-residual-bound} is given at the end of
Appendix~\ref{sec:Proof-of-inexact}.

We will also use the following incoherence condition on the projected
target and the misspecification of the side-information subspaces.

\begin{assumption}\label{assumption:inexact_SI_2inf}
The projected target $\bm L_{\mathrm{proj}}^\star$ is $\mu_{0}$-incoherent. 
In addition,
there exists a constant $C_{\mathrm{mis}}>0$ such that
\begin{equation}
\|(\bm I-\bm P_X)\bm U^\star\|_{2,\infty}\le C_{\mathrm{mis}}\sqrt{\frac{\mu_{0}r}{n_{1}}}\cdot\delta
\qquad \text{and} \qquad
\|(\bm I-\bm P_Y)\bm V^\star\|_{2,\infty}\le C_{\mathrm{mis}}\sqrt{\frac{\mu_{0}r}{n_{2}}}\cdot\delta.
\label{eq:2infty-inexact-si}
\end{equation}
\end{assumption}

\subsection{Main Result}

The theorem below shows that, under the inexact-side-information
conditions in Assumptions~\ref{assumption:inexact_SI}--\ref{assumption:inexact_SI_2inf},
Algorithm~\ref{alg:main}
maintains the inductive sample complexity from Theorem~\ref{thm:main}
even when the side-information is inexact. The resulting estimation error
degrades linearly with the side-information inexactness $\delta$.
Recall the effective noise size
\[
\gamma_{E}\coloneqq\frac{1}{p}\left\Vert \bm{X}^{\top}\mathcal{P}_{\Omega}(\bm{E})\bm{Y}\right\Vert .
\]

\begin{theorem}\label{thm:inexact} Suppose that Assumptions~\ref{assumption:main}, \ref{assumption:inexact_SI},
and~\ref{assumption:inexact_SI_2inf} hold. Assume that the sample size
and the effective noise level satisfy
\begin{equation}
n_{1}n_{2}p\ge C_{\mathrm{sample}}\kappa^{2}\left(\mu_{1}a_{1}\vee\mu_{2}a_{2}\right)\mu_{0}^{2}r^{2}\log n,\qquad\text{and}\qquad
\kappa\sqrt{\frac{\mu_{0}r\left(\mu_{1}a_{1}\vee\mu_{2}a_{2}\right)}{n_{1}n_{2}p}\log n}\cdot\delta+\frac{\gamma_{E}}{\sigma_{\min}}\le C_{\mathrm{noise}}\frac{1}{\sqrt{\kappa r}}\label{eq:sample-noise-cond-inexact}
\end{equation}
for some sufficiently large (resp.~small) constant $C_{\mathrm{sample}}>0$
(resp.~$C_{\mathrm{noise}}>0$). With probability at least $1-O(n^{-10})$,
if $\eta\le (480000\left(n/n_{\min}\right)^{2}\mu_{0}^{2}r^{2}\kappa\sigma_{\max})^{-1}$,
the iterates $\bm{L}^{t}$ of Algorithm~\ref{alg:main} satisfy
\begin{align}
\|\bm{L}^{t}-\bm{L}^{\star}\|_{\mathrm{F}} & \le C_{\mathrm{F}}\left[\rho^{t}+\sqrt{\kappa}\left(\delta+\frac{\gamma_{E}}{\sigma_{\min}}\right)\right]\|\bm{L}^{\star}\|_{\mathrm{F}}.\label{eq:L-bound-fro-extra-assumption}
\end{align}
where $\rho=(1-\eta\sigma_{\min}/20)^{1/2}$ and $C_{\mathrm{F}}>0$
is an absolute constant.
\end{theorem}

The proof of Theorem~\ref{thm:inexact} is deferred to
Appendix~\ref{sec:Proof-of-inexact}.
We make several remarks using the simplified square $\bm{L}^\star$ setup~\eqref{eq:simplified_case}
and treating $\kappa,\mu_{0},r$ as constants.

\paragraph{Tradeoff comparing to ordinary MC.}
Theorem~\ref{thm:inexact} shows the tradeoff of IMC in the inexact
side-information. IMC can succeed at the reduced sample
complexity where ordinary MC fails. However, the bias induced by the inexact side-information inflates the estimation error by a $\Omega(\delta)$ term. 
Consequently, contrary to the exact side-information setting where IMC always outperforms ordinary MC, ordinary MC can perform better in the inexact setting once enough samples are available.

\paragraph{Admissible side-information inexactness.} Under
Assumption~\ref{assumption:inexact_SI_2inf} and the minimal
sample complexity $n^{2}p\gtrsim\mu_{0}^{2}r^{2}\mu_{1}a\log n$, the
condition in Theorem~\ref{thm:inexact} permits $\delta\lesssim1$.
While bound~\eqref{eq:L-bound-fro-extra-assumption} is most informative
when $\delta=o(1)$, the analysis permits
side-information inexactness up to constant order.

\paragraph{Sharpness of the dependence on side-information inexactness.}
When $\bm{E}=\bm{0}$ and $\|\bm{L}^{\star}\|_{\mathrm{F}}=\Theta(1)$,
Theorem~\ref{thm:inexact} achieves $O(\delta)$ error with
$\Theta(\mu_{0}^{2}r^{2}\mu_{1}a\log n)$ samples. 
We make a heuristic argument that this is the sharp dependence at the reduced sample complexity. 

Recall that $\bm E_{\mathrm{mis}}$ is the error induced by inexact side-information and it has a size of $\Theta(\delta)$. By the contraction property of projection, any estimator that has $o(\delta)$ estimation error of $\bm L^\star$ must have $o(\delta)$ estimation error of $\bm E_{\mathrm{mis}}$ as well. However, as $\bm E_{\mathrm{mis}}$ lives in the ambient space, one cannot hope to do so with the reduced sample complexity. 

\subsection{Interpolation Scheme}

The preceding results suggest a bias--variance tradeoff under inexact
side-information. IMC improves sample
efficiency but introduces bias when those subspaces are inexact.
Ordinary MC has no such side-information bias, but it estimates the
matrix in the ambient space and therefore requires more observations.

We use a penalized formulation to interpolate between these two
estimators. Rather than enforcing exact alignment with
$\operatorname{col}(\bm X)$ and $\operatorname{col}(\bm Y)$, we penalize
the component of the estimate outside these subspaces. Specifically, for
$\bm A\in\mathbb{R}^{n_{1}\times r}$ and
$\bm B\in\mathbb{R}^{n_{2}\times r}$, consider
\begin{equation}
f_{\lambda}(\bm A,\bm B)=\frac{1}{2p}\left\Vert \mathcal{P}_{\Omega}\left(\bm A\bm B^{\top}-\bm{M}\right)\right\Vert _{\mathrm{F}}^{2}+\frac{1}{16}\|\bm A^{\top}\bm A-\bm B^{\top}\bm B\|_{\mathrm{F}}^{2}+\lambda\left\Vert \bm A\bm B^{\top}-\bm P_X\bm A\bm B^{\top}\bm P_Y\right\Vert _{\mathrm{F}}^{2}.\label{eq:loss_interpolation}
\end{equation}
The tuning parameter $\lambda\ge0$ controls the weight of the
side-information penalty. When $\lambda=0$, \eqref{eq:loss_interpolation}
is the usual non-inductive MC objective. In the limit $\lambda\to\infty$,
the penalty enforces
$\bm A\bm B^\top=\bm P_X\bm A\bm B^\top\bm P_Y$,
so the limiting
formulation recovers the IMC model.

In Section~\ref{sec:Experiments}, we evaluate this
interpolation scheme and select $\lambda$ by cross-validation.
Empirically, intermediate values of $\lambda$ can improve over both
pure IMC and ordinary MC in the inexact side-information regime.

%% file: sections/06-experiments.tex
\section{Experiments \label{sec:Experiments}}

We validate our theoretical findings through two sets of experiments using the empirical implementation described after Algorithm~\ref{alg:main}, namely gradient descent with spectral initialization and without the projection step.
First, we run synthetic simulations to compare
IMC with standard MC across a range of sampling rates in both the exact and inexact side-information settings.
We demonstrate the strength of IMC at small sample sizes and its graceful degradation under inexact side-information.
Second, we evaluate our method on the MovieLens 100K dataset to further illustrate the practical benefits of IMC.

\subsection{Simulations \label{subsec:Simulations}}

\subsubsection{Simulation setup}
In all synthetic experiments, the ground truth matrix $\bm{L}^{\star}$
has dimensions $n_{1}=n_{2}=1000$ and rank $r=10$, and the side-information
dimension is $a_{1}=a_{2}=50$. We generate orthonormal side-information matrices
$\bm{X}^{\star},\bm{Y}^{\star}\in\mathbb{R}^{n\times a}$ and a rank-$r$ matrix $\bm{Z}\in\mathbb{R}^{a\times a}$ with spectral norm 1,
and define
\[
\bm{L}^{\star}=\bm{X}^{\star}\bm{Z}{\bm{Y}^{\star}}^{\top}.
\]
We then observe
\[
\bm{M}=\mathcal{P}_{\Omega}(\bm{L}^{\star}+\bm{E}),
\]
where each entry is observed independently with probability $p$, and the entries of $\bm{E}$ are i.i.d.\ Gaussian with mean $0$
and standard deviation $0$ in the noiseless case and $\sigma=0.001$ in the noisy case. For inexact side-information, we replace
$\bm{X}^{\star},\bm{Y}^{\star}$ by matrices $\bm{X},\bm{Y}$ with prescribed
side-information inexactness $\delta$. Due to space constraints, we defer the detailed construction to Appendix~\ref{subsec:detailed-simulation-setup}.

\subsubsection{Exact side-information}

First, we perform simulations under exact side-information. We first consider noiseless matrix observations. We compare the IMC implementation described above with the standard MC method, which uses gradient descent with spectral initialization to optimize the MC objective:
\begin{equation}
f_{\mathrm{MC}}(\bm{A},\bm{B})=\frac{1}{2p}\|\mathcal{P}_{\Omega}(\bm{A}\bm{B}^{\top}-\bm{M})\|_{\mathrm{F}}^{2}+\frac{1}{16}\|\bm{A}^{\top}\bm{A}-\bm{B}^{\top}\bm{B}\|_{\mathrm{F}}^{2}.\label{eq:loss_mc}
\end{equation}

For the noiseless setting, Figure~\ref{fig:phase_plot} shows the success
rates of exact recovery\footnote{up to a relative error of $10^{-3}$} for MC and IMC for various values of $p$. We see that IMC always
achieves exact recovery for all values of $p$ above 0.01, while MC
has zero success at low values of $p$. Figure~\ref{fig:phase_plot_imc}
further shows that IMC achieves a perfect success rate for $p$ above
0.0014, which is much lower than the cutoff point for MC. This agrees with our theory that IMC, by incorporating side-information, allows a significantly smaller sample complexity.

\begin{figure}[p]
\centering
\captionsetup{font=footnotesize}
\setlength{\abovecaptionskip}{2pt}
\setlength{\belowcaptionskip}{2pt}
\begin{minipage}[t]{0.48\textwidth}
\centering
\includegraphics[width=0.96\textwidth]{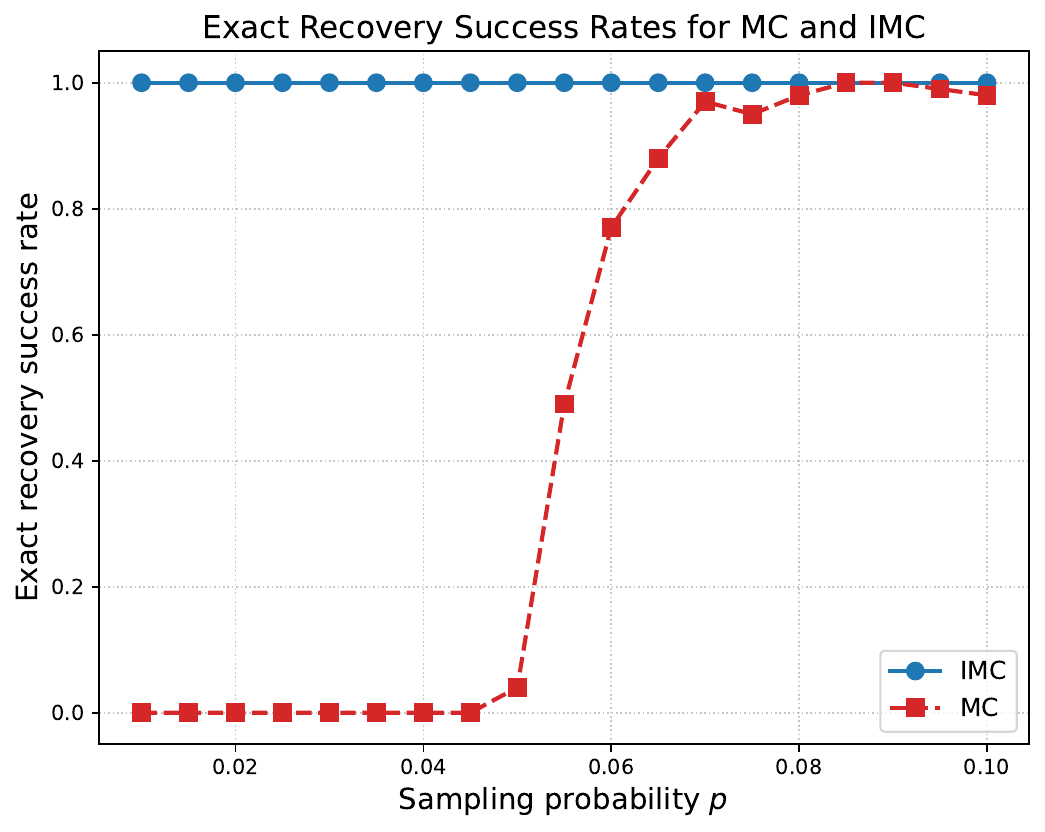}
\captionof{figure}{\label{fig:phase_plot}Exact recovery success rates for MC and IMC
in the noiseless exact side-information setting, averaged over 100
trials.}
\end{minipage}\hfill
\begin{minipage}[t]{0.48\textwidth}
\centering
\includegraphics[width=0.96\textwidth]{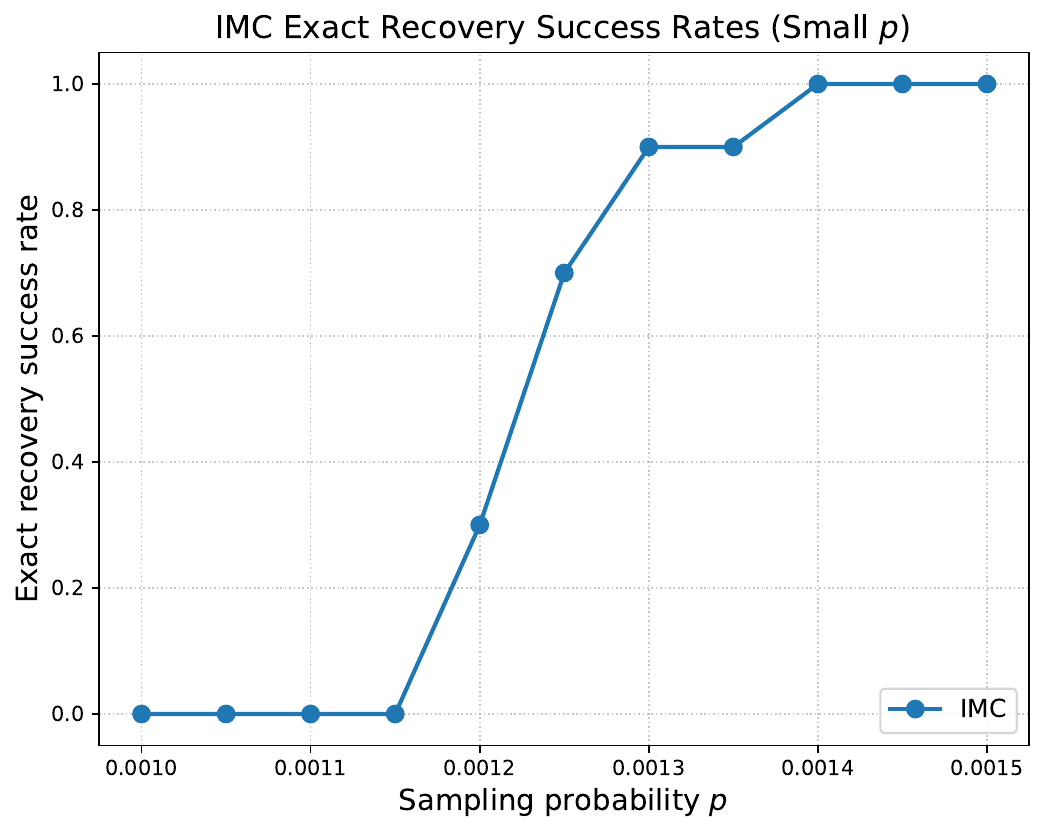}
\captionof{figure}{\label{fig:phase_plot_imc}IMC exact recovery success rates for small
sampling probabilities $p$ in the noiseless exact side-information
setting, averaged over 100 trials.}
\end{minipage}

\vspace{0.5em}

\begin{minipage}[t]{0.48\textwidth}
\centering
\includegraphics[width=0.96\textwidth]{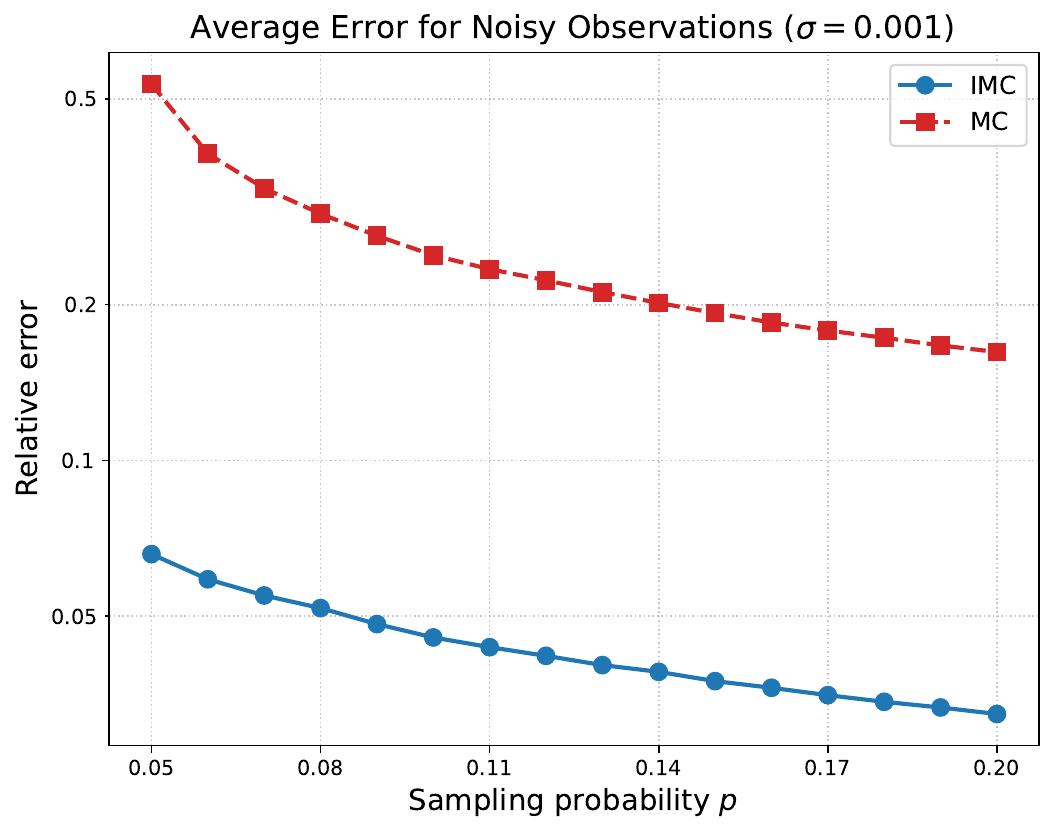}
\captionof{figure}{\label{fig:error_noisyobs}Relative error $\|\widehat{\bm{L}}-\bm{L}^{\star}\|_{\mathrm{F}}/\|\bm{L}^{\star}\|_{\mathrm{F}}$ in log scale for MC
and IMC with noisy observations and exact side-information, with
$\sigma=0.001$. Results are averaged over 100 trials for $p\in[0.05,0.20]$.}
\end{minipage}\hfill
\begin{minipage}[t]{0.48\textwidth}
\centering
\includegraphics[width=0.96\textwidth]{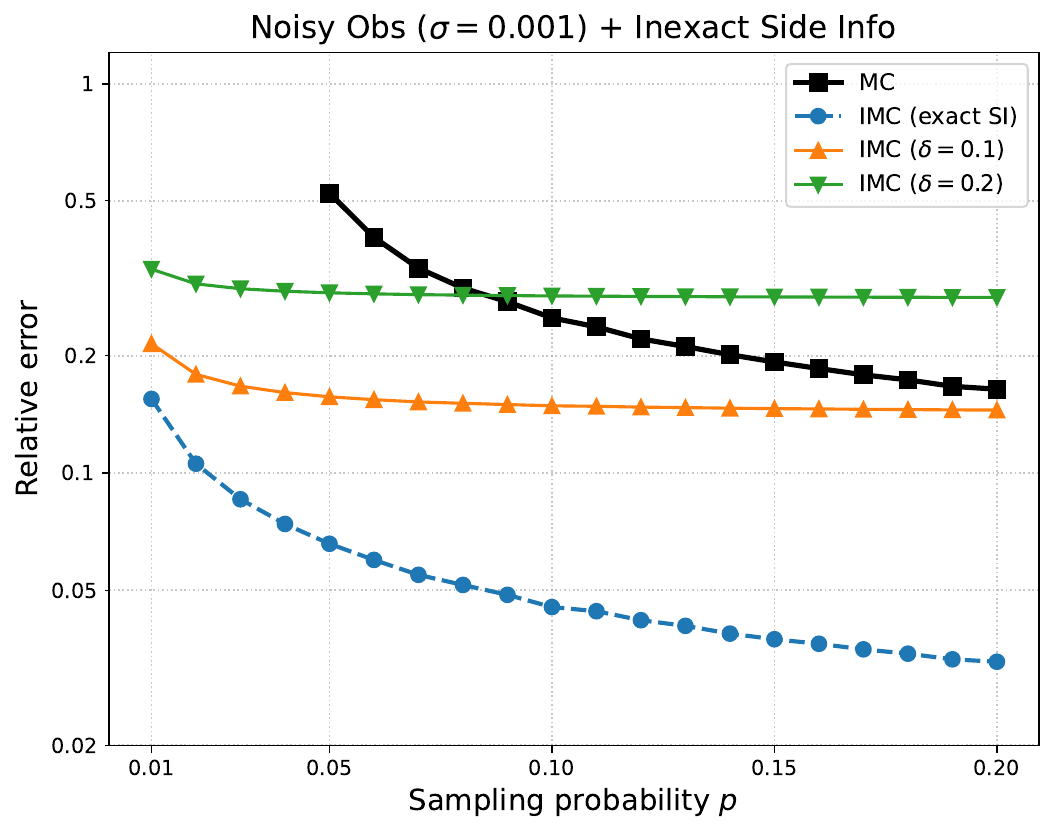}
\raggedright{}\captionof{figure}{\label{fig:errors_noisyobs_noisysi}Relative
error $\|\widehat{\bm{L}}-\bm{L}^{\star}\|_{\mathrm{F}}/\|\bm{L}^{\star}\|_{\mathrm{F}}$ for MC and IMC with noisy observations ($\sigma=0.001$) and
inexact side-information, over 100 trials and $p\in[0.01,0.20]$.}
\end{minipage}

\vspace{0.5em}

\begin{minipage}[t]{0.48\textwidth}
\centering
\includegraphics[width=0.96\textwidth]{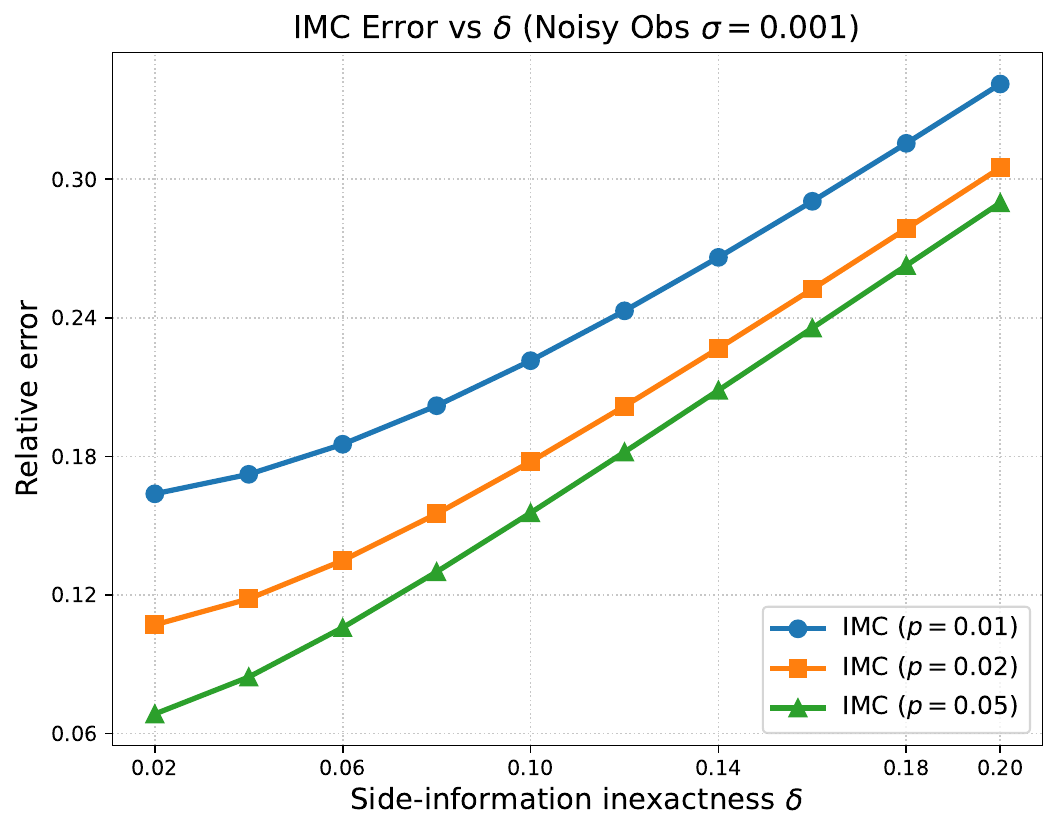}
\captionof{figure}{\label{fig:errors_delta}Relative error $\|\widehat{\bm{L}}-\bm{L}^{\star}\|_{\mathrm{F}}/\|\bm{L}^{\star}\|_{\mathrm{F}}$ for IMC with noisy
observations ($\sigma=0.001$) and inexact side-information as $\delta$
varies from 0.02 to 0.20. Results are averaged over 100 trials.}
\end{minipage}\hfill
\begin{minipage}[t]{0.48\textwidth}
\centering
\includegraphics[width=0.96\textwidth]{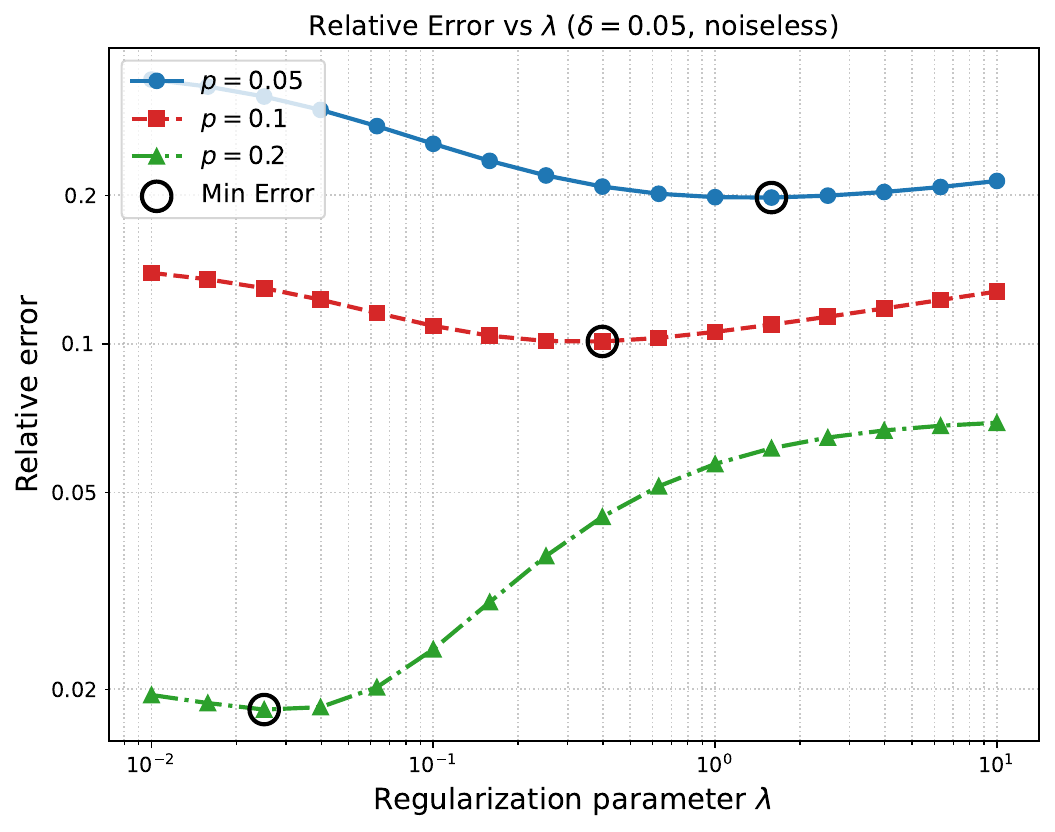}
\captionof{figure}{\label{fig:interpolation}Relative error $\|\widehat{\bm{L}}-\bm{L}^{\star}\|_{\mathrm{F}}/\|\bm{L}^{\star}\|_{\mathrm{F}}$ for the interpolation
method in the noiseless setting as $\lambda$ varies, for $p=0.05$,
$0.1$, and $0.2$. Side-information is inexact with $\delta=0.05$, and results
are averaged over 20 trials.}
\end{minipage}
\end{figure}

For the noisy setting, Figure~\ref{fig:error_noisyobs} compares the average relative error for MC and IMC. For all values of $p$, IMC achieves significantly lower relative error than MC, consistent with our theoretical predictions.

\subsubsection{Inexact side-information}

In the setting with inexact side-information, instead of knowing
$\bm{X}^{\star},\bm{Y}^{\star}$ exactly, we generate inexact side-information
matrices $\bm{X},\bm{Y}$ with side-information inexactness $\delta$
as defined in Assumption~\ref{assumption:inexact_SI}. 

Figure~\ref{fig:errors_noisyobs_noisysi} shows the relative error for IMC at side-information inexactness levels $\delta=0,0.1,0.2$, compared with MC across different sampling rates. 
When $\delta = 0.2$, at small sample sizes, IMC performs significantly better than MC. However, as the sample size increases, the error of MC decreases significantly while the error of IMC plateaus. This matches the trade-off described in Theorem~\ref{thm:inexact} between the low sample complexity requirement and the induced side-information bias of IMC.

Figure~\ref{fig:errors_delta} isolates this bias by varying the side-information
inexactness $\delta$ for $p=0.01,0.02,0.05$. Across
these sampling probabilities, the IMC error increases approximately
linearly with $\delta$, which agrees with Theorem~\ref{thm:inexact}.

\subsubsection{Interpolation Between MC and IMC}

In the inexact side-information setting, we have seen a trade-off between IMC's low sample requirement and MC's better estimation error at larger sample sizes.
The interpolation objective~\eqref{eq:loss_interpolation}
introduced in Section~\ref{sec:Inexact-SI} offers a solution that tries to capture the best of both worlds.

We test this interpolation objective in the noiseless setting with
$\delta=0.05$. For various values of $p$ and $\lambda$, we run
gradient descent with spectral initialization. Figure~\ref{fig:interpolation}
shows that the interpolation can improve
substantially over both MC and IMC for an intermediate choice of
$\lambda$. Such a choice of $\lambda$ is larger when the sample size is small, meaning that the interpolation is more aligned with IMC. This is consistent with our theory that IMC performs better when the sample size is too small for MC even when the side-information is inexact.

In practice, the parameter $\lambda$ can be selected by cross-validation.
We implement a 5-fold cross-validation procedure to select $\lambda$. Table~\ref{tab:comparison}
compares the resulting reconstruction error against MC and IMC for this representative run. The cross-validation procedure is able to take advantage of both MC and IMC and achieves better performance in both low- and high-sample-size regimes in this experiment.

\subsection{Application on the MovieLens Dataset}

We perform experiments on the MovieLens 100K \citep{harper2015movielens}
dataset and compare IMC against non-inductive matrix completion. The MovieLens
dataset consists of 100,000 ratings by 943 users on 1682 movies.
The ratings take integer values from 1 to 5. The dataset also includes
user demographic data and movie genre data as side-information
features that can be used in an inductive model. We also augment the side-information matrix with singular vectors computed from the training matrix in each split. More specifically,
we have 29 features for users (1 for intercept, 1 for gender,
7 for age groups, and 20 for occupations) and 30 features for movies
(1 for intercept, 19 for genres, and 10 for top 10 right singular
vectors of the training matrix). 

We consider the
performance of IMC using our proposed method and non-inductive MC using vanilla gradient descent at different sampling rates.
At each trial, we randomly select $m$ of the 100,000 samples as the
training set and the rest as the testing set. 
Figure~\ref{fig:IMC_MC} shows the RMSE of inductive and non-inductive
matrix completion on MovieLens 100K. The results show that when
sample size $m$ is small, IMC performs significantly better than
non-inductive MC, and when $m$ increases, non-inductive MC starts to
gain an advantage. This behavior is qualitatively consistent with the trade-off highlighted by
Theorem~\ref{thm:inexact}: IMC can perform well when the sample
size is too small for non-inductive MC, while imperfect or incomplete side-information
can limit its large-sample accuracy. For the MovieLens 100K dataset,
$m=5000$ can be viewed as the aforementioned low-sample-complexity setting.
In fact, over the 20 trials used for Figure~\ref{fig:IMC_MC}, an average of 92.6 users and 618.25
movies have no observations at all, which means non-inductive MC cannot provide any nontrivial estimation for these users and movies. This is the scenario
where IMC is powerful in utilizing side-information to provide good estimates even for the rows and columns that are completely unobserved.

\begin{table}[!tph]
\centering
\centering
\small
\setlength{\tabcolsep}{8pt}
\renewcommand{\arraystretch}{1.1}
\begin{tabular}{@{}lcc@{}}
\toprule
 & $p=0.02$ & $p=0.05$ \\
\midrule
Interpolation error & 0.0644 & 0.0340 \\
IMC error & 0.0723 & 0.0712 \\
MC error & 1.1883 &  0.0669 \\
\bottomrule
\end{tabular}

\caption{\label{tab:comparison}Relative errors of cross-validated interpolation,
IMC, and MC at $p=0.02$ and $0.05$ in one representative run. Inexact
side-information with $\delta=0.05$.}
\end{table}

\begin{figure}[!tph]
\centering
\includegraphics[width=0.45\textwidth]{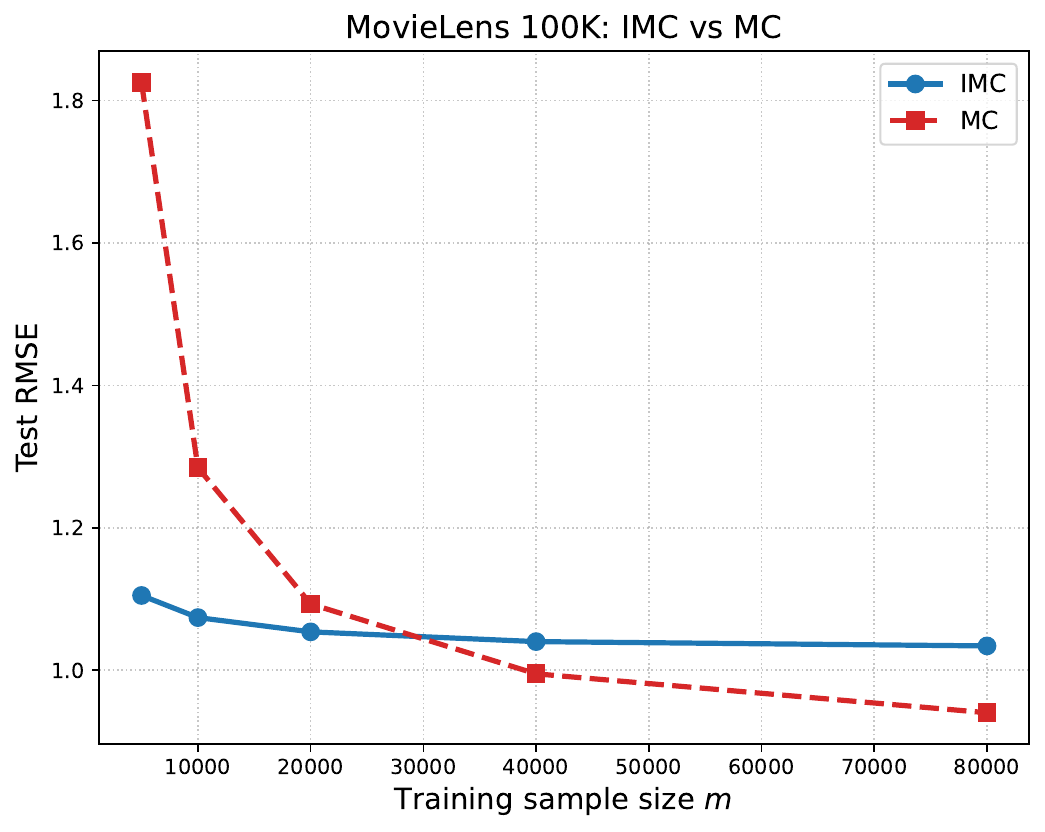}
\caption{\label{fig:IMC_MC}Test RMSE of IMC and non-inductive MC on MovieLens
100K as the training sample size $m$ varies. Results use rank $r=5$
and are averaged over 20 trials.}
\end{figure}

%% file: sections/07-discussion.tex
\section{Discussion \label{sec:Discussion}}

This paper studies noisy inductive matrix completion with exact and inexact
side-information. We showed that IMC has a reduced sample complexity requirement and characterized the estimation error in both exact and inexact cases. We close by highlighting two open questions for future work.
\begin{itemize}
\item \textbf{The (un)necessity of the projection step.} In our theoretical
algorithm, a projection step is used to ensure incoherence of the iterates.
Empirically, however, this projection is inactive in our simulations; see
Appendix~\ref{subsec:Redundancy-of-projection}. This raises the question
of whether comparable guarantees can be proved for Algorithm~\ref{alg:main}
without projection. A main technical difficulty is that the subspace
restriction onto $\mathrm{col}(\bm{X})$ and $\mathrm{col}(\bm{Y})$ couples
the row and column dependencies induced by the observation pattern, making
it difficult to apply the leave-one-out arguments used for non-inductive
MC \citep{ma2018implicit,Chen2020}.
\item \textbf{Partial side-information.} In our model for exact side-information,
we assume the side-information covers the entire column and row spaces
of $\bm{L}^{\star}$. It would be interesting to understand settings
where the available side-information is exact but incomplete, for example
when only a subset of informative features is observed or when only
part of the true subspace is known. Although our inexact side-information
theory can technically cover such cases, it may not provide the sharpest
or most interpretable characterization of this type of structural information.
\end{itemize}

%% file: sections/09-proof-of-subgaussian-corollary.tex
\section{Proofs of Theorem~\ref{thm:main} and Corollary~\ref{cor:main_subgaussian}}
\label{sec:Proof-of-main-and-subgaussian}

\subsection{Proof of Theorem~\ref{thm:main}}
\label{subsec:proof-of-main-theorem}

The proof combines Lemmas~\ref{lem:spec_init} and~\ref{lem:rc}
via an inductive argument. 

First we note that Lemma~\ref{lem:spec_init} shows that $\|\bm A^0\|\le 2\|\bm A^\star\|$. This implies that any $\bm{F}\in \mathcal{C}$ satisfies \eqref{eq:rsc_assumption_incoherence_A} and \eqref{eq:rsc_assumption_incoherence_B}.
Lemma~\ref{lem:spec_init} also shows that 
$\bm{F}^{1}$ satisfies \eqref{eq:rsc_assumption_distance}, which completes the base case. 

For the inductive
step, suppose \eqref{eq:rsc_assumption}
holds for $\bm{F}^{t}$. 
By the definition of $\bm{F}^{t+1}$,
\begin{equation}\label{eq:rotation_optimality}
    \|\bm{\Delta}^{t+1}\|_{\mathrm{F}}^{2}=\|\bm{F}^{t+1}-\bar{\bm{F}}^{t+1}\|_{\mathrm{F}}^{2} \le\|\bm{F}^{t+1}-\bar{\bm{F}}^{t}\|_{\mathrm{F}}^{2}
\end{equation}
where we used the optimality of $\bar{\bm{F}}^{t+1}$.

As right multiplication by a
orthogonal matrix does not change row norms,
\begin{equation}\label{eq:proj_check_A}
    \|\bm X\bar{\bm A}^{t}\|_{2,\infty}
    =
    \|\bm X\bm A^{\star}(\bm H^{t})^{\top}\|_{2,\infty}
    =
    \|\bm X\bm A^{\star}\|_{2,\infty}
    \le
    \sqrt{\frac{\mu_{0}r}{n_{\min}}}\|\bm A^{\star}\|
    \le
    2\sqrt{\frac{\mu_{0}r}{n_{\min}}}\|\bm A^{0}\|.
\end{equation}
Similarly,
\begin{equation}\label{eq:proj_check_B}
    \|\bm Y\bar{\bm B}^{t}\|_{2,\infty}
    \le
    2\sqrt{\frac{\mu_{0}r}{n_{\min}}}\|\bm A^{0}\|.
\end{equation}
Thus $\bar{\bm F}^{t}\in\mathcal C$.
Therefore the projection only makes the factors closer to $\bar{\bm{F}}^t$, i.e.,
\begin{equation}\label{eq:projection_nonexpansiveness}
    \|\bm{F}^{t+1} - \bar{\bm F}^t\|_\mathrm{F}^2 = \|\mathcal{P}_\mathcal{C}(\bm F^t - \eta \nabla f(\bm{F}^t)) - \bar{\bm F}^t)\|_\mathrm{F}^2  
    \le \|\bm F^t - \eta \nabla f(\bm{F}^t) - \bar{\bm F}^t)\|_\mathrm{F}^2 
\end{equation}
The rightmost term reorganizes to $\|\bm \Delta^t - \eta \nabla f(\bm F^t)\|_\mathrm{F}^2$. 
Then combining \eqref{eq:rotation_optimality} and \eqref{eq:projection_nonexpansiveness} gives
\begin{align*}
    \|\bm{\Delta}^{t+1}\|_{\mathrm{F}}^{2}=\|\bm{F}^{t+1}-\bar{\bm{F}}^{t+1}\|_{\mathrm{F}}^{2} & \le\|\bm{F}^{t+1}-\bar{\bm{F}}^{t}\|_{\mathrm{F}}^{2}                                                                      \\                                             & \le\|\widetilde{\bm{F}}^{t+1}-\bar{\bm{F}}^{t}\|_{\mathrm{F}}^{2}                          \\ & =\|\bm{\Delta}^{t}\|_{\mathrm{F}}^{2}+\eta^{2}\|\nabla f(\bm{F}^{t})\|_{\mathrm{F}}^{2}-2\eta\left\langle \nabla f(\bm{F}^{t}),\bm{\Delta}^{t}\right\rangle .
\end{align*}
Applying the regularity condition from Lemma~\ref{lem:rc} with
\[(\alpha,\beta,\gamma)=\left(\frac{40}{\sigma_{\min}},960000\left(\frac{n}{n_{\min}}\right)^{2}\mu_{0}^{2}r^{2}\kappa\sigma_{\max},21r\kappa\gamma_{E}^{2}\right),\]
we have
\begin{align*}
    \|\bm{\Delta}^{t+1}\|_{\mathrm{F}}^{2} & \le\left(1-\frac{2\eta}{\alpha}\right)\left\Vert \bm{\Delta}^{t}\right\Vert _{\mathrm{F}}^{2}+\eta\left(\eta-\frac{2}{\beta}\right)\left\Vert \nabla f(\bm{F}^{t})\right\Vert _{\mathrm{F}}^{2}+2\eta\gamma.
\end{align*}
Since $\eta\le2/\beta$ by assumption, the second term is non-positive
and we obtain the descent inequality
\begin{equation}
    \|\bm{\Delta}^{t+1}\|_{\mathrm{F}}^{2}\le\left(1-\frac{2\eta}{\alpha}\right)\left\Vert \bm{\Delta}^{t}\right\Vert _{\mathrm{F}}^{2}+2\eta\gamma.\label{eq:descent}
\end{equation}

Finally we check the induction hypothesis for $\bm F^{t+1}$. The noise condition $\frac{\gamma_{E}}{\sigma_{\min}}\le C_{\mathrm{noise}}\frac{1}{\sqrt{\kappa r}}$
ensures $\gamma=21r\kappa\gamma_{E}^{2}\le\frac{1}{\alpha}\cdot\frac{1}{100}\sigma_{\min}$,
so \eqref{eq:rsc_assumption_distance} for $\bm F^t$ and \eqref{eq:descent} implies that the condition \eqref{eq:rsc_assumption_distance} holds for $\bm F^{t+1}$:
\[
    \|\bm \Delta^{t+1}\|_\mathrm{F}^2\le \left(1-\frac{2\eta}{\alpha}\right)\left\| \bm \Delta^{t} \right\|_\mathrm{F}^2+2\eta\gamma 
    \le 
    \left(1-\frac{2\eta}{\alpha}\right) \frac{1}{100}{\sigma_{\min}} +\frac{2\eta}{100\alpha}{\sigma_{\min}} \le \frac{1}{100}{\sigma_{\min}}.
\]
Moreover, the projection step makes sure that \eqref{eq:rsc_assumption_incoherence_A} and \eqref{eq:rsc_assumption_incoherence_B} are satisfied. This completes the induction. 

Give the induction, we can apply \eqref{eq:descent}
repeatedly. Then for any $t\ge1$,
\begin{align}
    \|\bm{\Delta}^{t}\|_{\mathrm{F}}^{2}
    &\le
    \left(1-\frac{2\eta}{\alpha}\right)^{t-1}
    \|\bm{\Delta}^{1}\|_{\mathrm{F}}^{2}
    +
    2\eta\gamma
    \sum_{\ell=0}^{t-2}
    \left(1-\frac{2\eta}{\alpha}\right)^{\ell} \nonumber \\
    &\le
    \left(1-\frac{2\eta}{\alpha}\right)^{t-1}
    \|\bm{\Delta}^{1}\|_{\mathrm{F}}^{2}
    +
    \alpha\gamma ,
    \label{eq:Delta_k}
\end{align}
where the last inequality follows from the sum of geometric series.
Furthremore, we claim that 
\begin{equation}
    \left\Vert \bm{Z}^{t}-\bm{Z}^{\star}\right\Vert _{\mathrm{F}}^{2}\le6\sigma_{\max}\left\Vert \bm{\Delta}^{t}\right\Vert _{\mathrm{F}}^{2}.\label{eq:Z_from_Delta}
\end{equation}
The proof is deferred to the end of this section. 

Substituting $\alpha=40/\sigma_{\min}$,
$\gamma=21r\kappa\gamma_E^2$, and $\rho^2 = 1-\eta\sigma_{\min}/20$ into \eqref{eq:Delta_k} and \eqref{eq:Z_from_Delta}, we get
\begin{align}
    \|\bm{Z}^{t}-\bm{Z}^{\star}\|_{\mathrm{F}}^{2}
    &\le
    6\sigma_{\max}
    \rho^{2(t-1)}
    \|\bm{\Delta}^{1}\|_{\mathrm{F}}^{2}
    +
    \frac{5040r\kappa\sigma_{\max}\gamma_E^2}{\sigma_{\min}} .
    \label{eq:Z_error_pre}
\end{align}
By Lemma~\ref{lem:spec_init},
$\|\bm{\Delta}^{1}\|_{\mathrm{F}}^{2}\le\sigma_{\min}/100$.
Moreover, 
\[
    \sigma_{\max}\sigma_{\min}
    =
    \frac{\sigma_{\max}^{2}}{\kappa}
    \le
    \frac{\|\bm{Z}^{\star}\|_{\mathrm{F}}^{2}}{\kappa},
    \qquad
    \|\bm{Z}^{\star}\|_{\mathrm{F}}^{2}
    \ge
    r\sigma_{\min}^{2}.
\]
Therefore,
\[
    \|\bm{Z}^{t}-\bm{Z}^{\star}\|_{\mathrm{F}}^{2}
    \le
    \left[
        \frac{6}{100\kappa}
        \left(1-\frac{2\eta}{\alpha}\right)^{t-1}
        +
        5040\kappa^{2}
        \left(\frac{\gamma_E}{\sigma_{\min}}\right)^2
    \right]
    \|\bm{Z}^{\star}\|_{\mathrm{F}}^{2}.
\]
Finally, since
\[
    1-\frac{2\eta}{\alpha}
    =
    1-\frac{\eta\sigma_{\min}}{20}
    =
    \rho^{2},
\]
we obtain
\[
    \|\bm{Z}^{t}-\bm{Z}^{\star}\|_{\mathrm{F}}
    \le
    C_{\mathrm{F}}\left[
        \rho^{t-1}
        +
        \sqrt{\kappa}\frac{\gamma_E}{\sigma_{\min}}
    \right]
    \|\bm{Z}^{\star}\|_{\mathrm{F}} .
\]
for some large enough constant $C_{\mathrm{F}}>0$.
Using the fact that 
$\|\bm{X}\bm{Z}\bm{Y}^{\top}\|_{\mathrm{F}}
=
\|\bm{Z}\|_{\mathrm{F}}$
for all $\bm{Z}\in\mathbb{R}^{a_1\times a_2}$ then gives the desired
bound for $\|\bm{L}^{t}-\bm{L}^{\star}\|_{\mathrm{F}}$.

\subsection{Proof of Corollary~\ref{cor:main_subgaussian}}
\label{subsec:proof-of-main-subgaussian-corollary}

We first introduce a lemma that bounds the noise size when the 
entries of $\bm{E}$ are sub-exponential.
\begin{lemma}{[}Lemma 4.5, \citet{chen2022nonconvex}{]} \label{lem:POmega_E}
    If $\left[\bm{E}\right]_{ij}$ are i.i.d.\ zero-mean sub-exponential
    random variables that have variance $\sigma^{2}$ and satisfy Bernstein
    condition with parameter $b$, then with probability at least $1-n^{-10}$,
    \[
        \gamma_{E}=\left\Vert \frac{1}{p}\bm{X}^{\top}\mathcal{P}_{\Omega}(\bm{E})\bm{Y}\right\Vert \le C_{E}\left(\sigma\sqrt{\frac{a_{1}\vee a_{2}\log n}{p}}+\frac{1}{p}b\sqrt{\frac{\mu_{1}a_{1}\mu_{2}a_{2}}{n_{1}n_{2}}}\log n\right)
    \]
    for some constant $C_{E}$.

\end{lemma}

Under the standard sub-Gaussian variance-proxy assumption in
Corollary~\ref{cor:main_subgaussian}, the entries $[\bm{E}]_{ij}$
are sub-exponential. In particular, the Bernstein moment condition in
Lemma~\ref{lem:POmega_E} holds with variance parameter at most
$\sigma^{2}$ and Bernstein parameter $b\le C\sigma$ for an absolute
constant $C>0$. Therefore, with probability at least $1-n^{-10}$,
\[
\gamma_{E}
\le
C_{E}\sigma
\left(
\sqrt{\frac{(a_{1}\vee a_{2})\log n}{p}}
+\frac{1}{p}
\sqrt{\frac{\mu_{1}a_{1}\mu_{2}a_{2}}{n_{1}n_{2}}}\log n
\right)
\]
for an absolute constant $C_{E}>0$.

It remains to simplify the second term. The sample-size assumption in
Theorem~\ref{thm:main} implies that
\[
p\ge
\frac{(\mu_{1}a_{1}\wedge\mu_{2}a_{2})\log n}{n_{1}n_{2}}.
\]
Hence
\[
\gamma_{E}
\le
C'_{E}\sigma
\sqrt{\frac{(\mu_{1}a_{1}\vee\mu_{2}a_{2})\log n}{p}}
\]
for some absolute constant $C'_{E}>0$ with probability at least $1-n^{-10}$.

The assumed upper bound assumption on
noise level 
then implies the noise condition
$\gamma_{E}/\sigma_{\min}\le C_{\mathrm{noise}}/\sqrt{\kappa r}$
in Theorem~\ref{thm:main}, provided that $c_{\mathrm{noise}}$ is sufficiently
small. Substituting the bound on $\gamma_{E}$ into
Theorem~\ref{thm:main} gives the estimation error bound in Corollary~\ref{cor:main_subgaussian}. 

%% file: sections/10-proof-of-inexact-side-information.tex
\section{Proof of Theorem~\ref{thm:inexact} \label{sec:Proof-of-inexact}}

Consider the observed matrix $\bm{M}=\mathcal{P}_{\Omega}(\bm{L}^{\star}+\bm{E})$.
Recall that
\[
    \bm L_{\mathrm{proj}}^\star\coloneqq \bm P_X\bm L^\star\bm P_Y,
    \qquad
    \bm E_{\mathrm{mis}}\coloneqq \bm L^\star-\bm L_{\mathrm{proj}}^\star .
\]
Then
\[
    \bm M=\mathcal{P}_{\Omega}\left(\bm L_{\mathrm{proj}}^\star+\bm E_{\mathrm{mis}}+\bm E\right).
\]

One can apply Theorem~\ref{thm:main} treating $\bm L_{\mathrm{proj}}^\star$
as the ground truth matrix therein to obtain an upper bound on
$\widehat{\bm{L}}-\bm L_{\mathrm{proj}}^\star$.
To apply Theorem~\ref{thm:main}, it suffices to give an upper bound
on the noise induced by inexact side-information,
\[
    \gamma_{\mathrm{mis}}
    \coloneqq
    \left\|\bm{X}^{\top}\frac{1}{p}\mathcal{P}_{\Omega}
    (\bm E_{\mathrm{mis}})\bm{Y}\right\|.
\]

We first introduce an auxiliary lemma that controls the effective noise
induced by the misspecification residual $\bm E_{\mathrm{mis}}$. Its proof
is deferred to the end of this section.

\begin{lemma}
\label{lem:gamma-mis}
Suppose Assumptions~\ref{assumption:main},
\ref{assumption:inexact_SI}, and~\ref{assumption:inexact_SI_2inf}
hold. Suppose also that
\[
n_{1}n_{2}p\ge C_{\mathrm{sample}}\kappa^{2}\mu_{0}^{2}r^{2}\left(\mu_{1}a_{1}\vee\mu_{2}a_{2}\right)\log n
\]
for a sufficiently large constant $C_{\mathrm{sample}}>0$. With
probability at least $1-n^{-10}$,
\begin{equation}
\gamma_{\mathrm{mis}}
\le C\delta\|\bm{L}^{\star}\|\sqrt{\frac{\mu_{0}r\left(\mu_{1}a_{1}\vee\mu_{2}a_{2}\right)}{n_{1}n_{2}p}\log n}.
\label{eq:gamma-mis-simple}
\end{equation}
where $C>0$ is an absolute constant.
\end{lemma}

\subsection{Proof of Theorem~\ref{thm:inexact}}

The projected target $\bm L_{\mathrm{proj}}^\star$ satisfies the exact
side-information model by construction, and it is $\mu_0$-incoherent by
Assumption~\ref{assumption:inexact_SI_2inf}. Thus Theorem~\ref{thm:main}
can be applied to $\bm L_{\mathrm{proj}}^\star$ once we control the
effective noise level induced by $\bm E_{\mathrm{mis}}$.

By Lemma~\ref{lem:gamma-mis} and the sample-size assumption in
Theorem~\ref{thm:inexact},
\[
\gamma_{\mathrm{mis}}
\le C\delta\|\bm{L}^{\star}\|\sqrt{\frac{\mu_{0}r\left(\mu_{1}a_{1}\vee\mu_{2}a_{2}\right)}{n_{1}n_{2}p}\log n}.
\]
The sample-size assumption implies
\[
\kappa\sqrt{\frac{\mu_{0}r\left(\mu_{1}a_{1}\vee\mu_{2}a_{2}\right)}{n_{1}n_{2}p}\log n}\lesssim 1.
\]

We now apply Theorem~\ref{thm:main} directly to the projected target
$\bm L_{\mathrm{proj}}^\star=\bm P_X\bm L^\star\bm P_Y$. Since
\[
    \bm M
    =\mathcal{P}_{\Omega}
    \left(\bm L_{\mathrm{proj}}^\star+\bm E_{\mathrm{mis}}+\bm E\right),
\]
the effective projected noise level is bounded by
$\gamma_E+\gamma_{\mathrm{mis}}$. Therefore, under the noise condition
in Theorem~\ref{thm:inexact}, Theorem~\ref{thm:main} gives
\begin{align*}
\|\bm{L}^{t}-\bm L_{\mathrm{proj}}^\star\|_{\mathrm{F}}
& \le C_{\mathrm{F}}\left[\rho^{t}
    +\sqrt{\kappa}\frac{\gamma_{E}+\gamma_{\mathrm{mis}}}{\sigma_{\min}}\right]
    \|\bm L_{\mathrm{proj}}^\star\|_{\mathrm{F}}\\
& \le C_{\mathrm{F}}\left[\rho^{t}
    +\sqrt{\kappa}\left(\frac{\gamma_{E}}{\sigma_{\min}}
    +\delta\right)\right]\|\bm L^\star\|_{\mathrm{F}},
\end{align*}
after renaming absolute constants. Here we used
$\|\bm L_{\mathrm{proj}}^\star\|_{\mathrm{F}}\le\|\bm L^\star\|_{\mathrm{F}}$
and the preceding bound on $\gamma_{\mathrm{mis}}$.

It remains to pass from the projected target back to the original target.
By \eqref{eq:projection-residual-bound},
\[
\|\bm L^\star-\bm L_{\mathrm{proj}}^\star\|_{\mathrm{F}}
\le \sqrt{2}\,\delta\|\bm L^\star\|_{\mathrm{F}}.
\]
The triangle inequality then yields
\begin{align*}
\|\bm{L}^{t}-\bm{L}^{\star}\|_{\mathrm{F}}
& \le
\|\bm{L}^{t}-\bm L_{\mathrm{proj}}^\star\|_{\mathrm{F}}
+\|\bm L_{\mathrm{proj}}^\star-\bm{L}^{\star}\|_{\mathrm{F}}\\
& \le C_{\mathrm{F}}\left[\rho^{t}
    +\sqrt{\kappa}\left(\delta+\frac{\gamma_{E}}{\sigma_{\min}}\right)\right]
    \|\bm{L}^{\star}\|_{\mathrm{F}},
\end{align*}
where $C_{\mathrm{F}}>0$ is a new absolute constant. This proves
\eqref{eq:L-bound-fro-extra-assumption}.

\subsection{Proof of Lemma~\ref{lem:gamma-mis}}

The proof consists of two steps. First, we bound the row-wise and entrywise norms of the misspecification
residual $\bm E_{\mathrm{mis}}$. Then, we plug these deterministic bounds into matrix Bernstein's
inequality to obtain the bound stated in Lemma~\ref{lem:gamma-mis}.

Write
\[
    \bm R_X\coloneqq(\bm I-\bm P_X)\bm U^\star,
    \qquad
    \bm R_Y\coloneqq(\bm I-\bm P_Y)\bm V^\star .
\]
Then
\[
\bm E_{\mathrm{mis}}
=\bm R_X\bm\Sigma^\star\bm V^{\star\top}
 +\bm U^\star\bm\Sigma^\star\bm R_Y^\top
 -\bm R_X\bm\Sigma^\star\bm R_Y^\top.
\]
Assumption~\ref{assumption:inexact_SI_2inf} gives
\[
\|\bm R_X\|_{2,\infty}\le C_{\mathrm{mis}}\sqrt{\frac{\mu_{0}r}{n_{1}}}\,\delta,\qquad
\|\bm R_Y\|_{2,\infty}\le C_{\mathrm{mis}}\sqrt{\frac{\mu_{0}r}{n_{2}}}\,\delta,
\]
and Assumption~\ref{assumption:inexact_SI} gives
$\|\bm R_X\|\le\delta$ and $\|\bm R_Y\|\le\delta$.
Combining these bounds with the incoherence of $\bm U^\star$ and
$\bm V^\star$ yields
\begin{align*}
\|\bm E_{\mathrm{mis}}\|_{2,\infty}
& \le
\|\bm R_X\|_{2,\infty}\|\bm\Sigma^\star\bm V^{\star\top}\|
+\|\bm U^\star\bm\Sigma^\star\|_{2,\infty}\|\bm R_Y\|\\
&\quad
+\|\bm R_X\|_{2,\infty}\|\bm\Sigma^\star\|\|\bm R_Y\|\\
& \le 3\left(C_{\mathrm{mis}}\vee1\right)\sqrt{\frac{\mu_{0}r}{n_{1}}}\,\delta\|\bm{L}^{\star}\|,
\end{align*}
and similarly
\[
\|\bm E_{\mathrm{mis}}\|_{\infty,2}
\le 3\left(C_{\mathrm{mis}}\vee1\right)\sqrt{\frac{\mu_{0}r}{n_{2}}}\,\delta\|\bm{L}^{\star}\|.
\]
The same decomposition gives the entrywise bound
\begin{align*}
\|\bm E_{\mathrm{mis}}\|_{\infty}
& \le
\|\bm R_X\|_{2,\infty}\|\bm\Sigma^\star\bm V^{\star\top}\|_{\infty,2}
+\|\bm U^\star\bm\Sigma^\star\|_{2,\infty}\|\bm R_Y\|_{2,\infty}\\
&\quad
+\|\bm R_X\|_{2,\infty}\|\bm\Sigma^\star\|\|\bm R_Y\|_{2,\infty}\\
& \le 3\left(C_{\mathrm{mis}}\vee1\right)^{2}
\frac{\mu_{0}r}{\sqrt{n_{1}n_{2}}}\cdot\delta\|\bm{L}^{\star}\|,
\end{align*}
where we used $\delta\le1$.

We now construct the ingredients to apply matrix Bernstein's inequality. Let $\delta_{ij}$ be the indicator random variable for observation
at entry $(i,j)$. Then
\begin{align*}
\frac{1}{p}\bm{X}^{\top}\mathcal{P}_{\Omega}(\bm E_{\mathrm{mis}})\bm{Y} & =\frac{1}{p}\sum_{i,j}\delta_{ij}\bm{X}^{\top}[\bm E_{\mathrm{mis}}]_{ij}\bm{e}_{i}\bm{e}_{j}^{\top}\bm{Y}\\
 & =\frac{1}{p}\sum_{i,j}\left(\delta_{ij}-p\right)\bm{X}^{\top}[\bm E_{\mathrm{mis}}]_{ij}\bm{e}_{i}\bm{e}_{j}^{\top}\bm{Y}+\bm{X}^{\top}\bm E_{\mathrm{mis}}\bm{Y}\\
 & =\frac{1}{p}\sum_{i,j}\left(\delta_{ij}-p\right)\bm{X}^{\top}[\bm E_{\mathrm{mis}}]_{ij}\bm{e}_{i}\bm{e}_{j}^{\top}\bm{Y}\\
 & \eqqcolon\frac{1}{p}\sum_{i,j}\bm{Z}_{ij}.
\end{align*}
Furthermore
\begin{align*}
\|\bm{Z}_{ij}\| & =\left\Vert \left(\delta_{ij}-p\right)\bm{X}^{\top}[\bm E_{\mathrm{mis}}]_{ij}\bm{e}_{i}\bm{e}_{j}^{\top}\bm{Y}\right\Vert \\
 & \le|[\bm E_{\mathrm{mis}}]_{ij}|\|\bm{X}^{\top}\bm{e}_{i}\|\|\bm{e}_{j}^{\top}\bm{Y}\|\\
 & \le\|\bm E_{\mathrm{mis}}\|_{\infty}\sqrt{\frac{\mu_{1}a_{1}\mu_{2}a_{2}}{n_{1}n_{2}}}.
\end{align*}
Moreover,
\begin{align*}
\mathbb{E}\sum_{ij}\left(\frac{1}{p}\bm{Z}_{ij}\right)\left(\frac{1}{p}\bm{Z}_{ij}\right)^{\top} & =\frac{1}{p^{2}}\sum_{ij}\mathbb{E}\left(\delta_{ij}-p\right)^{2}[\bm E_{\mathrm{mis}}]_{ij}^{2}\bm{X}^{\top}\bm{e}_{i}\bm{e}_{j}^{\top}\bm{Y}\bm{Y}^{\top}\bm{e}_{j}\bm{e}_{i}^{\top}\bm{X}\\
 & \preccurlyeq\frac{1}{p}\frac{\mu_{2}a_{2}}{n_{2}}\sum_{i,j}[\bm E_{\mathrm{mis}}]_{ij}^{2}\bm{X}^{\top}\bm{e}_{i}\bm{e}_{i}^{\top}\bm{X}\\
 & \preccurlyeq\frac{1}{p}\frac{\mu_{2}a_{2}}{n_{2}}\|\bm E_{\mathrm{mis}}\|_{2,\infty}^{2}\bm{X}^{\top}\bm{I}\bm{X},
\end{align*}
so
\[
\left\Vert \mathbb{E}\sum_{ij}\left(\frac{1}{p}\bm{Z}_{ij}\right)\left(\frac{1}{p}\bm{Z}_{ij}\right)^{\top}\right\Vert \le\frac{1}{p}\frac{\mu_{2}a_{2}}{n_{2}}\|\bm E_{\mathrm{mis}}\|_{2,\infty}^{2}.
\]
Similarly
\[
\left\Vert \mathbb{E}\sum_{ij}\left(\frac{1}{p}\bm{Z}_{ij}\right)^{\top}\left(\frac{1}{p}\bm{Z}_{ij}\right)\right\Vert \le\frac{1}{p}\frac{\mu_{1}a_{1}}{n_{1}}\|\bm E_{\mathrm{mis}}\|_{\infty,2}^{2}.
\]
Matrix Bernstein's inequality gives
\[
\gamma_{\mathrm{mis}}
\le C\left(\sqrt{\frac{1}{p}\left(\frac{\mu_{2}a_{2}}{n_{2}}\|\bm E_{\mathrm{mis}}\|_{2,\infty}^{2}\vee\frac{\mu_{1}a_{1}}{n_{1}}\|\bm E_{\mathrm{mis}}\|_{\infty,2}^{2}\right)\log n}+\|\bm E_{\mathrm{mis}}\|_{\infty}\frac{1}{p}\sqrt{\frac{\mu_{1}a_{1}\mu_{2}a_{2}}{n_{1}n_{2}}}\log n\right).
\]
Substituting the preceding bounds on $\bm E_{\mathrm{mis}}$ gives
\[
\gamma_{\mathrm{mis}}
\le C\left(
\delta\|\bm{L}^{\star}\|\sqrt{\frac{\mu_{0}r\left(\mu_{1}a_{1}\vee\mu_{2}a_{2}\right)}{n_{1}n_{2}p}\log n}
+\delta\|\bm{L}^{\star}\|\frac{\mu_{0}r}{\sqrt{n_{1}n_{2}}}\cdot \frac{1}{p}\sqrt{\frac{\mu_{1}a_{1}\mu_{2}a_{2}}{n_{1}n_{2}}}\log n
\right).
\]
Since
\[
\sqrt{\frac{\mu_{1}a_{1}\mu_{2}a_{2}}{n_{1}n_{2}}}\le \frac{\mu_{1}a_{1}}{n_{1}}\vee\frac{\mu_{2}a_{2}}{n_{2}},
\]
the second term is dominated by the first term under the sample-size condition.
This proves \eqref{eq:gamma-mis-simple}.

\subsection{Proof of Equation~\texorpdfstring{\eqref{eq:projection-residual-bound}}{(9)}}

Let
$\bm L^\star=\bm U^\star\bm\Sigma^\star\bm V^{\star\top}$ be the
rank-$r$ singular value decomposition. Since $\bm P_X$ and $\bm P_Y$
are orthogonal projections,
\[
\bm E_{\mathrm{mis}}
=
(\bm I-\bm P_X)\bm L^\star
+\bm P_X\bm L^\star(\bm I-\bm P_Y),
\]
and the two terms on the right-hand side are orthogonal in Frobenius
inner product. Therefore
\begin{align*}
\|\bm E_{\mathrm{mis}}\|_{\mathrm F}^{2}
&=
\|(\bm I-\bm P_X)\bm U^\star\bm\Sigma^\star\|_{\mathrm F}^{2}
+\|\bm P_X\bm U^\star\bm\Sigma^\star\bm V^{\star\top}
    (\bm I-\bm P_Y)\|_{\mathrm F}^{2}  \\
&\le
\|(\bm I-\bm P_X)\bm U^\star\|^{2}\|\bm\Sigma^\star\|_{\mathrm F}^{2}
+\|\bm\Sigma^\star\bm V^{\star\top}(\bm I-\bm P_Y)\|_{\mathrm F}^{2} \\
&\le
(\delta_X^2+\delta_Y^2)\|\bm L^\star\|_{\mathrm F}^{2}
\le
2\delta^2\|\bm L^\star\|_{\mathrm F}^{2},
\end{align*}
which proves the upper bound.

For the lower bound, the same orthogonal decomposition gives
\[
\|\bm E_{\mathrm{mis}}\|_{\mathrm F}
\ge
\|(\bm I-\bm P_X)\bm U^\star\bm\Sigma^\star\|_{\mathrm F}
\ge
\sigma_{\min}\|(\bm I-\bm P_X)\bm U^\star\|_{\mathrm F}
\ge
\sigma_{\min}\delta_X.
\]
Applying the analogous decomposition
\[
\bm E_{\mathrm{mis}}
=
\bm L^\star(\bm I-\bm P_Y)
+(\bm I-\bm P_X)\bm L^\star\bm P_Y
\]
also yields
$\|\bm E_{\mathrm{mis}}\|_{\mathrm F}\ge \sigma_{\min}\delta_Y$.
Hence
\[
\|\bm E_{\mathrm{mis}}\|_{\mathrm F}\ge\sigma_{\min}\delta.
\]
Finally, $\|\bm L^\star\|_{\mathrm F}\le \sqrt r\,\sigma_{\max}
=\sqrt r\,\kappa\sigma_{\min}$, so
\[
\sigma_{\min}\delta
\ge
\frac{\delta}{\sqrt r\,\kappa}\|\bm L^\star\|_{\mathrm F}.
\]
This completes the proof of \eqref{eq:projection-residual-bound}.

%% file: sections/11-analysis-for-spectral-initialization.tex
\section{Analysis for spectral initialization \label{sec:Proof_of_spec_init} }

Defining $\widetilde{\bm{Z}}$ to as 
\[\widetilde{\bm{Z}}\coloneqq\frac{1}{p}\bm{X}^{\top}\mathcal{P}_{\Omega}(\bm{M})\bm{Y}\]
We start by presenting two lemmas. The first lemma relates $\|\bm{\Delta}^{0}\|_F$
to $\|\widetilde{\bm{Z}}-\bm{Z}^{\star}\|$. 

\begin{lemma}[Eq.~(44) in \citet{zheng2016convergence}]\label{lem:Delta_0}

    \[
        \|\bm{\Delta}^{0}\|_{\mathrm{F}}^{2}\le\frac{20r}{\sigma_{\min}}\|\widetilde{\bm{Z}}-\bm{Z}^{\star}\|^{2}
    \]

\end{lemma}
The second lemma characterizes the spectral estimation error $\|\widetilde{\bm{Z}}-\bm{Z}^{\star}\|$.
The proof is deferred to Appendix~\ref{subsec:Proof-of-Spec-init-Z}.
\begin{lemma}\label{lem:spec_init_Z}Suppose the sample complexity
    satisfies $n_{1}n_{2}p\ge C_{\mathrm{sample}}\left(\mu_{1}a_{1}\vee\mu_{2}a_{2}\right)\mu_{0}r\log n$,
    then
    \[
        \|\widetilde{\bm{Z}}-\bm{Z}^{\star}\|\le C_{\mathrm{spectral}}\|\bm{Z}^{\star}\|\sqrt{\frac{\mu_{0}r\left(\mu_{1}a_{1}\vee\mu_{2}a_{2}\right)}{n_{1}n_{2}p}\log n}+\gamma_{E}
    \]
    for some constants $C_{\mathrm{spectral}}>0$.

\end{lemma}

\paragraph{Proof of Lemma~\ref{lem:spec_init}.}

Let
\[
    \widetilde{\bm Z}_{r}
    =
    \bm U^{0}\bm \Sigma^{0}(\bm V^{0})^{\top}
\]
be the top-$r$ SVD of $\widetilde{\bm Z}$. Recall that the factors of the spectral initialization are
\[
    \bm A^{0}:=\bm U^{0}(\bm \Sigma^{0})^{1/2},
    \qquad
    \bm B^{0}:=\bm V^{0}(\bm \Sigma^{0})^{1/2},
    \qquad
    \bm F^{0}:=
    \begin{bmatrix}
        \bm A^{0}\\
        \bm B^{0}
    \end{bmatrix},
    \qquad
    \bm F^{1}:=\mathcal P_{\mathcal C}(\bm F^{0}).
\]

We first control $\|\bm\Delta^0\|_\mathrm{F}$.
Combining Lemma~\ref{lem:Delta_0} and Lemma~\ref{lem:spec_init_Z}, we get
\begin{equation}\label{eq:spec_init_0}
    \|\bm \Delta^{0}\|_{\mathrm F}
    \le
    \sqrt{\frac{20r}{\sigma_{\min}}}
    \left[
    C_{\mathrm{spectral}}\|\bm Z^{\star}\|
    \sqrt{
        \frac{\mu_{0}r(\mu_{1}a_{1}\vee \mu_{2}a_{2})}{n_{1}n_{2}p}
        \log n
    }
    +\gamma_E
    \right].
\end{equation}
Applying the sample-size condition and the noise condition in the lemma to the right hand side, we have
\[
    \|\bm \Delta^{0}\|_{\mathrm F}
    \le
    \frac{1}{10}\sqrt{\sigma_{\min}} \le \frac{1}{10}\|\bm A^\star\|.
\]
Using triangular inequality, this further implies that 
\[
    \frac{1}{2}\|\bm A^\star\| \le \|\bm A^0\|\le 2\|\bm A^\star\|.
\]

It remains to show that the projection step does not increase the
initialization error. Recall that
\[
    \bar{\bm F}^{0} = \bm F^{\star}(\bm H^{0})^{\top}
    =
    \begin{bmatrix}
        \bar{\bm A}^{0}\\
        \bar{\bm B}^{0}
    \end{bmatrix}
\]
is the optimally aligned ground-truth factor. Similar to the argument in \eqref{eq:proj_check_A} and \eqref{eq:proj_check_B}, we have
$\bar{\bm F}^{0}\in\mathcal C$. 

Since $\bm F^{1}$ is the Euclidean projection of $\bm F^{0}$ onto
$\mathcal C$, and since $\bar{\bm F}^{0}\in\mathcal C$, we have
\[
    \|\bm F^{1}-\bar{\bm F}^{0}\|_{\mathrm F}
    \le
    \|\bm F^{0}-\bar{\bm F}^{0}\|_{\mathrm F}
    =
    \|\bm \Delta^{0}\|_{\mathrm F}.
\]
Moreover, $\bar{\bm F}^{1}$ is the optimally aligned ground truth for
$\bm F^{1}$, so
\[
    \|\bm \Delta^{1}\|_{\mathrm F}
    =
    \|\bm F^{1}-\bar{\bm F}^{1}\|_{\mathrm F}
    \le
    \|\bm F^{1}-\bar{\bm F}^{0}\|_{\mathrm F}
    \le
    \|\bm{\Delta}^{0}\|_{\mathrm F}.
\]
This combined with \eqref{eq:spec_init_0} proves the lemma.

\subsection{Proof of Lemma~\ref{lem:spec_init_Z} \label{subsec:Proof-of-Spec-init-Z}}

By the triangle inequality
\[
    \|\widetilde{\bm{Z}}-\bm{Z}^{\star}\|\le\left\Vert \bm{X}^{\top}\left[\frac{1}{p}P_{\Omega}(\bm{L}^{\star})-\bm{L}^{\star}\right]\bm{Y}\right\Vert +\left\Vert \frac{1}{p}\bm{X}^{\top}P_{\Omega}(\bm{E})\bm{Y}\right\Vert
\]
For the first part, decompose the matrix by each observed entry to
reach
\[
    \bm{X}^{\top}\left[\frac{1}{p}P_{\Omega}(\bm{L}^{\star})-\bm{L}^{\star}\right]\bm{Y}=\sum_{i,j}\underbrace{\left(\frac{1}{p}\delta_{ij}-1\right)\left[\bm{L}^{\star}\right]_{ij}\bm{X}^{\top}\bm{e}_{i}\bm{e}_{j}^{\top}\bm{Y}}_{\eqqcolon\bm{Z}_{ij}}
\]
where $\delta_{ij}\coloneqq\mathds{1}_{(i,j)\in\Omega}$ and $\bm{e}_{i}$
is the standard unit vector with 1 at $i$-th coordinate and 0 elsewhere.
It is easy to see that $\bm{E}\left[\bm{Z}_{ij}\right]=\bm{0}.$ To
invoke matrix Bernstein's inequality we also want to control $\|\bm{Z}_{ij}\|$,
$\|\sum_{i,j}\bm{E}[\bm{Z}_{ij}\bm{Z}_{ij}^{\top}]\|$, and $\|\sum_{i,j}\bm{E}[\bm{Z}_{ij}^{\top}\bm{Z}_{ij}]\|$.

\paragraph*{Controlling $\|\bm{Z}_{ij}\|.$}

First observe that $|\frac{1}{p}\delta_{ij}-1|\le1/p$, then by the
incoherence assumption we know
\[
    \|\bm{X}^{\top}\bm{e}_{i}\|=\|\left(\bm{X}^{\top}\right)_{\cdot i}\|=\|\bm{X}_{i\cdot}\|\le\|\bm{X}\|_{2,\infty}\le\sqrt{\frac{\mu_{1}a_{1}}{n_{1}}}
\]
Likewise $\|\bm{e}_{j}^{\top}\bm{Y}\|\le\sqrt{\frac{\mu_{2}a_{2}}{n_{2}}}$.
For $\left[\bm{L}^{\star}\right]_{ij}$, using the fundamental fact
that $\|\bm{A}\bm{B}\bm{C}\|_{\infty}\le\|\bm{A}\|_{2,\infty}\|\bm{C}\|_{\infty,2}\|\bm{C}\|$
for any matrix $\bm{A},\bm{B},\bm{C}$, we have
\begin{equation}
    \left[\bm{L}^{\star}\right]_{ij}\le\left\Vert \bm{L}^{\star}\right\Vert _{\infty}\le\|\bm{X}\bm{U}^{\star}\|_{2,\infty}\|\bm{V}^{\star}\bm{Y}\|_{2,\infty}\|\bm{\Sigma}^{\star}\|\le\frac{\mu_{0}r}{\sqrt{n_{1}n_{2}}}\|\bm{\Sigma}^{\star}\|\label{eq:M_star_inf}
\end{equation}
Combine these to obtain
\begin{align*}
    \|\bm{Z}_{ij}\| & \le\frac{1}{p}\|\bm{X}^{\top}\bm{e}_{i}\|\|\bm{e}_{j}^{\top}\bm{Y}\|\left|\left[\bm{L}^{\star}\right]_{ij}\right|                                                                                                                  \\
                    & =\frac{1}{p}\sqrt{\frac{\mu_{1}a_{1}}{n_{1}}}\sqrt{\frac{\mu_{2}a_{2}}{n_{2}}}\frac{\mu_{0}r}{\sqrt{n_{1}n_{2}}}\|\bm{\Sigma}^{\star}\|=\frac{\sqrt{\mu_{0}^{2}r^{2}\mu_{1}a_{1}\mu_{2}a_{2}}}{n_{1}n_{2}p}\|\bm{\Sigma}^{\star}\|
\end{align*}

\paragraph{Controlling $\|\sum_{i,j}\bm{E}[\bm{Z}_{ij}\bm{Z}_{ij}^{\top}]\|,\|\sum_{i,j}\bm{E}[\bm{Z}_{ij}^{\top}\bm{Z}_{ij}]\|.$ }

\begin{align*}
    \sum_{i,j}\bm{E}[\bm{Z}_{ij}\bm{Z}_{ij}^{\top}] & =\sum_{i,j}\bm{E}\left[\left(\frac{1}{p}\delta_{ij}-1\right)^{2}\right]\left[\bm{L}^{\star}\right]_{ij}^{2}\bm{X}^{\top}\bm{e}_{i}\bm{e}_{j}^{\top}\bm{Y}\bm{Y}^{\top}\bm{e}_{j}\bm{e}_{i}^{\top}\bm{X} \\
                                                    & =\frac{1-p}{p}\cdot\sum_{i,j}\left[\bm{L}^{\star}\right]_{ij}^{2}\bm{X}^{\top}\bm{e}_{i}\bm{e}_{j}^{\top}\bm{Y}\bm{Y}^{\top}\bm{e}_{j}\bm{e}_{i}^{\top}\bm{X}                                           \\
                                                    & \overset{\text{(i)}}{\preccurlyeq}\frac{1}{p}\cdot\frac{\mu_{2}a_{2}}{n_{2}}\sum_{i,j}\left[\bm{L}^{\star}\right]_{ij}^{2}\bm{X}^{\top}\bm{e}_{i}\bm{e}_{i}^{\top}\bm{X}                                \\
                                                    & \preccurlyeq\frac{1}{p}\frac{\mu_{2}a_{2}}{n_{2}}\bm{X}^{\top}\sum_{i}\sum_{j}\left[\bm{L}^{\star}\right]_{ij}^{2}\bm{e}_{i}\bm{e}_{i}^{\top}\bm{X}.
\end{align*}
where (i) uses $\|\bm{e}_{j}^{\top}\bm{Y}\|\le\sqrt{\frac{\mu_{2}a_{2}}{n_{2}}}$.
As $\sum_{j}\left[\bm{L}^{\star}\right]_{ij}^{2}=\|\bm{L}_{i\cdot}^{\star}\|^{2}\le\|\bm{U}_{i\cdot}^{\star}\|^{2}\|\bm{\Sigma}^{\star}\|^{2}$,
the last line can be rewritten as
\[
    \sum_{i,j}\bm{E}[\bm{Z}_{ij}\bm{Z}_{ij}^{\top}]\preccurlyeq\frac{1}{p}\frac{\mu_{2}a_{2}}{n_{2}}\frac{\mu_{0}r}{n_{1}}\|\bm{\Sigma}^{\star}\|^{2}\bm{X}^{\top}\bm{I}_{n}\bm{X}
\]
Then
\[
    \|\bm{E}[\bm{Z}_{ij}\bm{Z}_{ij}^{\top}]\|\le\frac{1}{p}\frac{\mu_{2}b_{2}}{n_{2}}\frac{\mu_{0}r}{n_{1}}\|\bm{\Sigma}^{\star}\|^{2}
\]
Similarly we may have
\[
    \|\bm{E}[\bm{Z}_{ij}^{\top}\bm{Z}_{ij}]\|\le\frac{1}{p}\frac{\mu_{1}a_{1}}{n_{1}}\frac{\mu_{0}r}{n_{2}}\|\bm{\Sigma}^{\star}\|^{2}
\]
Now we are ready to invoke matrix Bernstein's inequality, which gives
us that with probability at least $1-n^{-10}$,
\begin{align*}
    \left\Vert \bm{X}^{\top}\left[\frac{1}{p}P_{\Omega}(\bm{L}^{\star})-\bm{L}^{\star}\right]\bm{Y}\right\Vert & \le C_{\mathrm{spectral}}\left(\sqrt{\frac{\mu_{0}r\left(\mu_{1}a_{1}\vee\mu_{2}a_{2}\right)}{n_{1}n_{2}p}\|\bm{\Sigma}^{\star}\|^{2}\log n}+\frac{\sqrt{\mu_{0}^{2}r^{2}\mu_{1}a_{1}\mu_{2}a_{2}}}{n_{1}n_{2}p}\|\bm{\Sigma}^{\star}\|\log n\right),
\end{align*}
where $C_{\mathrm{noise}}>0$ is some absolute constant. As the sample
complexity satisfies
\[n_{1}n_{2}p\ge C_{\mathrm{sample}}\mu_{0}r\left(\mu_{1}a_{1}\wedge\mu_{2}a_{2}\right)\log n,\]
the first term dominates, and
\[
    \|\bm{X}^{\top}\left[\frac{1}{p}P_{\Omega}(\bm{L}^{\star})-\bm{L}^{\star}\right]\bm{Y}\|\le2C_{\mathrm{spectral}}\|\bm{Z}^{\star}\|\sqrt{\frac{\mu_{0}r\left(\mu_{1}a_{1}\vee\mu_{2}a_{2}\right)}{n_{1}n_{2}p}\log n}
\]
Combine this with the definition of $\gamma_{E}\coloneqq\left\Vert \frac{1}{p}\bm{X}^{\top}P_{\Omega}(\bm{E})\bm{Y}\right\Vert $
and rename the constant, we have
\[
    \|\widetilde{\bm{Z}}-\bm{Z}^{\star}\|\le C_{\mathrm{spectral}}\|\bm{Z}^{\star}\|\sqrt{\frac{\mu_{0}r\left(\mu_{1}a_{1}\vee\mu_{2}a_{2}\right)}{n_{1}n_{2}p}\log n}+\gamma_{E}
\]
for some absolute constant $C_{\mathrm{spectral}}$.

%% file: sections/12-proof-of-regularity-condition.tex
\section{Proof of regularity condition \label{sec:Proof-of-regularity}}

First we present the following immediate result for cross-referencing
purposes. The proof is omitted for brevity.

\begin{lemma}Suppose the assumptions in Lemma~\ref{lem:rc} hold,
    then

    \begin{subequations}
        \begin{equation}
            \max\left\{\|\bm{A}\|_{\mathrm{F}},\|\bm{B}\|_{\mathrm{F}}\right\}\le2\|\bm{A}^{\star}\|_{\mathrm{F}}=2\sqrt{\sigma_{\max}}\label{eq:A_size}
        \end{equation}

        \begin{equation}
            \|\bm{\Delta}\|_{\mathrm{F}}=\left\Vert \bm{F}-\bar{\bm{F}}\right\Vert _{\mathrm{F}}\le\frac{1}{10}\sqrt{\sigma_{\min}}\label{eq:Delta_size}
        \end{equation}

        \begin{equation}
            \max\left\{\|\bm{X}\bm{A}\|_{2,\infty},\|\bm{Y}\bm{B}\|_{2,\infty}\right\}\le4\sqrt{\frac{\mu_{0}r}{n_{\min}}}\|\bm{A}^{\star}\|\label{eq:XA_2_infty}
        \end{equation}
        \begin{equation}
            \max\left\{\|\bm{X}\bm{\Delta}_{\bm{A}}\|_{2,\infty},\|\bm{Y}\bm{\Delta}_{\bm{B}}\|_{2,\infty}\right\}\le5\sqrt{\frac{\mu_{0}r}{n_{\min}}}\|\bm{A}^{\star}\|\label{eq:XDeltaA_2_infty}
        \end{equation}

    \end{subequations}

\end{lemma}

Furthermore we define the tangent space $T$ as
\[
    T=\left\{ \bm{X}\left(\bm{U}^{\star}\bm{D}^{\top}+\bm{C}\bm{V}^{\star\top}\right)\bm{Y}^{\top}:\bm{C}\in\mathbb{R}^{a_{1}\times r},\bm{C}\in\mathbb{R}^{a_{2}\times r}\right\}
\]
which is used in Lemma~\ref{lem:PT_POmega}.

\subsection{Proof of Lemma~\ref{lem:rc}}

We first assume the following two inequalities. The first one is the
local curvature condition:
\begin{equation}
    \left\langle \nabla f\left(\begin{bmatrix}\bm{A} \\
        \bm{B}
    \end{bmatrix}\right),\begin{bmatrix}\bm{\Delta}_{\bm{A}} \\
        \bm{\Delta}_{\bm{B}}
    \end{bmatrix}\right\rangle \ge\frac{1}{20}\sigma_{\min}\left\Vert \bm{\Delta}\right\Vert _{\mathrm{F}}^{2}+\frac{1}{8}\left\Vert \bar{\bm{A}}^{\top}\bm{\Delta}_{\bm{A}}-\bar{\bm{B}}^{\top}\bm{\Delta}_{\bm{B}}\right\Vert _{\mathrm{F}}^{2}-20r\kappa\gamma_{E}^{2}.\label{eq:local_curv}
\end{equation}
The second one is the local smoothness condition
\begin{equation}
    \left\Vert \nabla f(\bm{F})\right\Vert _{\mathrm{F}}^{2}\le22000\left(\frac{n}{n_{\min}}\right)^{2}\mu_{0}^{2}r^{2}\sigma_{\max}^{2}\left\Vert \bm{\Delta}\right\Vert _{\mathrm{F}}^{2}+24\sigma_{\max}\left\Vert \bar{\bm{A}}^{\top}\bm{\Delta}_{\bm{A}}-\bar{\bm{B}}^{\top}\bm{\Delta}_{\bm{B}}\right\Vert _{\mathrm{F}}^{2}+4r\sigma_{\max}\gamma_{E}^{2}\label{eq:local_smoo}
\end{equation}

Together they give us the regularity condition:

\begin{theorem}

    \[
        \left\langle \nabla f\left(\begin{bmatrix}\bm{A} \\
            \bm{B}
        \end{bmatrix}\right),\begin{bmatrix}\bm{\Delta}_{\bm{A}} \\
            \bm{\Delta}_{\bm{B}}
        \end{bmatrix}\right\rangle \ge\frac{1}{40}\sigma_{\min}\left\Vert \bm{\Delta}\right\Vert _{\mathrm{F}}^{2}+\frac{1}{1280000\mu_{0}^{2}r^{2}\kappa\sigma_{\max}}\left\Vert \nabla f(\bm{F})\right\Vert _{\mathrm{F}}^{2}-21r\kappa\gamma_{E}^{2}
    \]
\end{theorem}

\begin{proof}

    Since $\left\Vert \bar{\bm{A}}^{\top}\bm{\Delta}_{\bm{A}}-\bar{\bm{B}}^{\top}\bm{\Delta}_{\bm{B}}\right\Vert _{\mathrm{F}}^{2}$
    is non-negative,
    \begin{align*}
        \frac{1}{8}\left\Vert \bar{\bm{A}}^{\top}\bm{\Delta}_{\bm{A}}-\bar{\bm{B}}^{\top}\bm{\Delta}_{\bm{B}}\right\Vert _{\mathrm{F}}^{2} & \ge\frac{1}{5000\left(n/n_{\min}\right)^{2}64\mu_{0}^{2}r^{2}}\cdot\frac{1}{8}\left\Vert \bar{\bm{A}}^{\top}\bm{\Delta}_{\bm{A}}-\bar{\bm{B}}^{\top}\bm{\Delta}_{\bm{B}}\right\Vert _{\mathrm{F}}^{2}
    \end{align*}
    By \eqref{eq:local_smoo}
    \begin{align*}
        \frac{1}{8}\left\Vert \bar{\bm{A}}^{\top}\bm{\Delta}_{\bm{A}}-\bar{\bm{B}}^{\top}\bm{\Delta}_{\bm{B}}\right\Vert _{\mathrm{F}}^{2} & \ge\frac{1}{192\cdot5000\left(n/n_{\min}\right)^{2}64\mu_{0}^{2}r^{2}\kappa\sigma_{\max}}\left\Vert \nabla f(\bm{F})\right\Vert _{\mathrm{F}}^{2}-\frac{22000}{192\cdot5000}\sigma_{\min}\left\Vert \bm{\Delta}\right\Vert _{\mathrm{F}}^{2} \\
                                                                                                                                           & \quad-\frac{r}{48\cdot5000\left(n/n_{\min}\right)^{2}\mu_{0}^{2}r^{2}\kappa}\gamma_{E}^{2}                                                                                                                                                   \\
                                                                                                                                           & \ge\frac{1}{960000\left(n/n_{\min}\right)^{2}\mu_{0}^{2}r^{2}\kappa\sigma_{\max}}\left\Vert \nabla f(\bm{F})\right\Vert _{\mathrm{F}}^{2}-\frac{1}{40}\sigma_{\min}\left\Vert \bm{\Delta}\right\Vert _{\mathrm{F}}^{2}-r\gamma_{\bm{E}}^{2}.
    \end{align*}
    Combine this with \eqref{eq:local_curv} to reach
    \begin{align*}
        \left\langle \nabla f\left(\begin{bmatrix}\bm{A} \\
                                       \bm{B}
                                   \end{bmatrix}\right),\begin{bmatrix}\bm{\Delta}_{\bm{A}} \\
                                                            \bm{\Delta}_{\bm{B}}
                                                        \end{bmatrix}\right\rangle & \ge\frac{1}{20}\sigma_{\min}\left\Vert \bm{\Delta}\right\Vert _{\mathrm{F}}^{2}+\frac{1}{8}\left\Vert \bar{\bm{A}}^{\top}\bm{\Delta}_{\bm{A}}-\bar{\bm{B}}^{\top}\bm{\Delta}_{\bm{B}}\right\Vert _{\mathrm{F}}^{2}-20r\kappa\gamma_{E}^{2} \\
                                                                    & \ge\frac{1}{40}\sigma_{\min}\left\Vert \bm{\Delta}\right\Vert _{\mathrm{F}}^{2}+\frac{1}{960000\left(n/n_{\min}\right)^{2}\mu_{0}^{2}r^{2}\kappa\sigma_{\max}}\left\Vert \nabla f(\bm{F})\right\Vert _{\mathrm{F}}^{2}-21r\kappa\gamma_{E}^{2}
    \end{align*}

\end{proof}

\subsection{Local curvature condition~\texorpdfstring{\eqref{eq:local_curv}}{equation}}

We first start with the noiseless case. Let $f_{\mathrm{clean}}$
be defined as $f_{\mathrm{clean}}(\bm{A},\bm{B})\coloneqq\frac{1}{2p}\|\mathcal{P}_{\Omega}(\bm{X}\bm{A}\bm{B}^{\top}\bm{Y}^{\top}-\bm{L}^{\star})\|_{\mathrm{F}}^{2}+\frac{1}{16}\|\bm{A}^{\top}\bm{A}-\bm{B}^{\top}\bm{B}\|_{\mathrm{F}}^{2}$,
then

\begin{align}
    \left\langle \nabla f_{\mathrm{clean}}\left(\begin{bmatrix}\bm{A} \\
                                                    \bm{B}
                                                \end{bmatrix}\right),\begin{bmatrix}\bm{\Delta}_{\bm{A}} \\
                                                                         \bm{\Delta}_{\bm{B}}
                                                                     \end{bmatrix}\right\rangle & =\underbrace{\frac{1}{p}\left\langle \bm{X}^{\top}\mathcal{P}_{\Omega}(\bm{X}\bm{A}\bm{B}^{\top}\bm{Y}^{\top}-\bm{L}^{\star})\bm{Y}\bm{B},\bm{\Delta}_{\bm{A}}\right\rangle +\frac{1}{p}\left\langle \bm{Y}^{\top}\mathcal{P}_{\Omega}(\bm{X}\bm{A}\bm{B}^{\top}\bm{Y}^{\top}-\bm{L}^{\star})^{\top}\bm{X}\bm{A},\bm{\Delta}_{\bm{B}}\right\rangle }_{\alpha_{1}}\label{eq:curvature_decomp} \\
                                                                         & \quad+\underbrace{\frac{1}{4}\left\langle \bm{A}\left(\bm{A}^{\top}\bm{A}-\bm{B}^{\top}\bm{B}\right),\bm{\Delta}_{\bm{A}}\right\rangle -\frac{1}{4}\left\langle \bm{B}\left(\bm{A}^{\top}\bm{A}-\bm{B}^{\top}\bm{B}\right),\bm{\Delta}_{\bm{B}}\right\rangle }_{\alpha_{2}}.\nonumber
\end{align}
We control $\alpha_{1}$ and $\alpha_{2}$ separately.

\paragraph{Controlling $\alpha_{1}$.}

We reorder the terms in the inner products and decompose them with
$\bm{A}=\bar{\bm{A}}+\bm{\Delta}_{\bm{A}},\bm{B}=\bar{\bm{B}}+\bm{\Delta}_{\bm{B}}$
to have

\begin{align*}
    \alpha_{1} & =\frac{1}{p}\left\langle \bm{X}^{\top}\mathcal{P}_{\Omega}(\bm{X}\bm{A}\bm{B}^{\top}\bm{Y}^{\top}-\bm{L}^{\star})\bm{Y}\bm{B},\bm{\Delta}_{\bm{A}}\right\rangle +\frac{1}{p}\left\langle \bm{Y}^{\top}\mathcal{P}_{\Omega}(\bm{X}\bm{A}\bm{B}^{\top}\bm{Y}^{\top}-\bm{L}^{\star})^{\top}\bm{X}\bm{A},\bm{\Delta}_{\bm{B}}\right\rangle                                                             \\
               & =\frac{1}{p}\left\langle \mathcal{P}_{\Omega}(\bm{X}\bm{A}\bm{B}^{\top}\bm{Y}^{\top}-\bm{L}^{\star}),\bm{X}\bm{\Delta}_{\bm{A}}\bm{B}^{\top}\bm{Y}^{\top}\right\rangle +\frac{1}{p}\left\langle \mathcal{P}_{\Omega}(\bm{X}\bm{A}\bm{B}^{\top}\bm{Y}^{\top}-\bm{L}^{\star}),\bm{X}\bm{A}\bm{\Delta}_{\bm{B}}^{\top}\bm{Y}^{\top}\right\rangle                                                      \\
               & =\frac{1}{p}\left\langle \mathcal{P}_{\Omega}\left(\bm{X}\bar{\bm{A}}\bm{\Delta}_{\bm{B}}^{\top}\bm{Y}^{\top}+\bm{X}\bm{\Delta}_{\bm{A}}\bar{\bm{B}}^{\top}\bm{Y}^{\top}+\bm{X}\bm{\Delta}_{\bm{A}}\bm{\Delta}_{\bm{B}}^{\top}\bm{Y}^{\top}\right),\bm{X}\bm{\Delta}_{\bm{A}}\bar{\bm{B}}^{\top}\bm{Y}^{\top}+\bm{X}\bm{\Delta}_{\bm{A}}\bm{\Delta}_{\bm{B}}\bm{Y}^{\top}\right\rangle             \\
               & \quad+\frac{1}{p}\left\langle \mathcal{P}_{\Omega}\left(\bm{X}\bar{\bm{A}}\bm{\Delta}_{\bm{B}}^{\top}\bm{Y}^{\top}+\bm{X}\bm{\Delta}_{\bm{A}}\bar{\bm{B}}^{\top}\bm{Y}^{\top}+\bm{X}\bm{\Delta}_{\bm{A}}\bm{\Delta}_{\bm{B}}^{\top}\bm{Y}^{\top}\right),\bm{X}\bar{\bm{A}}\bm{\Delta}_{\bm{B}}^{\top}\bm{Y}^{\top}+\bm{X}\bm{\Delta}_{\bm{A}}\bm{\Delta}_{\bm{B}}^{\top}\bm{Y}^{\top}\right\rangle
\end{align*}
Collecting the terms we have
\begin{align}
    \alpha_{1} & =\frac{1}{p}\left\Vert \mathcal{P}_{\Omega}\left(\bm{X}\bar{\bm{A}}\bm{\Delta}_{\bm{B}}^{\top}\bm{Y}^{\top}+\bm{X}\bm{\Delta}_{\bm{A}}\bar{\bm{B}}^{\top}\bm{Y}^{\top}\right)\right\Vert _{\mathrm{F}}^{2}+\frac{2}{p}\left\Vert \mathcal{P}_{\Omega}\left(\bm{X}\bm{\Delta}_{\bm{A}}\bm{\Delta}_{\bm{B}}^{\top}\bm{Y}^{\top}\right)\right\Vert _{\mathrm{F}}^{2}\label{eq:alpha1_lb}                   \\
               & \quad+\frac{3}{p}\left\langle \mathcal{P}_{\Omega}\left(\bm{X}\bar{\bm{A}}\bm{\Delta}_{\bm{B}}^{\top}\bm{Y}^{\top}+\bm{X}\bm{\Delta}_{\bm{A}}\bar{\bm{B}}^{\top}\bm{Y}^{\top}\right),\bm{X}\bm{\Delta}_{\bm{A}}\bm{\Delta}_{\bm{B}}^{\top}\bm{Y}^{\top}\right\rangle \nonumber                                                                                                                          \\
               & \ge\underbrace{\frac{1}{2p}\left\Vert \mathcal{P}_{\Omega}\left(\bm{X}\bar{\bm{A}}\bm{\Delta}_{\bm{B}}^{\top}\bm{Y}^{\top}+\bm{X}\bm{\Delta}_{\bm{A}}\bar{\bm{B}}^{\top}\bm{Y}^{\top}\right)\right\Vert _{\mathrm{F}}^{2}}_{\beta_{1}}-\frac{5}{2p}\left\Vert \mathcal{P}_{\Omega}\left(\bm{X}\bm{\Delta}_{\bm{A}}\bm{\Delta}_{\bm{B}}^{\top}\bm{Y}^{\top}\right)\right\Vert _{\mathrm{F}}^{2}\nonumber
\end{align}
where the last line comes from Cauchy-Schwarz inequality and the fact
that
\[
    a^{2}+2b^{2}-3ab=\frac{1}{2}a^{2}-\frac{5}{2}b^{2}+\left(\frac{1}{\sqrt{2}}a-\frac{3}{\sqrt{2}}b\right)^{2}\ge\frac{1}{2}a^{2}-\frac{5}{2}b^{2}
\]
for any $a,b\in\mathbb{R}$.

For the second term $\frac{5}{2p}\left\Vert \mathcal{P}_{\Omega}\left(\bm{X}\bm{\Delta}_{\bm{A}}\bm{\Delta}_{\bm{B}}^{\top}\bm{Y}^{\top}\right)\right\Vert _{\mathrm{F}}^{2}$,
apply Lemma~\ref{lem:P_Omega_norm} to have
\begin{align*}
    \frac{5}{2p}\left\Vert \mathcal{P}_{\Omega}\left(\bm{X}\bm{\Delta}_{\bm{A}}\bm{\Delta}_{\bm{B}}^{\top}\bm{Y}^{\top}\right)\right\Vert _{\mathrm{F}}^{2} & \le200\left(\sqrt{\frac{\mu_{1}a_{1}n_{2}\log n}{n_{1}p}}\|\bm{\Delta}_{\bm{A}}\|_{\mathrm{F}}^{2}\|\bm{Y}\bm{\Delta}_{\bm{B}}\|_{2,\infty}^{2}\wedge\sqrt{\frac{\mu_{2}a_{2}n_{1}\log n}{n_{2}p}}\|\bm{\Delta}_{\bm{B}}\|_{\mathrm{F}}^{2}\|\bm{X}\bm{\Delta}_{\bm{A}}\|_{2,\infty}^{2}\right) \\
                                                                                                                                                            & \quad+\frac{5}{2}\left\Vert \bm{\Delta}_{\bm{A}}\right\Vert _{\mathrm{F}}^{2}\left\Vert \bm{\Delta}_{\bm{B}}\right\Vert _{\mathrm{F}}^{2}                                                                                                                                                       \\
                                                                                                                                                            & \le5000\left(\sqrt{\frac{\mu_{1}a_{1}n_{2}\mu_{0}^{2}r^{2}\log n}{n_{1}n_{\min}^{2}p}}\wedge\sqrt{\frac{\mu_{2}a_{2}n_{1}\mu_{0}^{2}r^{2}\log n}{n_{2}n_{\min}^{2}p}}\right)\sigma_{\max}\|\bm{\Delta}\|_{\mathrm{F}}^{2}+\frac{5}{2}\left\Vert \bm{\Delta}\right\Vert _{\mathrm{F}}^{4}.
\end{align*}
For the last line we use \eqref{eq:XDeltaA_2_infty}. Note that since
$n_{\min}\le n_{2}$, $n_{2}/(n_{1}n_{\min}^{2})\le1/(n_{1}n_{2})$.
Similarly $n_{1}/(n_{2}n_{\min}^{2})\le1/(n_{1}n_{2})$. Then the
last line can be rewritten as
\begin{equation}
    \frac{5}{2p}\left\Vert \mathcal{P}_{\Omega}\left(\bm{X}\bm{\Delta}_{\bm{A}}\bm{\Delta}_{\bm{B}}^{\top}\bm{Y}^{\top}\right)\right\Vert _{\mathrm{F}}^{2}\le5000\left(\sqrt{\frac{\left(\mu_{1}a_{1}\vee\mu_{2}a_{2}\right)\mu_{0}^{2}r^{2}\log n}{n_{1}n_{2}p}}\right)\sigma_{\max}\|\bm{\Delta}\|_{\mathrm{F}}^{2}+\frac{5}{2}\left\Vert \bm{\Delta}\right\Vert _{\mathrm{F}}^{4}\label{eq:Omega_X_Delta_Delta_Y}
\end{equation}
For the term $\beta_{1}$,
\begin{align*}
    \beta_{1} & =\frac{1}{2p}\left\Vert \mathcal{P}_{\Omega}\left(\bm{X}\bar{\bm{A}}\bm{\Delta}_{\bm{B}}^{\top}\bm{Y}^{\top}\right)\right\Vert _{\mathrm{F}}^{2}+\frac{1}{2p}\left\Vert \mathcal{P}_{\Omega}\left(\bm{X}\bm{\Delta}_{\bm{A}}\bar{\bm{B}}^{\top}\bm{Y}^{\top}\right)\right\Vert _{\mathrm{F}}^{2} \\
              & \quad+\frac{1}{p}\left\langle \mathcal{P}_{\Omega}\left(\bm{X}\bar{\bm{A}}\bm{\Delta}_{\bm{B}}^{\top}\bm{Y}^{\top}\right),\mathcal{P}_{\Omega}\left(\bm{X}\bm{\Delta}_{\bm{A}}\bar{\bm{B}}^{\top}\bm{Y}^{\top}\right)\right\rangle
\end{align*}
As both $\bm{X}\bar{\bm{A}}\bm{\Delta}_{\bm{B}}^{\top}\bm{Y}^{\top}$,
$\bm{X}\bm{\Delta}_{\bm{A}}\bar{\bm{B}}^{\top}\bm{Y}^{\top}$ are
in $T$, we can apply Lemma~\ref{lem:PT_POmega} to reach
\begin{align*}
    \beta_{1} & \ge\frac{1}{2}\cdot\left(1-\frac{1}{10}\right)\left(\left\Vert \bm{X}\bar{\bm{A}}\bm{\Delta}_{\bm{B}}^{\top}\bm{Y}^{\top}\right\Vert _{\mathrm{F}}^{2}+\left\Vert \bm{X}\bm{\Delta}_{\bm{A}}\bar{\bm{B}}^{\top}\bm{Y}^{\top}\right\Vert _{\mathrm{F}}^{2}\right)                                                                                                                      \\
              & \quad+\left\langle \bm{X}\bar{\bm{A}}\bm{\Delta}_{\bm{B}}^{\top}\bm{Y}^{\top},\bm{X}\bm{\Delta}_{\bm{A}}\bar{\bm{B}}^{\top}\bm{Y}^{\top}\right\rangle -\frac{1}{10}\left\Vert \bm{X}\bar{\bm{A}}\bm{\Delta}_{\bm{B}}^{\top}\bm{Y}^{\top}\right\Vert _{\mathrm{F}}\left\Vert \bm{X}\bm{\Delta}_{\bm{A}}\bar{\bm{B}}^{\top}\bm{Y}^{\top}\right\Vert _{\mathrm{F}}                       \\
              & \overset{\text{(i)}}{\ge}\frac{1}{2}\cdot\left(1-\frac{1}{10}\right)\left(\left\Vert \bm{X}\bar{\bm{A}}\bm{\Delta}_{\bm{B}}^{\top}\bm{Y}^{\top}\right\Vert _{\mathrm{F}}^{2}+\left\Vert \bm{X}\bm{\Delta}_{\bm{A}}\bar{\bm{B}}^{\top}\bm{Y}^{\top}\right\Vert _{\mathrm{F}}^{2}\right)                                                                                                \\
              & \quad+\left\langle \bm{X}\bar{\bm{A}}\bm{\Delta}_{\bm{B}}^{\top}\bm{Y}^{\top},\bm{X}\bm{\Delta}_{\bm{A}}\bar{\bm{B}}^{\top}\bm{Y}^{\top}\right\rangle -\frac{1}{20}\left(\left\Vert \bm{X}\bar{\bm{A}}\bm{\Delta}_{\bm{B}}^{\top}\bm{Y}^{\top}\right\Vert _{\mathrm{F}}^{2}+\left\Vert \bm{X}\bm{\Delta}_{\bm{A}}\bar{\bm{B}}^{\top}\bm{Y}^{\top}\right\Vert _{\mathrm{F}}^{2}\right)
\end{align*}
where (i) uses the fact that $a^{2}+b^{2}\ge2ab$ for any $a,b\in\mathbb{R}$.
This then simplifies to
\begin{equation}
    \beta_{1}\ge\frac{2}{5}\left(\left\Vert \bar{\bm{A}}\bm{\Delta}_{\bm{B}}^{\top}\right\Vert _{\mathrm{F}}^{2}+\left\Vert \bm{\Delta}_{\bm{A}}\bar{\bm{B}}^{\top}\right\Vert _{\mathrm{F}}^{2}\right)+\left\langle \bar{\bm{A}}\bm{\Delta}_{\bm{B}}^{\top},\bm{\Delta}_{\bm{A}}\bar{\bm{B}}^{\top}\right\rangle \label{eq:beta1_lb}
\end{equation}

\paragraph{Controlling $\alpha_{2}+\left\langle \bar{\bm{A}}\bm{\Delta}_{\bm{B}}^{\top},\bm{\Delta}_{\bm{A}}\bar{\bm{B}}^{\top}\right\rangle $
    as a whole. }

Since $\bar{\bm{A}}^{\top}\bar{\bm{A}}=\bar{\bm{B}}^{\top}\bar{\bm{B}}$,
we can decompose the balancing term $\bm{A}^{\top}\bm{A}-\bm{B}^{\top}\bm{B}$
as
\begin{align*}
    \bm{A}^{\top}\bm{A}-\bm{B}^{\top}\bm{B} & =\bm{A}^{\top}\bm{A}-\bar{\bm{A}}^{\top}\bar{\bm{A}}+\bar{\bm{B}}^{\top}\bar{\bm{B}}-\bm{B}^{\top}\bm{B}                                                                                                                                                          \\
                                            & =\bar{\bm{A}}^{\top}\bm{\Delta}_{\bm{A}}-\bar{\bm{B}}^{\top}\bm{\Delta}_{\bm{B}}+\bm{\Delta}_{\bm{A}}^{\top}\bar{\bm{A}}-\bm{\Delta}_{\bm{B}}^{\top}\bar{\bm{B}}+\bm{\Delta}_{\bm{A}}^{\top}\bm{\Delta}_{\bm{A}}-\bm{\Delta}_{\bm{B}}^{\top}\bm{\Delta}_{\bm{B}}.
\end{align*}
Substituting this into $\alpha_{2}$, we have

\begin{align*}
    \alpha_{2} & =\frac{1}{4}\left\langle \bm{A}\left(\bm{A}^{\top}\bm{A}-\bm{B}^{\top}\bm{B}\right),\bm{\Delta}_{\bm{A}}\right\rangle -\frac{1}{4}\left\langle \bm{B}\left(\bm{A}^{\top}\bm{A}-\bm{B}^{\top}\bm{B}\right),\bm{\Delta}_{\bm{B}}\right\rangle                                                                                                                                                                                                                                                                 \\
               & =\frac{1}{4}\left\langle \bar{\bm{A}}^{\top}\bm{\Delta}_{\bm{A}}-\bar{\bm{B}}^{\top}\bm{\Delta}_{\bm{B}}+\bm{\Delta}_{\bm{A}}^{\top}\bar{\bm{A}}-\bm{\Delta}_{\bm{B}}^{\top}\bar{\bm{B}}+\bm{\Delta}_{\bm{A}}^{\top}\bm{\Delta}_{\bm{A}}-\bm{\Delta}_{\bm{B}}^{\top}\bm{\Delta}_{\bm{B}},\bar{\bm{A}}^{\top}\bm{\Delta}_{\bm{A}}-\bm{\bar{B}}^{\top}\bm{\Delta}_{\bm{B}}+\bm{\Delta}_{\bm{A}}^{\top}\bm{\Delta}_{\bm{A}}-\bm{\Delta}_{\bm{B}}^{\top}\bm{\Delta}_{\bm{B}}\right\rangle                       \\
               & =\frac{1}{4}\left\Vert \bar{\bm{A}}^{\top}\bm{\Delta}_{\bm{A}}-\bar{\bm{B}}^{\top}\bm{\Delta}_{\bm{B}}\right\Vert _{\mathrm{F}}^{2}+\frac{1}{4}\left\Vert \bm{\Delta}_{\bm{A}}^{\top}\bm{\Delta}_{\bm{A}}-\bm{\Delta}_{\bm{B}}^{\top}\bm{\Delta}_{\bm{B}}\right\Vert _{\mathrm{F}}^{2}+\frac{1}{2}\left\langle \bar{\bm{A}}^{\top}\bm{\Delta}_{\bm{A}}-\bar{\bm{B}}^{\top}\bm{\Delta}_{\bm{B}},\bm{\Delta}_{\bm{A}}^{\top}\bm{\Delta}_{\bm{A}}-\bm{\Delta}_{\bm{B}}^{\top}\bm{\Delta}_{\bm{B}}\right\rangle \\
               & \quad+\frac{1}{4}\left\langle \bm{\Delta}_{\bm{A}}^{\top}\bar{\bm{A}}-\bm{\Delta}_{\bm{B}}^{\top}\bar{\bm{B}},\bm{\Delta}_{\bm{A}}^{\top}\bm{\Delta}_{\bm{A}}-\bm{\Delta}_{\bm{B}}^{\top}\bm{\Delta}_{\bm{B}}\right\rangle +\frac{1}{4}\left\langle \bar{\bm{A}}^{\top}\bm{\Delta}_{\bm{A}}-\bar{\bm{B}}^{\top}\bm{\Delta}_{\bm{B}},\bm{\Delta}_{\bm{A}}^{\top}\bar{\bm{A}}-\bm{\Delta}_{\bm{B}}^{\top}\bar{\bm{B}}\right\rangle .
\end{align*}
Since $\bm{\Delta}_{\bm{A}}^{\top}\bm{\Delta}_{\bm{A}}-\bm{\Delta}_{\bm{B}}^{\top}\bm{\Delta}_{\bm{B}}$
is symmetric, we can collect the terms to get
\begin{align*}
    \alpha_{2} & =\frac{1}{4}\left\Vert \bar{\bm{A}}^{\top}\bm{\Delta}_{\bm{A}}-\bar{\bm{B}}^{\top}\bm{\Delta}_{\bm{B}}\right\Vert _{\mathrm{F}}^{2}+\frac{1}{4}\left\Vert \bm{\Delta}_{\bm{A}}^{\top}\bm{\Delta}_{\bm{A}}-\bm{\Delta}_{\bm{B}}^{\top}\bm{\Delta}_{\bm{B}}\right\Vert _{\mathrm{F}}^{2}+\frac{3}{4}\left\langle \bar{\bm{A}}^{\top}\bm{\Delta}_{\bm{A}}-\bar{\bm{B}}^{\top}\bm{\Delta}_{\bm{B}},\bm{\Delta}_{\bm{A}}^{\top}\bm{\Delta}_{\bm{A}}-\bm{\Delta}_{\bm{B}}^{\top}\bm{\Delta}_{\bm{B}}\right\rangle \\
               & \quad+\frac{1}{4}\left\langle \bar{\bm{A}}^{\top}\bm{\Delta}_{\bm{A}}-\bar{\bm{B}}^{\top}\bm{\Delta}_{\bm{B}},\bm{\Delta}_{\bm{A}}^{\top}\bar{\bm{A}}-\bm{\Delta}_{\bm{B}}^{\top}\bar{\bm{B}}\right\rangle .
\end{align*}
Furthermore as $a-6ab+9b^{2}\ge0$ for any $a,b\in\mathbb{R},$ let
$a=\left\Vert \bar{\bm{A}}^{\top}\bm{\Delta}_{\bm{A}}-\bar{\bm{B}}^{\top}\bm{\Delta}_{\bm{B}}\right\Vert _{\mathrm{F}},b=\left\Vert \bm{\Delta}_{\bm{A}}^{\top}\bm{\Delta}_{\bm{A}}-\bm{\Delta}_{\bm{B}}^{\top}\bm{\Delta}_{\bm{B}}\right\Vert _{\mathrm{F}}$,
the above equation can be further rewritten as
\begin{equation}
    \alpha_{2}\ge\frac{1}{8}\left\Vert \bar{\bm{A}}^{\top}\bm{\Delta}_{\bm{A}}-\bar{\bm{B}}^{\top}\bm{\Delta}_{\bm{B}}\right\Vert _{\mathrm{F}}^{2}-\frac{7}{8}\left\Vert \bm{\Delta}_{\bm{A}}^{\top}\bm{\Delta}_{\bm{A}}-\bm{\Delta}_{\bm{B}}^{\top}\bm{\Delta}_{\bm{B}}\right\Vert _{\mathrm{F}}^{2}+\frac{1}{4}\left\langle \bar{\bm{A}}^{\top}\bm{\Delta}_{\bm{A}}-\bar{\bm{B}}^{\top}\bm{\Delta}_{\bm{B}},\bm{\Delta}_{\bm{A}}^{\top}\bar{\bm{A}}-\bm{\Delta}_{\bm{B}}^{\top}\bar{\bm{B}}\right\rangle \label{eq:alpha2_lb1}
\end{equation}
Moreover observe that
\begin{align}
     & \frac{1}{4}\left\langle \bar{\bm{A}}^{\top}\bm{\Delta}_{\bm{A}}-\bar{\bm{B}}^{\top}\bm{\Delta}_{\bm{B}},\bm{\Delta}_{\bm{A}}^{\top}\bar{\bm{A}}-\bm{\Delta}_{\bm{B}}^{\top}\bar{\bm{B}}\right\rangle +\left\langle \bar{\bm{A}}\bm{\Delta}_{\bm{B}}^{\top},\bm{\Delta}_{\bm{A}}\bar{\bm{B}}^{\top}\right\rangle \label{eq:alpha2_balance} \\
     & \quad=\frac{1}{4}\left\langle \bar{\bm{A}}^{\top}\bm{\Delta}_{\bm{A}}+\bar{\bm{B}}^{\top}\bm{\Delta}_{\bm{B}},\bm{\Delta}_{\bm{A}}^{\top}\bar{\bm{A}}+\bm{\Delta}_{\bm{B}}^{\top}\bar{\bm{B}}\right\rangle \nonumber                                                                                                                      \\
     & \quad=\frac{1}{4}\left\Vert \bar{\bm{A}}^{\top}\bm{\Delta}_{\bm{A}}+\bar{\bm{B}}^{\top}\bm{\Delta}_{\bm{B}}\right\Vert _{\mathrm{F}}^{2},\nonumber
\end{align}
where the last line uses the fact that $\bm{\Delta}_{\bm{A}}^{\top}\bar{\bm{A}}+\bm{\Delta}_{\bm{B}}^{\top}\bar{\bm{B}}$
is symmetric. Combining \eqref{eq:alpha2_lb1} and \eqref{eq:alpha2_balance},
we have
\begin{align}
    \alpha_{2}+\left\langle \bar{\bm{A}}\bm{\Delta}_{\bm{B}}^{\top},\bm{\Delta}_{\bm{A}}\bar{\bm{B}}^{\top}\right\rangle & \ge\frac{1}{8}\left\Vert \bar{\bm{A}}^{\top}\bm{\Delta}_{\bm{A}}-\bar{\bm{B}}^{\top}\bm{\Delta}_{\bm{B}}\right\Vert _{\mathrm{F}}^{2}-\frac{7}{8}\left\Vert \bm{\Delta}_{\bm{A}}^{\top}\bm{\Delta}_{\bm{A}}-\bm{\Delta}_{\bm{B}}^{\top}\bm{\Delta}_{\bm{B}}\right\Vert _{\mathrm{F}}^{2}\label{eq:alpha2_plus} \\
                                                                                                                         & \ge\frac{1}{8}\left\Vert \bar{\bm{A}}^{\top}\bm{\Delta}_{\bm{A}}-\bar{\bm{B}}^{\top}\bm{\Delta}_{\bm{B}}\right\Vert _{\mathrm{F}}^{2}-\frac{7}{8}\left\Vert \bm{\Delta}\right\Vert _{\mathrm{F}}^{4}.\nonumber
\end{align}
The last line follows from
\begin{equation}
    \left\Vert \bm{\Delta}_{\bm{A}}^{\top}\bm{\Delta}_{\bm{A}}-\bm{\Delta}_{\bm{B}}^{\top}\bm{\Delta}_{\bm{B}}\right\Vert _{\mathrm{F}}\le\left\Vert \bm{\Delta}_{\bm{A}}^{\top}\bm{\Delta}_{\bm{A}}\right\Vert _{\mathrm{F}}+\left\Vert \bm{\Delta}_{\bm{B}}^{\top}\bm{\Delta}_{\bm{B}}\right\Vert _{\mathrm{F}}\le\left\Vert \bm{\Delta}_{\bm{A}}\right\Vert _{\mathrm{F}}^{2}+\left\Vert \bm{\Delta}_{\bm{B}}\right\Vert _{\mathrm{F}}^{2}=\|\bm{\Delta}\|_{\mathrm{F}}^{2}.\label{eq:Delta_balance}
\end{equation}
At this stage we have
\begin{align*}
    \left\langle \nabla f_{\mathrm{clean}}\left(\begin{bmatrix}\bm{A} \\
                                                    \bm{B}
                                                \end{bmatrix}\right),\begin{bmatrix}\bm{\Delta}_{\bm{A}} \\
                                                                         \bm{\Delta}_{\bm{B}}
                                                                     \end{bmatrix}\right\rangle & \stackrel{\text{(i)}}{=}\alpha_{1}+\alpha_{2}                                                                                                                                                                                                                                                                                                                                                  \\
                                                                         & \stackrel{\text{(ii)}}{\ge}\beta_{1}+\alpha_{2}+\left\langle \bar{\bm{A}}\bm{\Delta}_{\bm{B}}^{\top},\bm{\Delta}_{\bm{A}}\bar{\bm{B}}^{\top}\right\rangle -\frac{5}{2p}\left\Vert \mathcal{P}_{\Omega}\left(\bm{X}\bm{\Delta}_{\bm{A}}\bm{\Delta}_{\bm{B}}^{\top}\bm{Y}^{\top}\right)\right\Vert _{\mathrm{F}}^{2}                                                                                                    \\
                                                                         & \stackrel{\text{(iii)}}{\ge}\frac{2}{5}\left(\left\Vert \bar{\bm{A}}\bm{\Delta}_{\bm{B}}^{\top}\right\Vert _{\mathrm{F}}^{2}+\left\Vert \bm{\Delta}_{\bm{A}}\bar{\bm{B}}^{\top}\right\Vert _{\mathrm{F}}^{2}\right)+\frac{1}{8}\left\Vert \bar{\bm{A}}^{\top}\bm{\Delta}_{\bm{A}}-\bar{\bm{B}}^{\top}\bm{\Delta}_{\bm{B}}\right\Vert _{\mathrm{F}}^{2}-\frac{7}{8}\left\Vert \bm{\Delta}\right\Vert _{\mathrm{F}}^{4} \\
                                                                         & \quad-5000\sqrt{\frac{\left(\mu_{1}a_{1}\vee\mu_{2}a_{2}\right)\mu_{0}^{2}r^{2}\log n}{n_{1}n_{2}p}}\sigma_{\max}\|\bm{\Delta}\|_{\mathrm{F}}^{2}-\frac{5}{2}\left\Vert \bm{\Delta}\right\Vert _{\mathrm{F}}^{4}
\end{align*}
where (i) uses \eqref{eq:curvature_decomp}, (ii) uses \eqref{eq:alpha1_lb}
and (iii) uses \eqref{eq:alpha2_plus}, \eqref{eq:Omega_X_Delta_Delta_Y},
and \eqref{eq:beta1_lb}. By the sample complexity assumption and
\eqref{eq:Delta_size}, $\left\Vert \bm{\Delta}\right\Vert _{\mathrm{F}}\le\frac{1}{10}\sqrt{\sigma_{\min}}$
and $n_{1}n_{2}p\ge C_{\mathrm{sample}}\kappa^{2}\left(\mu_{1}a_{1}\vee\mu_{2}a_{2}\right)\mu_{0}^{2}r^{2}\log n$
for some absolute constant $C_{\mathrm{sample}}$ large enough, then
the above equation simplifies to
\begin{align*}
    \left\langle \nabla f_{\mathrm{clean}}\left(\begin{bmatrix}\bm{A} \\
                                                    \bm{B}
                                                \end{bmatrix}\right),\begin{bmatrix}\bm{\Delta}_{\bm{A}} \\
                                                                         \bm{\Delta}_{\bm{B}}
                                                                     \end{bmatrix}\right\rangle & \ge\frac{2}{5}\sigma_{\min}\left\Vert \bm{\Delta}\right\Vert _{\mathrm{F}}^{2}+\frac{1}{8}\left\Vert \bar{\bm{A}}^{\top}\bm{\Delta}_{\bm{A}}-\bar{\bm{B}}^{\top}\bm{\Delta}_{\bm{B}}\right\Vert _{\mathrm{F}}^{2} \\
                                                                         & \quad-5000\sqrt{\frac{\left(\mu_{1}a_{1}\vee\mu_{2}a_{2}\right)\mu_{0}^{2}r^{2}\log n}{n_{1}n_{2}p}}\sigma_{\max}\|\bm{\Delta}\|_{\mathrm{F}}^{2}-\frac{7}{2}\left\Vert \bm{\Delta}\right\Vert _{\mathrm{F}}^{4}                         \\
                                                                         & \ge\frac{1}{10}\sigma_{\min}\left\Vert \bm{\Delta}\right\Vert _{\mathrm{F}}^{2}+\frac{1}{8}\left\Vert \bar{\bm{A}}^{\top}\bm{\Delta}_{\bm{A}}-\bar{\bm{B}}^{\top}\bm{\Delta}_{\bm{B}}\right\Vert _{\mathrm{F}}^{2}                       \\
\end{align*}
where the first line uses $\|\bar{\bm{A}}\|,\|\bar{\bm{B}}\|\le\sqrt{\sigma_{\min}}$.

\paragraph{Including the noise term.}

Now we include the noise term $\bm{E}$. From the definition of $f$ and
$f_{\mathrm{clean}}$ we have
\begin{align*}
    \left\langle \nabla f\left(\begin{bmatrix}\bm{A} \\
                                   \bm{B}
                               \end{bmatrix}\right),\begin{bmatrix}\bm{\Delta}_{\bm{A}} \\
                                                        \bm{\Delta}_{\bm{B}}
                                                    \end{bmatrix}\right\rangle & =\left\langle \nabla f_{\mathrm{clean}}\left(\begin{bmatrix}\bm{A} \\
                                                                                                                                  \bm{B}
                                                                                                                              \end{bmatrix}\right),\begin{bmatrix}\bm{\Delta}_{\bm{A}} \\
                                                                                                                                                       \bm{\Delta}_{\bm{B}}
                                                                                                                                                   \end{bmatrix}\right\rangle -\left\langle \frac{1}{p}\begin{bmatrix}\bm{X}^{\top}\mathcal{P}_{\Omega}(\bm{E})\bm{Y}\bm{B} \\
                                                                                                                                                                                                           \bm{Y}^{\top}\mathcal{P}_{\Omega}(\bm{E})^{\top}\bm{X}\bm{A}
                                                                                                                                                                                                       \end{bmatrix},\begin{bmatrix}\bm{\Delta}_{\bm{A}} \\
                                                                                                                                                                                                                         \bm{\Delta}_{\bm{B}}
                                                                                                                                                                                                                     \end{bmatrix}\right\rangle .
\end{align*}
Using the fact that $\|\bm{B}\|_{\mathrm{F}}\le2\|\bm{A}^{\star}\|_{\mathrm{F}}\le2\sqrt{r\sigma_{\max}}$,
we have
\begin{align*}
    \frac{1}{p}\left\langle \bm{X}^{\top}\mathcal{P}_{\Omega}(\bm{E})\bm{Y}\bm{B}\text{,}\bm{\Delta}_{\bm{A}}\right\rangle & \le\left\Vert \frac{1}{p}\bm{X}^{\top}\mathcal{P}_{\Omega}(\bm{E})\bm{Y}\right\Vert \left\Vert \bm{B}\right\Vert _{\mathrm{F}}\left\Vert \bm{\Delta}_{\bm{A}}\right\Vert _{\mathrm{F}} \\
                                                                                                                           & \le\gamma_{E}\cdot2\sqrt{r\sigma_{\max}}\left\Vert \bm{\Delta}\right\Vert _{\mathrm{F}}
\end{align*}
Then
\begin{align*}
    \left\langle \nabla f\left(\begin{bmatrix}\bm{A} \\
                                   \bm{B}
                               \end{bmatrix}\right),\begin{bmatrix}\bm{\Delta}_{\bm{A}} \\
                                                        \bm{\Delta}_{\bm{B}}
                                                    \end{bmatrix}\right\rangle & \ge\frac{1}{10}\sigma_{\min}\left\Vert \bm{\Delta}\right\Vert _{\mathrm{F}}^{2}+\frac{1}{8}\left\Vert \bar{\bm{A}}^{\top}\bm{\Delta}_{\bm{A}}-\bar{\bm{B}}^{\top}\bm{\Delta}_{\bm{B}}\right\Vert _{\mathrm{F}}^{2}-\gamma_{E}\cdot2\sqrt{r\sigma_{\max}}\left\Vert \bm{\Delta}\right\Vert _{\mathrm{F}} \\
                                                                & \ge\frac{1}{20}\sigma_{\min}\left\Vert \bm{\Delta}\right\Vert _{\mathrm{F}}^{2}+\frac{1}{8}\left\Vert \bar{\bm{A}}^{\top}\bm{\Delta}_{\bm{A}}-\bar{\bm{B}}^{\top}\bm{\Delta}_{\bm{B}}\right\Vert _{\mathrm{F}}^{2}-20r\kappa\gamma_{E}^{2}                                                                             \\
                                                                & \quad+\frac{1}{20}\left(\sqrt{\sigma_{\min}}\left\Vert \bm{\Delta}\right\Vert _{\mathrm{F}}-20\sqrt{r\kappa}\cdot\gamma_{E}\right)^{2}                                                                                                                                                                                 \\
                                                                & \ge\frac{1}{20}\sigma_{\min}\left\Vert \bm{\Delta}\right\Vert _{\mathrm{F}}^{2}+\frac{1}{8}\left\Vert \bar{\bm{A}}^{\top}\bm{\Delta}_{\bm{A}}-\bar{\bm{B}}^{\top}\bm{\Delta}_{\bm{B}}\right\Vert _{\mathrm{F}}^{2}-20r\kappa\gamma_{E}^{2}
\end{align*}

\subsection{Proof of local smoothness condition~\texorpdfstring{\eqref{eq:local_smoo}}{equation}}

Let $\bm{V}=\begin{bmatrix}\bm{V}_{\bm{A}} \\
        \bm{V}_{\bm{B}}
    \end{bmatrix}$ be an arbitrary $(n_{1}+n_{2})\times r$ matrix such that $\|\bm{V}\|_{\mathrm{F}}\le1$.
It suffices to control $\left\langle \nabla f(\bm{F}),\bm{V}\right\rangle ^{2}$.
We start with the noiseless objective function,
\begin{align*}
    \left\langle \nabla f_{\mathrm{clean}}\left(\begin{bmatrix}\bm{A} \\
                                                    \bm{B}
                                                \end{bmatrix}\right),\begin{bmatrix}\bm{V}_{\bm{A}} \\
                                                                         \bm{V}_{\bm{B}}
                                                                     \end{bmatrix}\right\rangle & =\frac{1}{p}\left\langle \bm{X}^{\top}\mathcal{P}_{\Omega}(\bm{X}\bm{A}\bm{B}^{\top}\bm{Y}^{\top}-\bm{L}^{\star})\bm{Y}\bm{B},\bm{V}_{\bm{A}}\right\rangle +\frac{1}{p}\left\langle \bm{Y}^{\top}\mathcal{P}_{\Omega}(\bm{X}\bm{A}\bm{B}^{\top}\bm{Y}^{\top}-\bm{L}^{\star})^{\top}\bm{X}\bm{A},\bm{V}_{\bm{B}}\right\rangle \\
                                                                         & \quad+\frac{1}{2}\left\langle \bm{A}\left(\bm{A}^{\top}\bm{A}-\bm{B}^{\top}\bm{B}\right),\bm{V}_{\bm{A}}\right\rangle -\frac{1}{2}\left\langle \bm{B}\left(\bm{A}^{\top}\bm{A}-\bm{B}^{\top}\bm{B}\right),\bm{V}_{\bm{B}}\right\rangle                                                                                                              \\
                                                                         & =\underbrace{\frac{1}{p}\left\langle \mathcal{P}_{\Omega}(\bm{X}\bm{A}\bm{B}^{\top}\bm{Y}^{\top}-\bm{L}^{\star}),\mathcal{P}_{\Omega}\left(\bm{X}\bm{V}_{\bm{A}}\bm{B}^{\top}\bm{Y}^{\top}+\bm{X}\bm{A}^{\top}\bm{V}_{\bm{B}}^{\top}\bm{Y}^{\top}\right)\right\rangle }_{\alpha_{1}}                                                                \\
                                                                         & \quad+\underbrace{\frac{1}{2}\left\langle \bm{A}\left(\bm{A}^{\top}\bm{A}-\bm{B}^{\top}\bm{B}\right),\bm{V}_{\bm{A}}\right\rangle -\frac{1}{2}\left\langle \bm{B}\left(\bm{A}^{\top}\bm{A}-\bm{B}^{\top}\bm{B}\right),\bm{V}_{\bm{B}}\right\rangle }_{\alpha_{2}}
\end{align*}

\paragraph{Controlling $\alpha_{1}$.}

We further decompose $\alpha_{1}$ by
\[
    \alpha_{1}=\frac{1}{p}\left\langle \mathcal{P}_{\Omega}\left(\bm{X}\bar{\bm{A}}\bm{\Delta}_{\bm{B}}^{\top}\bm{Y}^{\top}+\bm{X}\bm{\Delta}_{\bm{A}}\bar{\bm{B}}^{\top}\bm{Y}^{\top}+\bm{X}\bm{\Delta}_{\bm{A}}\bm{\Delta}_{\bm{B}}^{\top}\bm{Y}^{\top}\right),\mathcal{P}_{\Omega}\left(\bm{X}\bm{V}_{\bm{A}}\bm{B}^{\top}\bm{Y}^{\top}+\bm{X}\bm{A}^{\top}\bm{V}_{\bm{B}}^{\top}\bm{Y}^{\top}\right)\right\rangle
\]
By Cauchy-Schwarz inequality and $k\sum_{i=1}^{k}a_{i}^{2}\ge\left(\sum_{i=1}^{k}a_{i}\right)$
for any $k\in\mathbb{N},\left\{ a_{k}\right\} \subset\mathbb{R}$,
\begin{align*}
    \alpha_{1}^{2} & \le\frac{1}{p}\left(\left\Vert \mathcal{P}_{\Omega}\left(\bm{X}\bar{\bm{A}}\bm{\Delta}_{\bm{B}}^{\top}\bm{Y}^{\top}\right)\right\Vert _{\mathrm{F}}+\left\Vert \mathcal{P}_{\Omega}\left(\bm{X}\bm{\Delta}_{\bm{A}}\bar{\bm{B}}^{\top}\bm{Y}^{\top}\right)\right\Vert _{\mathrm{F}}+\left\Vert \mathcal{P}_{\Omega}\left(\bm{X}\bm{\Delta}_{\bm{A}}\bm{\Delta}_{\bm{B}}^{\top}\bm{Y}^{\top}\right)\right\Vert _{\mathrm{F}}\right)^{2}         \\
                   & \quad\cdot\frac{1}{p}\left(\left\Vert \mathcal{P}_{\Omega}\left(\bm{X}\bm{V}_{\bm{A}}\bm{B}^{\top}\bm{Y}^{\top}\right)\right\Vert _{\mathrm{F}}+\left\Vert \mathcal{P}_{\Omega}\left(\bm{X}\bm{A}^{\top}\bm{V}_{\bm{B}}^{\top}\bm{Y}^{\top}\right)\right\Vert _{\mathrm{F}}\right)^{2}                                                                                                                                                         \\
                   & \le\frac{3}{p}\left(\left\Vert \mathcal{P}_{\Omega}\left(\bm{X}\bar{\bm{A}}\bm{\Delta}_{\bm{B}}^{\top}\bm{Y}^{\top}\right)\right\Vert _{\mathrm{F}}^{2}+\left\Vert \mathcal{P}_{\Omega}\left(\bm{X}\bm{\Delta}_{\bm{A}}\bar{\bm{B}}^{\top}\bm{Y}^{\top}\right)\right\Vert _{\mathrm{F}}^{2}+\left\Vert \mathcal{P}_{\Omega}\left(\bm{X}\bm{\Delta}_{\bm{A}}\bm{\Delta}_{\bm{B}}^{\top}\bm{Y}^{\top}\right)\right\Vert _{\mathrm{F}}^{2}\right) \\
                   & \quad\cdot\frac{2}{p}\left(\left\Vert \mathcal{P}_{\Omega}\left(\bm{X}\bm{V}_{\bm{A}}\bm{B}^{\top}\bm{Y}^{\top}\right)\right\Vert _{\mathrm{F}}^{2}+\left\Vert \mathcal{P}_{\Omega}\left(\bm{X}\bm{A}^{\top}\bm{V}_{\bm{B}}^{\top}\bm{Y}^{\top}\right)\right\Vert _{\mathrm{F}}^{2}\right)
\end{align*}
Applying Lemma~\ref{lem:P_Omega_IP},
\begin{align*}
    \alpha_{1}^{2} & \le24n^{2}\left(2\left\Vert \bm{\Delta}_{\bm{B}}\right\Vert _{\mathrm{F}}^{2}\left\Vert \bm{X}\bar{\bm{A}}\right\Vert _{2,\infty}^{2}+2\left\Vert \bm{\Delta}_{\bm{A}}\right\Vert _{\mathrm{F}}^{2}\left\Vert \bm{Y}\bar{\bm{B}}\right\Vert _{2,\infty}^{2}+2\left\Vert \bm{\Delta}_{\bm{B}}\right\Vert _{\mathrm{F}}^{2}\left\Vert \bm{X}\bm{\Delta}_{\bm{A}}\right\Vert _{2,\infty}^{2}\right) \\
                   & \quad\cdot\left(2\left\Vert \bm{V}_{\bm{A}}\right\Vert _{\mathrm{F}}^{2}\left\Vert \bm{Y}\bar{\bm{B}}\right\Vert _{2,\infty}^{2}+2\left\Vert \bm{V}_{\bm{B}}\right\Vert _{\mathrm{F}}^{2}\left\Vert \bm{X}\bar{\bm{A}}\right\Vert _{2,\infty}^{2}\right)                                                                                                                                         \\
                   & \overset{\text{(i)}}{\le}24n^{2}\cdot54\frac{\mu_{0}r}{n_{\min}}\sigma_{\max}\left\Vert \bm{\Delta}\right\Vert _{\mathrm{F}}^{2}\cdot4\frac{\mu_{0}r}{n_{\min}}\sigma_{\max}                                                                                                                                                                                                                     \\
                   & =5184\left(\frac{n}{n_{\min}}\right)^{2}\mu_{0}^{2}r^{2}\sigma_{\max}^{2}\left\Vert \bm{\Delta}\right\Vert _{\mathrm{F}}^{2},
\end{align*}
where (i) uses \eqref{eq:XA_2_infty} and \eqref{eq:XDeltaA_2_infty}.

\paragraph{Controlling $\alpha_{2}$. }

We first see that since $\bar{\bm{A}}^{\top}\bar{\bm{A}}-\bar{\bm{B}}^{\top}\bar{\bm{B}}=\bm{\Sigma}-\bm{\Sigma}=0$,
\begin{align*}
    \bm{A}^{\top}\bm{A}-\bm{B}^{\top}\bm{B} & =\bm{A}^{\top}\bm{A}-\bar{\bm{A}}^{\top}\bar{\bm{A}}+\bar{\bm{B}}^{\top}\bar{\bm{B}}-\bm{B}^{\top}\bm{B}                                                                                                                                                         \\
                                            & =\bar{\bm{A}}^{\top}\bm{\Delta}_{\bm{A}}-\bar{\bm{B}}^{\top}\bm{\Delta}_{\bm{B}}+\bm{\Delta}_{\bm{A}}^{\top}\bar{\bm{A}}-\bm{\Delta}_{\bm{B}}^{\top}\bar{\bm{B}}+\bm{\Delta}_{\bm{A}}^{\top}\bm{\Delta}_{\bm{A}}-\bm{\Delta}_{\bm{B}}^{\top}\bm{\Delta}_{\bm{B}}
\end{align*}
where the second line uses the decomposition $\bm{A}=\bar{\bm{A}}+\bm{\Delta}_{\bm{A}}$
and $\bm{B}=\bar{\bm{B}}+\bm{\Delta}_{\bm{B}}$. Then as
\[\bar{\bm{A}}^{\top}\bm{\Delta}_{\bm{A}}-\bar{\bm{B}}^{\top}\bm{\Delta}_{\bm{B}}=(\bm{\Delta}_{\bm{A}}^{\top}\bar{\bm{A}}-\bm{\Delta}_{\bm{B}}^{\top}\bar{\bm{B}})^{\top},\]
we have
\begin{align*}
    \left\Vert \bm{A}^{\top}\bm{A}-\bm{B}^{\top}\bm{B}\right\Vert _{\mathrm{F}}^{2} & \le\left(\left\Vert \bar{\bm{A}}^{\top}\bm{\Delta}_{\bm{A}}-\bar{\bm{B}}^{\top}\bm{\Delta}_{\bm{B}}\right\Vert _{\mathrm{F}}+\left\Vert \bar{\bm{A}}^{\top}\bm{\Delta}_{\bm{A}}-\bar{\bm{B}}^{\top}\bm{\Delta}_{\bm{B}}\right\Vert _{\mathrm{F}}+\left\Vert \bm{\Delta}_{\bm{A}}^{\top}\bm{\Delta}_{\bm{A}}-\bm{\Delta}_{\bm{B}}^{\top}\bm{\Delta}_{\bm{B}}\right\Vert _{\mathrm{F}}\right)^{2} \\
                                                                                    & \le6\left\Vert \bar{\bm{A}}^{\top}\bm{\Delta}_{\bm{A}}-\bar{\bm{B}}^{\top}\bm{\Delta}_{\bm{B}}\right\Vert _{\mathrm{F}}^{2}+3\left\Vert \bm{\Delta}_{\bm{A}}^{\top}\bm{\Delta}_{\bm{A}}-\bm{\Delta}_{\bm{B}}^{\top}\bm{\Delta}_{\bm{B}}\right\Vert _{\mathrm{F}}^{2}                                                                                                                            \\
                                                                                    & \le6\left\Vert \bar{\bm{A}}^{\top}\bm{\Delta}_{\bm{A}}-\bar{\bm{B}}^{\top}\bm{\Delta}_{\bm{B}}\right\Vert _{\mathrm{F}}^{2}+3\left\Vert \bm{\Delta}\right\Vert _{\mathrm{F}}^{4}
\end{align*}
where the last line follows from \eqref{eq:Delta_balance}. Furthermore
by \eqref{eq:A_size},
\begin{align*}
    \left\Vert \bm{A}^{\top}\bm{V}_{\bm{A}}-\bm{B}^{\top}\bm{V}_{\bm{B}}\right\Vert _{\mathrm{F}} & \le\left\Vert \bm{A}^{\top}\bm{V}_{\bm{A}}\right\Vert _{\mathrm{F}}+\left\Vert \bm{B}^{\top}\bm{V}_{\bm{B}}\right\Vert _{\mathrm{F}} \\
                                                                                                  & \le2\cdot2\sqrt{\sigma_{\max}}=4\sqrt{\sigma_{\max}},
\end{align*}
Then
\begin{align*}
    \alpha_{2}^{2} & \le\frac{1}{4}\left\langle \bm{A}^{\top}\bm{A}-\bm{B}^{\top}\bm{B},\bm{A}^{\top}\bm{V}_{\bm{A}}-\bm{B}^{\top}\bm{V}_{\bm{B}}\right\rangle ^{2}                                                              \\
                   & =\frac{1}{4}\left\Vert \bm{A}^{\top}\bm{A}-\bm{B}^{\top}\bm{B}\right\Vert _{\mathrm{F}}^{2}\left\Vert \bm{A}^{\top}\bm{V}_{\bm{A}}-\bm{B}^{\top}\bm{V}_{\bm{B}}\right\Vert _{\mathrm{F}}^{2}                \\
                   & \le6\sigma_{\max}\left\Vert \bar{\bm{A}}^{\top}\bm{\Delta}_{\bm{A}}-\bar{\bm{B}}^{\top}\bm{\Delta}_{\bm{B}}\right\Vert _{\mathrm{F}}^{2}+3\sigma_{\max}\left\Vert \bm{\Delta}\right\Vert _{\mathrm{F}}^{4}.
\end{align*}
Combining $\alpha_{1}$ and $\alpha_{2}$,
\begin{align}
    \left\Vert \nabla f_{\mathrm{clean}}(\bm{F})\right\Vert _{\mathrm{F}}^{2} & =\sup_{\left\Vert \bm{V}\right\Vert _{\mathrm{F}}\le1}\left\langle \nabla f(\bm{F}),\bm{V}\right\rangle ^{2}\label{eq:grad_F_clean}                                                                                                                                                                                                               \\
                                                                              & \le2\left(\alpha_{1}^{2}+\alpha_{2}^{2}\right)\nonumber                                                                                                                                                                                                                                                                                           \\
                                                                              & \le10368\left(\frac{n}{n_{\min}}\right)^{2}\mu_{0}^{2}r^{2}\sigma_{\max}^{2}\left\Vert \bm{\Delta}\right\Vert _{\mathrm{F}}^{2}+12\sigma_{\max}\left\Vert \bar{\bm{A}}^{\top}\bm{\Delta}_{\bm{A}}-\bar{\bm{B}}^{\top}\bm{\Delta}_{\bm{B}}\right\Vert _{\mathrm{F}}^{2}+6\sigma_{\max}\left\Vert \bm{\Delta}\right\Vert _{\mathrm{F}}^{4}\nonumber \\
                                                                              & \le10374\left(\frac{n}{n_{\min}}\right)^{2}\mu_{0}^{2}r^{2}\sigma_{\max}^{2}\left\Vert \bm{\Delta}\right\Vert _{\mathrm{F}}^{2}+12\sigma_{\max}\left\Vert \bar{\bm{A}}^{\top}\bm{\Delta}_{\bm{A}}-\bar{\bm{B}}^{\top}\bm{\Delta}_{\bm{B}}\right\Vert _{\mathrm{F}}^{2}.\nonumber
\end{align}
The last line uses \eqref{eq:Delta_size}.

\paragraph{Including the noise term}

Finally we add the noise term. First we see that by \eqref{eq:A_size},
\begin{align*}
    \|\nabla f_{\mathrm{diff}}(\bm{F})\|_{\mathrm{F}}^{2} & =\left\Vert \frac{1}{p}\bm{X}^{\top}\mathcal{P}_{\Omega}(\bm{E})\bm{Y}\bm{B}\right\Vert _{\mathrm{F}}^{2}+\left\Vert \frac{1}{p}\bm{Y}^{\top}\mathcal{P}_{\Omega}(\bm{E})^{\top}\bm{X}\bm{A}\right\Vert _{\mathrm{F}}^{2} \\
                                                          & \le\left\Vert \frac{1}{p}\bm{X}^{\top}\mathcal{P}_{\Omega}(\bm{E})\bm{Y}\right\Vert ^{2}\cdot\left(\left\Vert \bm{A}\right\Vert _{\mathrm{F}}^{2}+\left\Vert \bm{B}\right\Vert _{\mathrm{F}}^{2}\right)                   \\
                                                          & \le2r\sigma_{\max}\gamma_{E}^{2}.
\end{align*}
Add this to \eqref{eq:grad_F_clean} to reach
\begin{align*}
    \|\nabla f(\bm{F})\|_{\mathrm{F}}^{2} & \le2\left\Vert \nabla f_{\mathrm{clean}}(\bm{F})\right\Vert _{\mathrm{F}}^{2}+2\|\nabla f_{\mathrm{diff}}(\bm{F})\|_{\mathrm{F}}^{2}                                                                                                                                   \\
                                          & \le22000\left(\frac{n}{n_{\min}}\right)^{2}\mu_{0}^{2}r^{2}\sigma_{\max}^{2}\left\Vert \bm{\Delta}\right\Vert _{\mathrm{F}}^{2}+24\sigma_{\max}\left\Vert \bar{\bm{A}}^{\top}\bm{\Delta}_{\bm{A}}-\bar{\bm{B}}^{\top}\bm{\Delta}_{\bm{B}}\right\Vert _{\mathrm{F}}^{2} \\
                                          & \quad+4r\sigma_{\max}\gamma_{E}^{2}.
\end{align*}

%% file: sections/13-supporting-lemmas.tex
\section{Supporting Lemmas}

\begin{lemma}[Lemma 6 in \citet{chen2015incoherence}, modified]\label{lem:PT_POmega}
    Suppose that the sample complexity satisfies
    \[n_{1}n_{2}p\ge C_{\mathrm{sample}}\left(\mu_{1}a_{1}\vee\mu_{2}a_{2}\right)\log n\]
    for some large enough constant $C_{\mathrm{sample}}$, then with probability
    at least $1-n^{-10}$,
    \[
        \left\Vert \frac{1}{p}\mathcal{P}_{T}\mathcal{P}_{\Omega}\mathcal{P}_{T}-\mathcal{P}_{T}\right\Vert \le\frac{1}{10}.
    \]

\end{lemma}

\begin{lemma}[Lemma 35 in \citet{ma2018implicit}]\label{lem:Orthogonal_Procrustes}
    For any matrix $\bm{A}_{1},\bm{A}_{2}\in\mathbb{R}^{a\times r}$,
    consider the orthogonal Procrustes problem
    \[
        \min_{\bm{R}\in\mathcal{O}^{r\times r}}\|\bm{A}_{1}\bm{R}-\bm{A}_{2}\|_{\mathrm{F}}.
    \]
    Then $\widehat{\bm{R}}$ is the minimizer of this problem if and only
    if $\widehat{\bm{R}}^{\top}\bm{A}_{1}^{\top}\bm{A}_{2}$ is symmetric
    and positive semidefinite.

\end{lemma}

\begin{lemma}\label{lem:P_Omega_IP} Suppose the sample complexity
    satisfies $n_{1}n_{2}p\ge C(\mu_{1}a_{1}\vee\mu_{2}a_{2})\log n$
    for some constant $C>0$ large enough, then with probability at least
    $1-n^{-10}$, any $\bm{A}_{1},\bm{A}_{2}\in\mathbb{R}^{a_{1}\times r_{1}},\bm{B}_{1},\bm{B}_{2}\in\mathbb{R}^{a_{2}\times r_{2}}$
    satisfies

    \begin{align}
         & \frac{1}{p}\left\langle \mathcal{P}_{\Omega}\left(\bm{X}\bm{A}_{1}\bm{B}_{1}^{\top}\bm{Y}^{\top}\right),\mathcal{P}_{\Omega}\left(\bm{X}\bm{A}_{2}\bm{B}_{2}^{\top}\bm{Y}^{\top}\right)\right\rangle \label{eq:P_Omega_IP}                                              \\
         & \quad\le2n\cdot\min_{k,k',l,l'\in\{1,2\},k\neq k',l\neq l'}\left\{ \left(\|\bm{A}_{k}\|_{\mathrm{F}}^{2}+\|\bm{B}_{l}\|_{\mathrm{F}}^{2}\right)\left\Vert \bm{X}\bm{A}_{k'}\right\Vert _{2,\infty}\left\Vert \bm{Y}\bm{B}_{l'}\right\Vert _{2,\infty}\right\} \nonumber \\
         & \qquad\wedge\left(\|\bm{A}_{1}\|_{\mathrm{F}}^{2}+\|\bm{A}_{2}\|_{\mathrm{F}}^{2}\right)\left\Vert \bm{Y}\bm{B}_{1}\right\Vert _{2,\infty}\left\Vert \bm{Y}\bm{B}_{2}\right\Vert _{2,\infty}\nonumber                                                                   \\
         & \qquad\wedge\left(\|\bm{B}_{1}\|_{\mathrm{F}}^{2}+\|\bm{B}_{2}\|_{\mathrm{F}}^{2}\right)\left\Vert \bm{X}\bm{A}_{1}\right\Vert _{2,\infty}\left\Vert \bm{X}\bm{A}_{2}\right\Vert _{2,\infty}.\nonumber
    \end{align}

\end{lemma}

\begin{lemma}\label{lem:P_Omega_norm} Suppose the sample complexity
    satisfies $n^{2}p\ge C\mu_{1}a\log n$ for some constant $C>0$ large
    enough, then with probability at least $1-n^{-10}$, any $\bm{A},\bm{B}\in\mathbb{R}^{a\times r}$
    satisfies

    \begin{align}
         & \frac{1}{p}\left\Vert \mathcal{P}_{\Omega}\left(\bm{X}\bm{A}\bm{B}^{\top}\bm{Y}^{\top}\right)\right\Vert _{\mathrm{F}}^{2}\le80\left(\sqrt{\frac{\mu_{1}a_{1}n_{2}p}{n_{1}}\log n}\|\bm{A}\|_{\mathrm{F}}^{2}\|\bm{Y}\bm{B}\|_{2,\infty}^{2}\wedge\sqrt{\frac{\mu_{2}a_{2}n_{1}p}{n_{2}}\log n}\|\bm{B}\|_{\mathrm{F}}^{2}\|\bm{X}\bm{A}\|_{2,\infty}^{2}\right)+\left\Vert \bm{A}\right\Vert _{\mathrm{F}}^{2}\left\Vert \bm{B}\right\Vert _{\mathrm{F}}^{2}\label{eq:P_Omega_norm}
    \end{align}

\end{lemma}

\subsection{Proof of Lemma \ref{lem:P_Omega_IP}}

Observe that the first term on the left hand side of (\ref{lem:P_Omega_IP})
can be rewritten as
\begin{align}
     & \frac{1}{p}\left\langle \mathcal{P}_{\Omega}\left(\bm{X}\bm{A}_{1}\bm{B}_{1}^{\top}\bm{Y}^{\top}\right),\mathcal{P}_{\Omega}\left(\bm{X}\bm{A}_{2}\bm{B}_{2}^{\top}\bm{Y}^{\top}\right)\right\rangle \label{eq:P_OmegaXABY_IP}                                                                            \\
     & \quad=\sum_{i,j}\frac{1}{p}\delta_{ij}\left[\bm{X}\bm{A}_{1}\bm{B}_{1}^{\top}\bm{Y}^{\top}\right]\left[\bm{X}\bm{A}_{2}\bm{B}_{2}^{\top}\bm{Y}^{\top}\right]_{ij}\nonumber                                                                                                                                \\
     & \quad=\sum_{i,j}\frac{1}{p}\delta_{ij}\left(\bm{X}\bm{A}_{1}\right)_{i,\cdot}\left(\bm{Y}\bm{B}_{1}\right)_{j,\cdot}^{\top}\left(\bm{X}\bm{A}_{2}\right)_{i,\cdot}\left(\bm{Y}\bm{B}_{2}\right)_{j,\cdot}^{\top}\nonumber                                                                                 \\
     & \quad\le\sum_{i,j}\frac{1}{p}\delta_{ij}\left\Vert \left(\bm{X}\bm{A}_{1}\right)_{i,\cdot}\right\Vert \left\Vert \left(\bm{Y}\bm{B}_{1}\right)_{j,\cdot}\right\Vert \left\Vert \left(\bm{X}\bm{A}_{2}\right)_{i,\cdot}\right\Vert \left\Vert \left(\bm{Y}\bm{B}_{2}\right)_{j,\cdot}\right\Vert \nonumber
\end{align}
where $(\cdot)_{i,\cdot}$ denotes the $1\times r$ row vector. Bounding
$\|(\bm{X}\bm{A}_{2})_{i,\cdot}\|,\|(\bm{Y}\bm{B}_{2})_{j,\cdot}\|$
by $\|\bm{X}\bm{A}_{2}\|_{2,\infty},\|\bm{Y}\bm{B}_{2}\|_{2,\infty}$,
we have
\begin{align}
     & \frac{1}{p}\left\langle \mathcal{P}_{\Omega}\left(\bm{X}\bm{A}_{1}\bm{B}_{1}^{\top}\bm{Y}^{\top}\right),\mathcal{P}_{\Omega}\left(\bm{X}\bm{A}_{2}\bm{B}_{2}^{\top}\bm{Y}^{\top}\right)\right\rangle \label{eq:P_Omega_XABY_IP2}                                                                                        \\
     & \quad\le\sum_{i,j}\frac{1}{p}\delta_{ij}\left\Vert \left(\bm{X}\bm{A}_{1}\right)_{i,\cdot}\right\Vert \left\Vert \left(\bm{Y}\bm{B}_{1}\right)_{j,\cdot}\right\Vert \left\Vert \bm{X}\bm{A}_{2}\right\Vert _{2,\infty}\left\Vert \bm{Y}\bm{B}_{2}\right\Vert _{2,\infty}\nonumber                                      \\
     & \quad\le\sum_{i,j}\frac{1}{p}\delta_{ij}\cdot\frac{1}{2}\left(\left\Vert \left(\bm{X}\bm{A}_{1}\right)_{i,\cdot}\right\Vert ^{2}+\left\Vert \left(\bm{Y}\bm{B}_{1}\right)_{j,\cdot}\right\Vert ^{2}\right)\left\Vert \bm{X}\bm{A}_{2}\right\Vert _{2,\infty}\left\Vert \bm{Y}\bm{B}_{2}\right\Vert _{2,\infty}\nonumber
\end{align}
Now we consider $\sum_{i,j}\frac{1}{p}\delta_{ij}\cdot\frac{1}{2}\left\Vert \left(\bm{X}\bm{A}_{1}\right)_{i,\cdot}\right\Vert ^{2}$,
which can be rewritten as
\begin{align}
    \sum_{i,j}\frac{1}{p}\delta_{ij}\left(\bm{X}\bm{A}_{1}\right)_{i,\cdot}\left(\bm{X}\bm{A}_{1}\right)_{i,\cdot}^{\top} & =\sum_{i,j}\frac{1}{p}\delta_{ij}\bm{X}_{i,\cdot}\bm{A}_{1}\left(\bm{X}_{i,\cdot}\bm{A}_{1}\right)^{\top}\label{eq:delta_XA}                                             \\
                                                                                                                          & =\sum_{i,j}\frac{1}{p}\left\langle \bm{A}_{1}\bm{A}_{1}^{\top},\delta_{ij}\bm{X}_{i,\cdot}^{\top}\bm{X}_{i,\cdot}\right\rangle \nonumber                                 \\
                                                                                                                          & \stackrel{\text{(i)}}{\le}\frac{1}{p}\|\bm{A}_{1}\bm{A}_{1}^{\top}\|_{\star}\left\Vert \sum_{i,j}\delta_{ij}\bm{X}_{i,\cdot}^{\top}\bm{X}_{i,\cdot}\right\Vert \nonumber \\
                                                                                                                          & \stackrel{\text{(ii)}}{\le}\frac{1}{p}\|\bm{A}_{1}\|_{\mathrm{F}}^{2}\left\Vert \sum_{i,j}\delta_{ij}\bm{X}_{i,\cdot}^{\top}\bm{X}_{i,\cdot}\right\Vert \nonumber
\end{align}
where (i) uses matrix H\"{o}lder's inequality and (ii) uses the fact
that
\[
    \|\bm{A}_{1}\bm{A}_{1}^{\top}\|_{\star}=\sum_{i=1}^{r}\sigma_{i}(\bm{A}_{1}\bm{A}_{1}^{\top})=\sum_{i=1}^{r}\sigma_{i}(\bm{A}_{1})^{2}=\|\bm{A}_{1}\|_{\mathrm{F}}^{2}.
\]
To control $\|\sum_{i,j}\delta_{ij}\bm{X}_{i,\cdot}^{\top}\bm{X}_{i,\cdot}\|$,
consider the decomposition
\begin{equation}
    \sum_{i,j}\delta_{ij}\bm{X}_{i,\cdot}^{\top}\bm{X}_{i,\cdot}=\sum_{i,j}\left(\delta_{ij}-p\right)\bm{X}_{i,\cdot}^{\top}\bm{X}_{i,\cdot}+p\sum_{i,j}\bm{X}_{i,\cdot}^{\top}\bm{X}_{i,\cdot}\label{eq:XT-X-decomposition}
\end{equation}
For the second term in \eqref{eq:XT-X-decomposition},
\begin{align*}
    p\sum_{i,j}\bm{X}_{i,\cdot}^{\top}\bm{X}_{i,\cdot} & =n_{2}p\sum_{i}\bm{X}_{i,\cdot}^{\top}\bm{X}_{i,\cdot} \\
                                                       & =n_{2}p\bm{X}^{\top}\bm{X}                             \\
                                                       & =n_{2}p\bm{I}_{r}
\end{align*}
To control the first term in \eqref{eq:XT-X-decomposition} we will
apply matrix Bernstein's inequality. We have that
\[
    \left\Vert \left(\delta_{ij}-p\right)\bm{X}_{i,\cdot}^{\top}\bm{X}_{i,\cdot}\right\Vert \le\left\Vert \bm{X}_{i,\cdot}\right\Vert ^{2}\le\frac{\mu_{1}a_{1}}{n_{1}}.
\]
As $\bm{X}_{i,\cdot}\bm{X}_{i,\cdot}^{\top}\le\|\bm{X}\|_{2,\infty}^{2}$
is a scalar,
\begin{align*}
    \mathbb{E}\sum_{i,j}\left(\delta_{ij}-p\right)^{2}\bm{X}_{i,\cdot}^{\top}\bm{X}_{i,\cdot}\bm{X}_{i,\cdot}^{\top}\bm{X}_{i,\cdot} & =\sum_{i}n_{2}p(1-p)\bm{X}_{i,\cdot}^{\top}\bm{X}_{i,\cdot}\bm{X}_{i,\cdot}^{\top}\bm{X}_{i,\cdot} \\
                                                                                                                                     & \preccurlyeq\sum_{i}n_{2}p(1-p)\|\bm{X}\|_{2,\infty}^{2}\bm{X}_{i,\cdot}^{\top}\bm{X}_{i,\cdot}    \\
                                                                                                                                     & =\sum_{i}n_{2}p(1-p)\|\bm{X}\|_{2,\infty}^{2}\bm{I}_{r}.
\end{align*}
Then
\[
    \left\Vert \mathbb{E}\sum_{i,j}\left(\delta_{ij}-p\right)^{2}\bm{X}_{i,\cdot}^{\top}\bm{X}_{i,\cdot}\bm{X}_{i,\cdot}^{\top}\bm{X}_{i,\cdot}\right\Vert \le n_{2}p\cdot\|\bm{X}\|_{2,\infty}^{2}=\frac{\mu_{1}a_{1}n_{2}p}{n_{1}}.
\]
Applying matrix Bernstein's inequality yields that with probability at least
$1-n^{-10}$,
\begin{equation}
    \left\Vert \sum_{i,j}\left(\delta_{ij}-p\right)\bm{X}_{i,\cdot}^{\top}\bm{X}_{i,\cdot}\right\Vert \le40\left(\sqrt{\frac{\mu_{1}a_{1}n_{2}p}{n_{1}}\log n}+\frac{\mu_{1}a_{1}}{n_{1}}\log n\right).\label{eq:delta-p_XTX}
\end{equation}
Combining the two terms in \eqref{eq:XT-X-decomposition}, we have
that
\begin{equation}
    \left\Vert \sum_{i,j}\delta_{ij}\bm{X}_{i,\cdot}^{\top}\bm{X}_{i,\cdot}\right\Vert \le40\left(\sqrt{\frac{\mu_{1}a_{1}n_{2}p}{n_{1}}\log n}+\frac{\mu_{1}a_{1}}{n_{1}}\log n\right)+n_{2}p\le2n_{2}p\label{eq:delta_XT_X}
\end{equation}
when $n_{1}n_{2}p\ge C\mu_{1}a_{1}\log n$ for some large enough constant
$C$. Feeding these back to \eqref{eq:P_Omega_XABY_IP2} and \eqref{eq:delta_XA},
we have that
\[
    \sum_{i,j}\frac{1}{p}\delta_{ij}\left\Vert \left(\bm{X}\bm{A}_{1}\right)_{i,\cdot}\right\Vert ^{2}\left\Vert \bm{X}\bm{A}_{2}\right\Vert _{2,\infty}\left\Vert \bm{Y}\bm{B}_{2}\right\Vert _{2,\infty}\le2n_{2}\|\bm{A}_{1}\|_{\mathrm{F}}^{2}\left\Vert \bm{X}\bm{A}_{2}\right\Vert _{2,\infty}\left\Vert \bm{Y}\bm{B}_{2}\right\Vert _{2,\infty}.
\]
Similarly when $n_{1}n_{2}p\ge C\mu_{2}a_{2}\log n$ we can obtain
\[
    \sum_{i,j}\frac{1}{p}\delta_{ij}\left\Vert \left(\bm{Y}\bm{B}_{1}\right)_{i,\cdot}\right\Vert ^{2}\left\Vert \bm{X}\bm{A}_{2}\right\Vert _{2,\infty}\left\Vert \bm{Y}\bm{B}_{2}\right\Vert _{2,\infty}\le2n_{1}\|\bm{B}_{1}\|_{\mathrm{F}}^{2}\left\Vert \bm{X}\bm{A}_{2}\right\Vert _{2,\infty}\left\Vert \bm{Y}\bm{B}_{2}\right\Vert _{2,\infty}.
\]
Combining both bounds and substituting them into \eqref{eq:P_OmegaXABY_IP} yields
\[
    \frac{1}{p}\left\langle \mathcal{P}_{\Omega}\left(\bm{X}\bm{A}_{1}\bm{B}_{1}^{\top}\bm{Y}^{\top}\right),\mathcal{P}_{\Omega}\left(\bm{X}\bm{A}_{2}\bm{B}_{2}^{\top}\bm{Y}^{\top}\right)\right\rangle \le2n\left(\|\bm{A}_{1}\|_{\mathrm{F}}^{2}+\|\bm{B}_{1}\|_{\mathrm{F}}^{2}\right)\left\Vert \bm{X}\bm{A}_{2}\right\Vert _{2,\infty}\left\Vert \bm{Y}\bm{B}_{2}\right\Vert _{2,\infty}.
\]
Notice that the four terms in the last line of \eqref{eq:P_OmegaXABY_IP}
are exchangeable, so we have \eqref{eq:P_Omega_IP}.

\subsection{Proof of Lemma \ref{lem:P_Omega_norm}}

Similar to \eqref{eq:P_Omega_IP} we have
\begin{align}
    \frac{1}{p}\left\Vert \mathcal{P}_{\Omega}\left(\bm{X}\bm{A}\bm{B}^{\top}\bm{Y}^{\top}\right)\right\Vert _{\mathrm{F}}^{2} & \le\sum_{i,j}\frac{1}{p}\delta_{ij}\left\Vert \left(\bm{X}\bm{A}\right)_{i,\cdot}\right\Vert ^{2}\left\Vert \left(\bm{Y}\bm{B}\right)_{j,\cdot}\right\Vert ^{2}\nonumber                                                                                                                                                                             \\
                                                                                                                               & \le\sum_{i,j}\frac{1}{p}\left(\delta_{ij}-p\right)\left\Vert \left(\bm{X}\bm{A}\right)_{i,\cdot}\right\Vert ^{2}\left\Vert \left(\bm{Y}\bm{B}\right)_{j,\cdot}\right\Vert ^{2}+\sum_{i,j}\left\Vert \left(\bm{X}\bm{A}\right)_{i,\cdot}\right\Vert ^{2}\left\Vert \left(\bm{Y}\bm{B}\right)_{j,\cdot}\right\Vert ^{2}.\label{eq:Omega_XABY_decomp}
\end{align}
The second term can be bounded by
\begin{align}
    \sum_{i,j}\left\Vert \left(\bm{X}\bm{A}\right)_{i,\cdot}\right\Vert ^{2}\left\Vert \left(\bm{Y}\bm{B}\right)_{j,\cdot}\right\Vert ^{2} & =\sum_{i}\left\Vert \left(\bm{X}\bm{A}\right)_{i,\cdot}\right\Vert ^{2}\sum_{j}\left\Vert \left(\bm{Y}\bm{B}\right)_{j,\cdot}\right\Vert ^{2}\nonumber \\
                                                                                                                                           & =\left\Vert \bm{X}\bm{A}\right\Vert _{\mathrm{F}}^{2}\left\Vert \bm{Y}\bm{B}\right\Vert _{\mathrm{F}}^{2}\label{eq:XA_YB_AB}                           \\
                                                                                                                                           & =\left\Vert \bm{A}\right\Vert _{\mathrm{F}}^{2}\left\Vert \bm{B}\right\Vert _{\mathrm{F}}^{2}.\nonumber
\end{align}
For the first term, we first see that
\begin{equation}
    \sum_{i,j}\frac{1}{p}\left(\delta_{ij}-p\right)\left\Vert \left(\bm{X}\bm{A}\right)_{i,\cdot}\right\Vert ^{2}\left\Vert \left(\bm{Y}\bm{B}\right)_{j,\cdot}\right\Vert ^{2}\le\sum_{i,j}\frac{1}{p}\left(\delta_{ij}-p\right)\left\Vert \left(\bm{X}\bm{A}\right)_{i,\cdot}\right\Vert ^{2}\|\bm{Y}\bm{B}\|_{2,\infty}^{2}.\label{eq:delta-p_YB}
\end{equation}
Similar to \eqref{eq:delta_XA} we have with probability at least
$1-n^{-10}$,
\begin{equation}
    \sum_{i,j}\frac{1}{p}\left(\delta_{ij}-p\right)\left\Vert \left(\bm{X}\bm{A}\right)_{i,\cdot}\right\Vert ^{2}\le\frac{1}{p}\|\bm{A}\|_{\mathrm{F}}^{2}\left\Vert \sum_{i,j}\left(\delta_{ij}-p\right)\bm{X}_{i,\cdot}^{\top}\bm{X}_{i,\cdot}\right\Vert .\label{eq:delta-p_A_X}
\end{equation}
By \eqref{eq:delta-p_XTX} we have that
\begin{equation}
    \left\Vert \sum_{i,j}\left(\delta_{ij}-p\right)\bm{X}_{i,\cdot}^{\top}\bm{X}_{i,\cdot}\right\Vert \le40\left(\sqrt{\frac{\mu_{1}a_{1}n_{2}p}{n_{1}}\log n}+\frac{\mu_{1}a_{1}}{n_{1}}\log n\right)\le80\sqrt{\frac{\mu_{1}a_{1}n_{2}p}{n_{1}}\log n}\label{eq:delta-p_XTX-1}
\end{equation}
when $n_{1}n_{2}p\ge C\mu_{1}a_{1}\log n$ for some constant $C$
large enough. Feeding \eqref{eq:delta-p_YB}, \eqref{eq:delta-p_A_X},
\eqref{eq:delta-p_XTX-1} and \eqref{eq:XA_YB_AB} back to \eqref{eq:Omega_XABY_decomp}
yields
\[
    \frac{1}{p}\left\Vert \mathcal{P}_{\Omega}\left(\bm{X}\bm{A}\bm{B}^{\top}\bm{Y}^{\top}\right)\right\Vert _{\mathrm{F}}^{2}\le80\sqrt{\frac{\mu_{1}a_{1}n_{2}p}{n_{1}}\log n}\|\bm{A}\|_{\mathrm{F}}^{2}\|\bm{Y}\bm{B}\|_{2,\infty}^{2}+\left\Vert \bm{A}\right\Vert _{\mathrm{F}}^{2}\left\Vert \bm{B}\right\Vert _{\mathrm{F}}^{2}.
\]
Similarly when $n_{1}n_{2}p\ge C\mu_{2}a_{2}\log n$ we have
\[
    \frac{1}{p}\left\Vert \mathcal{P}_{\Omega}\left(\bm{X}\bm{A}\bm{B}^{\top}\bm{Y}^{\top}\right)\right\Vert _{\mathrm{F}}^{2}\le80\sqrt{\frac{\mu_{2}a_{2}n_{1}p}{n_{2}}\log n}\|\bm{B}\|_{\mathrm{F}}^{2}\|\bm{X}\bm{A}\|_{2,\infty}^{2}+\left\Vert \bm{A}\right\Vert _{\mathrm{F}}^{2}\left\Vert \bm{B}\right\Vert _{\mathrm{F}}^{2}.
\]
Combining the two inequalities yields the desired result.

\subsection{Proof of Equation \texorpdfstring{\eqref{eq:Z_from_Delta}}{equation}\label{subsec:Proof-of-Z_from_delta}}

Since $\bm{Z}^{\star}=\bm{A}^{\star}\bm{B}^{\star\top}=\bar{\bm{A}}\bar{\bm{B}}^{\top}$,
\begin{align*}
    \widehat{\bm{Z}}-\bm{Z}^{\star} & =\bm{A}^{t_{0}}\bm{B}^{t_{0}\top}-\bar{\bm{A}}\bar{\bm{B}}^{\top}                                                                                          \\
                                    & =\bm{\Delta}_{\bm{A}}^{t_{0}}\bar{\bm{B}}^{\top}+\bar{\bm{A}}\bm{\Delta}_{\bm{B}}^{t_{0}\top}+\bm{\Delta}_{\bm{A}}^{t_{0}}\bm{\Delta}_{\bm{B}}^{t_{0}\top}
\end{align*}
Then
\begin{align*}
    \left\Vert \widehat{\bm{Z}}-\bm{Z}^{\star}\right\Vert _{\mathrm{F}}^{2} & \le3\left(\left\Vert \bm{\Delta}_{\bm{A}}^{t_{0}}\bar{\bm{B}}^{\top}\right\Vert _{\mathrm{F}}^{2}+\left\Vert \bar{\bm{A}}\bm{\Delta}_{\bm{B}}^{t_{0}\top}\right\Vert _{\mathrm{F}}^{2}+\left\Vert \bm{\Delta}_{\bm{A}}^{t_{0}}\bm{\Delta}_{\bm{B}}^{t_{0}\top}\right\Vert _{\mathrm{F}}^{2}\right)                                          \\
                                                                            & \le3\left(\|\bar{\bm{B}}\|^{2}\left\Vert \bm{\Delta}_{\bm{A}}^{t_{0}}\right\Vert _{\mathrm{F}}^{2}+\|\bar{\bm{A}}\|^{2}\left\Vert \bm{\Delta}_{\bm{B}}^{t_{0}}\right\Vert _{\mathrm{F}}^{2}+\left\Vert \bm{\Delta}_{\bm{A}}^{t_{0}}\right\Vert _{\mathrm{F}}^{2}\left\Vert \bm{\Delta}_{\bm{B}}^{t_{0}}\right\Vert _{\mathrm{F}}^{2}\right) \\
                                                                            & \le3\left(\sigma_{\max}\left\Vert \bm{\Delta}^{t_{0}}\right\Vert _{\mathrm{F}}^{2}+\left\Vert \bm{\Delta}^{t_{0}}\right\Vert _{\mathrm{F}}^{4}\right)                                                                                                                                                                                       \\
                                                                            & \le6\sigma_{\max}\left\Vert \bm{\Delta}^{t_{0}}\right\Vert _{\mathrm{F}}^{2}.
\end{align*}
The last line uses $\left\Vert \bm{\Delta}^{t_{0}}\right\Vert _{\mathrm{F}}\le\frac{1}{10}\sqrt{\sigma_{\min}}$.

%% file: sections/14-further-experimental-results.tex
\section{Further experimental results \label{sec:further-experimental-results}}

\subsection{Detailed simulation setup \label{subsec:detailed-simulation-setup}}

This subsection gives the details omitted from the simulation setup
in Section~\ref{subsec:Simulations}. The ground truth matrix
$\bm{L}^{\star}$ has dimensions $n_{1}=n_{2}=n=1000$ and rank $r=10$,
and the side-information dimension is $a_{1}=a_{2}=a=50$. We first
generate orthonormal matrices
$\bm{X}^{\star},\bm{Y}^{\star}\in\mathbb{R}^{n\times a}$ and a rank-$r$
matrix $\bm{Z}\in\mathbb{R}^{a\times a}$ with spectral norm $1$,
and define
\[
    \bm{L}^{\star}=\bm{X}^{\star}\bm{Z}{\bm{Y}^{\star}}^{\top}.
\]
The matrix $\bm{L}^{\star}$ is generated once and then kept fixed
throughout the synthetic experiments. For each trial, the observed
index set $\Omega$ is resampled with independent entrywise sampling
probability $p$. In the noiseless setting, we observe
\[
    \bm{M}=\mathcal{P}_{\Omega}(\bm{L}^{\star}),
\]
while in the noisy setting we observe
\[
    \bm{M}=\mathcal{P}_{\Omega}(\bm{L}^{\star}+\bm{E}),
\]
where the entries of $\bm{E}$ are i.i.d.\ Gaussian with mean $0$
and standard deviation $\sigma=0.001$.

For experiments with inexact side-information, we replace
$\bm{X}^{\star},\bm{Y}^{\star}$ by matrices $\bm{X},\bm{Y}$ with
prescribed side-information inexactness $\delta$ as in
Assumption~\ref{assumption:inexact_SI}. Here we describe the details of the construction
for the column side-information $\bm{X}$; the construction for
$\bm{Y}$ is analogous.

Let $\bm{U}^{\star}$ be the left singular matrix of $\bm{L}^{\star}$.
We first choose target principal angles
\[
    0\le\theta_{1}\le\cdots\le\theta_{r}=\arcsin\delta
\]
with the first
$r-1$ angles drawn uniformly from $[0,\arcsin\delta]$. Next, we extend
$\bm{U}^{\star}$ to an orthonormal basis of $\mathbb{R}^{n}$ and select
orthonormal matrices $\bm{U}_{1}\in\mathbb{R}^{n\times r}$ and
$\bm{U}_{2}\in\mathbb{R}^{n\times(a-r)}$ such that both $\bm{U}_{1}$ and $\bm{U}_{2}$ are orthogonal
to $\bm{U}^{\star}$ and $\bm{U}_{2}$ is also orthogonal to $\bm{U}_{1}$.
We then define the first $r$ columns of $\bm{X}$ by
\[
    \bm{X}_{\cdot,i}=\cos(\theta_{i})[\bm{U}^{\star}]_{\cdot,i}+\sin(\theta_{i})[\bm{U}_{1}]_{\cdot,i},
    \qquad 1\le i\le r,
\]
where $[\bm{U}^{\star}]_{\cdot,i}$ and $[\bm{U}_{1}]_{\cdot,i}$ are the $i$-th columns of $\bm{U}^{\star}$ and $\bm{U}_{1}$, respectively. We then take the remaining $a-r$ columns from $\bm{U}_{2}$.

This construction
ensures that the principal angles between $\mathrm{col}(\bm{X})$
and $\mathrm{col}(\bm{U}^{\star})$ are exactly
$\theta_{1},\dots,\theta_{r}$, and hence
$\|\sin\bm{\Theta}_{1}\|=\delta$.
For each $1\le i\le r$, the vector $\bm{X}_{i}$ lies in the
two-dimensional plane spanned by $\bm{U}^{\star}_{i}$ and
$(\bm{U}_{1})_{i}$, and makes angle exactly $\theta_{i}$ with
$\bm{U}^{\star}_{i}$. These planes are mutually orthogonal because
the columns of $\bm{U}^{\star}$ and $\bm{U}_{1}$ are orthonormal,
and the remaining columns from $\bm{U}_{2}$ are orthogonal to
$\mathrm{col}(\bm{U}^{\star})$. Therefore
$\bm{X}^{\top}\bm{U}^{\star}$ is diagonal in these coordinates, with
diagonal entries $\cos(\theta_{1}),\dots,\cos(\theta_{r})$. Hence
the singular values of $\bm{X}^{\top}\bm{U}^{\star}$ are exactly
$\cos(\theta_{1}),\dots,\cos(\theta_{r})$, which means the principal
angles between $\mathrm{col}(\bm{X})$ and $\mathrm{col}(\bm{U}^{\star})$
are precisely $\theta_{1},\dots,\theta_{r}$.

\subsection{Redundancy of Projection Steps in Algorithm~\ref{alg:main}
    \label{subsec:Redundancy-of-projection}\label{subsec:experiments_appendix}}

We mentioned in Section~\ref{sec:Problem-Setup} that the projection step in Algorithm~\ref{alg:main} is mainly a proof device and is inactive in our simulated experiments. To quantify this, we track how close the unprojected iterates get to the boundary of the projection
set $\mathcal{C}$. Recall that the projection constraint is
\[
    \max\left\{\|\bm{X}\bm{A}\|_{2,\infty},
    \|\bm{Y}\bm{B}\|_{2,\infty}\right\}
    \le2\sqrt{\frac{\mu_{0}r}{n_{\min}}}\|\bm{A}^{0}\|.
\]
For each iteration $t$, we define the relative ratio to the projection boundary as
\[
    \rho_{t}\coloneqq
    \frac{\max\left\{\|\bm{X}\bm{A}^{t}\|_{2,\infty},
    \|\bm{Y}\bm{B}^{t}\|_{2,\infty}\right\}}
    {2\sqrt{\mu_{0}r/n_{\min}}\|\bm{A}^{0}\|}.
\]
The projection is inactive whenever $\rho_{t}\le 1$.

We showcase the inactive behavior of the projection step in the following simulation.
We use the same synthetic setup as in Section~\ref{subsec:Simulations}
to generate the ground truth matrix $\bm{L}^{\star}$ and exact side-information
matrices $\bm{X},\bm{Y}$, with parameters
$n_{1}=n_{2}=1000$, $a_{1}=a_{2}=50$, and $r=10$. We observe noiseless partial observations
$\bm{M}=\mathcal{P}_{\Omega}(\bm{L}^{\star})$ with observation probability $p$. For each
$p\in\{0.005,0.01,0.02,0.05\}$, we run unprojected gradient descent
for 200 iterations over 50 independent trials, resampling the
ground truth and $\Omega$ in each trial. Figure~\ref{fig:projection_redundancy}
shows the distribution of $\max_{t}\rho_{t}$ over the trials. In all tested regimes, $\max_{t}\rho_{t}<1$,
so the projection step is inactive throughout the iterates.

\begin{figure}[tp]
    \centering{

        \includegraphics[width=0.55\textwidth]{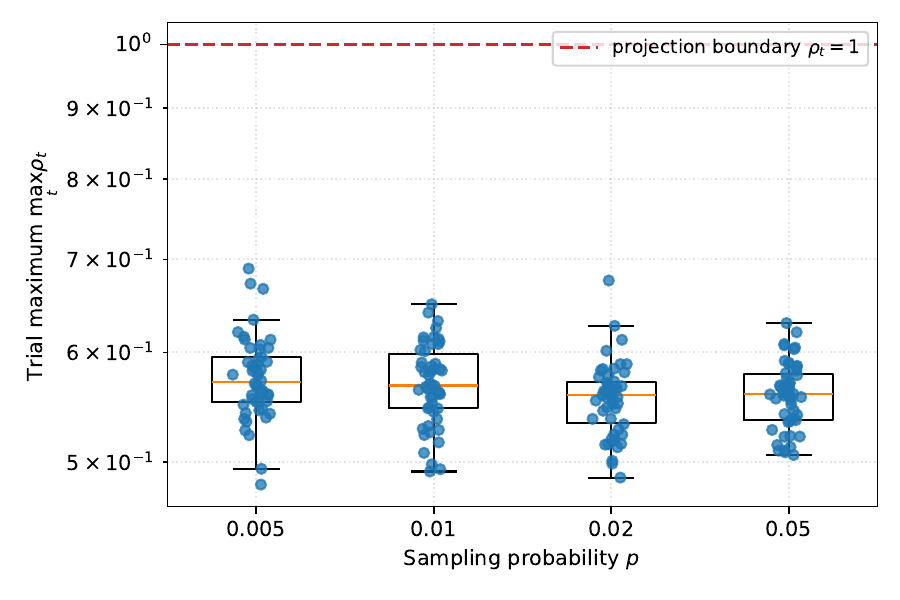}

    }

    \caption{\label{fig:projection_redundancy}Distribution
        of $\max_{t}\rho_{t}$ over the trials.
        The dashed horizontal line is the projection boundary $\rho_{t}=1$.}
\end{figure}